\documentclass{article}

\usepackage{microtype}
\usepackage{graphicx}
\usepackage{subcaption}
\usepackage{booktabs}       
\usepackage{url}            
\usepackage{amsfonts}       
\usepackage{nicefrac}       
\usepackage{natbib}         
\usepackage{tikz}
\usetikzlibrary{calc, backgrounds, fit, shapes, positioning, arrows.meta}
\usepackage{xcolor}
\usepackage{lipsum}
\usepackage{titletoc}       
\usepackage{wrapfig}
\usepackage{tikz}
\usetikzlibrary{arrows.meta} 

\graphicspath{{./images/}}

\usepackage[dvipsnames]{xcolor}
\definecolor{lowcolor}{rgb}{0.43, 0.21, 0.1}
\definecolor{highcolor}{rgb}{0.0, 0.5, 0.0}
\definecolor{softpurple}{RGB}{186, 145, 235}
\definecolor{softgreen}{RGB}{144, 238, 144}

\usepackage{amsmath,amssymb,amsthm}
\usepackage{mathtools}
\usepackage{mathabx}
\usepackage{bbold}
\usepackage{bm}

\newtheorem{remark}{Remark}

\DeclareMathOperator*{\argmin}{arg\,min}

\usepackage{tikz}
\usetikzlibrary{bayesnet,shapes.geometric,fit,backgrounds}

\usepackage{minitoc}

\usepackage{algorithmic}

\usepackage{multirow}
\usepackage{enumitem}
\usepackage{booktabs}

\usepackage{tcolorbox}

\usepackage[colorlinks=true, citecolor=darkgreen, linkcolor=black, urlcolor=darkgreen]{hyperref}

\usepackage[textsize=tiny]{todonotes}

\newcommand{\AlgComment}[1]{\hfill{\footnotesize\emph{#1}}}

\usepackage[preprint]{icml2026}

\theoremstyle{plain}
\newtheorem{theorem}{Theorem}[section]
\newtheorem{proposition}[theorem]{Proposition}
\newtheorem{lemma}[theorem]{Lemma}

\theoremstyle{definition}

\newtheorem{defn}[theorem]{Definition}

\theoremstyle{remark}

\icmltitlerunning{Distributionally Robust Causal Abstractions}

\newcommand{\dom}[1]{\operatorname{dom}[#1]}

\newcommand{\scmsimple}{\mathcal{M}}

\newcommand{\scmbase}{\mathcal{M}^{\ell}}
\newcommand{\scmabst}{\mathcal{M}^{h}}
\newcommand{\mixbase}{\mathbf{L}}
\newcommand{\mixabst}{\mathbf{H}}
\newcommand{\intervsetbase}{\mathcal{I}^{\ell}}
\newcommand{\intervsetabst}{\mathcal{I}^{h}}
\newcommand{\intervsetprod}{\mathcal{I}}

\newcommand{\mixfbase}{\mathbf{g}^\ell}
\newcommand{\mixfabst}{\mathbf{g}^h}

\newcommand{\structfuncset}{\mathcal{F}}
\newcommand{\scmdag}{\mathcal{G}}

\newcommand{\enbase}{\mathcal{X}^{\ell}}
\newcommand{\enabst}{\mathcal{X}^{h}}

\newcommand{\tauu}{\tau_{\mathcal{U}}}
\newcommand{\tauupush}{\tau_{\mathcal{U} \#}}
\newcommand{\taupush}{\tau_{\#}}

\newcommand{\exbase}{\mathcal{U}^{\ell}}
\newcommand{\exabst}{\mathcal{U}^{h}}

\newcommand{\en}{\mathcal{X}}
\newcommand{\ex}{\mathcal{U}}

\newcommand{\stablebase}{\cA^\ell}
\newcommand{\stableabst}{\cA^h}
\newcommand{\stableprod}{\cA}

\newcommand{\tauomega}{(\tau,\omega)}

\newcommand{\rhoiota}{(\rho,\iota)}

\newcommand{\prob}{\mathbb{P}}\newcommand{\expval}{\mathbb{E}}

\newcommand{\contbase}{\mathbf{u}^L}
\newcommand{\contabst}{\mathbf{u}^H}

\newcommand{\exprobbase}{\rho^\ell}
\newcommand{\exprobabst}{\rho^h}
\newcommand{\exprob}{\boldsymbol{\rho}}

\newcommand{\exprod}{\boldsymbol{\cU}}

\newcommand{\radbase}{\varepsilon_\ell}
\newcommand{\radabst}{\varepsilon_h}
\newcommand{\radprod}{\epsilon}

\newcommand{\lintau}{\operatorname{T}}

\newcommand{\cA}{{\cal A}}

\newcommand{\cD}{{\cal D}}

\newcommand{\cF}{{\cal F}}

\newcommand{\cI}{{\cal I}}

\newcommand{\cN}{{\cal N}}

\newcommand{\cP}{{\cal P}}

\newcommand{\cS}{{\cal S}}

\newcommand{\cU}{{\cal U}}

\newcommand{\bB}{{\mathbb B}}

\newcommand{\bR}{{\mathbb R}}

\newcommand{\mul}{\mathbf{\mu}_\ell}
\newcommand{\muh}{\mathbf{\mu}_h}
\newcommand{\sigl}{\Sigma_\ell}
\newcommand{\sigh}{\Sigma_h}

\newcommand{\absp}{\text{AbsLin}_\texttt{p}}
\newcommand{\absn}{\text{AbsLin}_\texttt{n}}

\begin{document}

\twocolumn[
  \icmltitle{Distributionally Robust Causal Abstractions}



  \icmlsetsymbol{equal}{*}

  \begin{icmlauthorlist}
    \icmlauthor{Yorgos Felekis}{cs}
    \icmlauthor{Theodoros Damoulas}{cs,stat}
    \icmlauthor{Paris Giampouras}{cs}
  \end{icmlauthorlist}

  \icmlaffiliation{cs}{Department of Computer Science, University of Warwick, Coventry, UK}
  \icmlaffiliation{stat}{Department of Statistics, University of Warwick, Coventry, UK}

  \icmlcorrespondingauthor{Yorgos Felekis}{yorgos.felekis@warwick.ac.uk}

  \icmlkeywords{Machine Learning, ICML}

  \vskip 0.3in
]



\printAffiliationsAndNotice{}  

\begin{abstract}
Causal Abstraction (CA) theory provides a principled framework for relating causal models that describe the same system at different levels of granularity while ensuring interventional consistency between them. Recent methods for learning CAs, however, assume fixed and well-specified exogenous distributions, leaving them vulnerable to environmental shifts and model misspecification. In this work, we address these limitations by introducing the first class of distributionally robust CAs and their associated learning algorithms. The latter cast robust causal abstraction learning as a constrained min-max optimization problem with Wasserstein ambiguity sets. We provide theoretical guarantees for both empirical and Gaussian environments, enabling principled selection of ambiguity-set radii and establish quantitative guarantees on worst-case abstraction error. Furthermore, we present empirical evidence across different problems and CA learning methods, demonstrating our framework’s robustness not only to environmental shifts but also to structural and intervention mapping misspecification.
\end{abstract}

\begin{figure*}[t]
\centering
\definecolor{myBlue}{RGB}{60, 100, 190}
\definecolor{myOrange}{RGB}{240, 100, 20}
\definecolor{myMagenta}{RGB}{200, 0, 100}
\definecolor{myTeal}{RGB}{0, 150, 130}
\definecolor{myYellow}{RGB}{210, 145, 15}

\begin{tikzpicture}[
    scale=0.7, 
    transform shape, 
    font=\sffamily,
    >={Stealth[round]}
]

    \def\HighLevelY{2.9} 


    \coordinate (HighCenter) at (0, \HighLevelY);
    \coordinate (LowCenter) at (0, 0);

    
    \fill[myYellow!10] 
        (-2.8, \HighLevelY) to[out=260, in=100] (-3.4, 0.2) -- 
        (3.2, 0.4) to[out=80, in=280] (2.8, \HighLevelY) -- cycle;
        
    \fill[myMagenta!15] 
        (-1.4, \HighLevelY) to[out=265, in=95] (-1.8, 0.0) -- 
        (1.8, 0.0) to[out=85, in=275] (1.4, \HighLevelY) -- cycle;

    
    \fill[myBlue!20, opacity=0.8] 
        plot[smooth cycle, tension=0.7] coordinates {
            ($(LowCenter)+(-3.2, 0.5)$)
            ($(LowCenter)+(-3.5, -0.8)$)
            ($(LowCenter)+(-1.0, -1.5)$)
            ($(LowCenter)+(2.8, -1.2)$)
            ($(LowCenter)+(3.2, 0.8)$)
            ($(LowCenter)+(1.0, 1.4)$)
            ($(LowCenter)+(-1.5, 1.0)$)
        };
    
    \fill[myBlue!30] 
        plot[smooth cycle, tension=0.6] coordinates {
            ($(LowCenter)+(-1.8, 0.1)$)
            ($(LowCenter)+(-0.9, -0.6)$)
            ($(LowCenter)+(1.8, -0.1)$)
            ($(LowCenter)+(0.9, 0.6)$)
        };
    \draw[myBlue, line width=1.2pt] 
        plot[smooth cycle, tension=0.6] coordinates {
            ($(LowCenter)+(-1.8, 0.1)$)
            ($(LowCenter)+(-0.9, -0.6)$)
            ($(LowCenter)+(1.8, -0.1)$)
            ($(LowCenter)+(0.9, 0.6)$)
        };
    
    
    \draw[->, myYellow, line width=1pt] (-3.4, 0.2) to[out=100, in=260] (-2.95, 2.55);
    \draw[->, myYellow, line width=1pt] (3.27, 0.6) to[out=80, in=280] (2.8, \HighLevelY);

    \draw[->, myMagenta, line width=0.8pt] (-1.86, 0.0) to[out=95, in=265] (-1.4, \HighLevelY);
    \draw[->, myMagenta, line width=0.8pt] (1.86, 0.0) to[out=85, in=275] (1.4, \HighLevelY);

    
    \fill[myOrange!20, opacity=0.8] 
        plot[smooth cycle, tension=0.7] coordinates {
            ($(HighCenter)+(-2.8, 0.6)$)
            ($(HighCenter)+(-3.0, -0.4)$)
            ($(HighCenter)+(-1.5, -1.0)$)
            ($(HighCenter)+(1.8, -0.6)$) 
            ($(HighCenter)+(2.8, 0.5)$)
            ($(HighCenter)+(0.5, 1.0)$)   
            ($(HighCenter)+(-1.0, 0.8)$)
        };

    \fill[myOrange!30] 
        plot[smooth cycle, tension=0.6] coordinates {
            ($(HighCenter)+(-1.4, 0)$)
            ($(HighCenter)+(-0.6, -0.5)$)
            ($(HighCenter)+(1.4, 0)$)
            ($(HighCenter)+(0.6, 0.5)$)
        };
    \draw[myOrange, line width=1.2pt] 
        plot[smooth cycle, tension=0.6] coordinates {
            ($(HighCenter)+(-1.4, 0)$)
            ($(HighCenter)+(-0.6, -0.5)$)
            ($(HighCenter)+(1.4, 0)$)
            ($(HighCenter)+(0.6, 0.5)$)
        };

    
    \node[myOrange!100, font=\bfseries\footnotesize, anchor=south] at (1.9, 3.15) {$\mathcal{P}_{m}(\mathcal{U}^h)$};
    \node[circle, fill=myOrange, inner sep=2pt] (HighDot) at (HighCenter) {};
    \node[myOrange, font=\bfseries\footnotesize] at (0.7, 3.1) {$\mathcal{P}_0(\mathcal{U}^h)$};
    
    \node[myBlue!100, font=\bfseries\footnotesize, anchor=north] at (2.3, 0.0) {$\mathcal{P}_{m}(\mathcal{U}^\ell)$};
    \node[circle, fill=myBlue, inner sep=2pt] (LowDot) at (LowCenter) {};
    \node[myBlue, font=\bfseries\footnotesize, anchor=north] at (0.75, 0.25) {$\mathcal{P}_0(\mathcal{U}^\ell)$};

    
    \draw[->, myTeal, line width=1.5pt] (LowDot) -- (HighDot);
    
    \node[myYellow!100, anchor=east, align=right, font=\small] (citUniform) at (-3.8, \HighLevelY - 0.2) {
        \textbf{Uniform} $(\tau, \omega)$\\
        \citet{beckers2018abstracting}
    };
    \draw[->, dashed, myYellow, opacity=0.7, line width=1pt] (citUniform.east) to[bend left=10] (-3.15, 1.9);

    \node[myMagenta, anchor=east, align=right, font=\small] (citRobust) at (-4, \HighLevelY/2) {
        \textbf{Robust} $(\rho, \iota)$\\
        \textbf{This work}
    };
    \draw[->, dashed, myMagenta!80, opacity=0.7, line width=1pt] (citRobust.east) -- (-1.7, \HighLevelY/2);

    \node[myTeal, anchor=east, align=right, font=\small] (citExact) at (-4, 0.2) {
        \textbf{Exact} $(\tau, \omega)$ / $(R,a,\alpha)$\\
        \citet{rubenstein2017causal}\\
        \citet{rischel2020category}
    };
    \draw[->, dashed, myTeal, opacity=0.7, line width=1pt] (citExact.east) to[bend right=10] (-0.05, \HighLevelY/3);

    \node[font=\bfseries, myBlue!60] at (3.75, -1) {$\mathcal{P}_{\infty}(\mathcal{U}^\ell)$};
    \node[font=\bfseries, myOrange!60] at (3.5, 3.5) {$\mathcal{P}_{\infty}(\mathcal{U}^h)$};

    \draw[dashed, line width=0.8pt, black!50] (4.5, -1.45) -- (4.5, \HighLevelY + 1.0);

    \node (toparr) at (-7.5, \HighLevelY) {};
    \node (botarr) at (-7.5,-0.3) {};
    \draw[->, line width=1pt] (botarr) -- (toparr);
    \node[rotate=90] at (-7.9, \HighLevelY/2) {\footnotesize Environmental Strength};

    \node at (-9.5, \HighLevelY/2) {\bfseries \normalsize Causal};
    \node at (-9.5, \HighLevelY/2 - 0.5) {\bfseries \normalsize Abstractions};
    \node at (-9.5, \HighLevelY/2 - 1.0) {\bfseries \normalsize (CAs)};

    \node at (12.8, \HighLevelY/2) {\bfseries \normalsize CA};
    \node at (12.8, \HighLevelY/2- 0.5) {\bfseries \normalsize Learning};
    \node at (12.8, \HighLevelY/2- 1.0) {\bfseries \normalsize (CAL)};

    \coordinate (RightCenter) at (8.2, \HighLevelY/2 + 0.5);


    \draw[myYellow!5, fill=myYellow!10, rounded corners=8pt] (5.0, -1.5) rectangle (11.4, \HighLevelY + 1.0);
    \node[myYellow, anchor=north east, font=\large] at (11.4, \HighLevelY + 0.9) {$\mathcal{P}(\mathcal{U}^\ell \times \mathcal{U}^h)$};

    \node[circle, fill=myTeal, inner sep=2.5pt] (Empirical) at (RightCenter) {};
    \node[myTeal, font=\footnotesize, anchor=north] at ($(RightCenter)+(0,-0.3)$) {$\widehat{\mathcal{P}_0(\mathcal{U})}$};
    
    \coordinate (TrueLoc) at ($(RightCenter)+(-0.5, 0.5)$);
    \node[star, fill=myMagenta!90, inner sep=2.5pt] (TrueDot) at (TrueLoc) {};
    \node[magenta!90, font=\footnotesize, anchor=east] at ($(TrueLoc)+(0.6, 0.4)$) {$\mathcal{P}_0(\mathcal{U})$};
      
    \draw[myMagenta, thick, dashed] (RightCenter) circle (1.5cm); 
    \fill[myMagenta, opacity=0.15] (RightCenter) circle (1.5cm);
    \node[myMagenta!80!black, font=\bfseries\Large] at ($(RightCenter)+(135:1.9)$) {$\mathbb{B}_\epsilon$};
    \draw[->, black!70] (RightCenter) -- node[above, font=\large] {$\epsilon$} ++(20:1.5cm);

    \node[align=left, font=\scriptsize, anchor=west] (calStandard) at (8.3, -0.5) {
        \citet{zennaro2023jointly}\\
        \citet{felekis2024causal}\\
        \citet{kekić2023targeted}\\
        \citet{xia2024neural}\\
        \citet{dacunto2025causalabstractionlearningbased}
    };
    \draw[->, myTeal, line width=1pt] (calStandard.north) to[bend right=20] ($(RightCenter)+(0.2,-0.2)$);

    \node[myMagenta, align=left, font=\footnotesize, anchor=west] (calRobust) at (5.4, -0.4) {
        \textbf{This work}
    };
    
    \draw[->, myMagenta, line width=1pt] (calRobust.north) to[bend left=30] ($(RightCenter)+(-0.8,-0.3)$);
    
    \node[myTeal, font=\large] at ($(RightCenter)+(2.2, -1.0)$) {$\epsilon=0$};
    \node[myMagenta, font=\large] at ($(RightCenter)+(-2.4, -1.5)$) {$\epsilon \geq 0$};

\end{tikzpicture}
\caption{\textbf{Left:} Hierarchy of CAs based on the \textit{environmental assumptions} required for the abstraction to be valid. {\color{myTeal} Exact CAs} assume consistency in a single joint environment,  {\color{myYellow} Uniform CAs} require consistency across all possible joint environments while our framework of {\color{myMagenta}Robust CAs} requires the mapping to hold over a constrained subset of relevant joint environments. \textbf{Right:} CAL methods positioned relative to this hierarchy on the joint environment {\color{myYellow}$\cP(\exbase \times \exabst)$}. Our framework models environmental uncertainty over a subset of this using a Wasserstein ball {\color{myMagenta}$\mathbb{B}_{\epsilon}$} with {\color{myMagenta}$\epsilon \geq 0$}, around an empirical joint environment {\color{myTeal}$\widehat{\mathcal{P}_0(\mathcal{U})}$}. Prior approaches correspond to {\color{myTeal}$\epsilon=0$}, implicitly assuming a fixed environment to learn an approximate CA.}
\label{fig:casoverview}
\end{figure*}

\section{Introduction}\label{sec:main_intro}
Causal reasoning provides the foundation for understanding and influencing complex systems: it enables us to move beyond correlation to estimate the effects of interventions, simulate counterfactual scenarios, and uncover the underlying mechanisms that drive a system’s behavior. These capabilities are central to out-of-distribution generalization and to building interpretable, fair, and socially aligned machine learning systems. Motivated by this, Causal Representation Learning \citep{scholkopf2021toward} seeks latent representations that both compress data and preserve the causal structure needed for robust reasoning. This challenge is particularly apparent in systems that naturally operate at multiple levels of abstraction, such as cells versus organs, neurons versus brain regions, or individuals versus populations, where causal relationships may differ across scales. While fine-grained models are often more expressive, they tend to be opaque and computationally expensive; more abstract models are interpretable and efficient, but only faithful if they preserve the causal semantics of the systems they simplify. This trade-off motivates the need for principled causal reasoning across abstraction levels, allowing reliable reasoning at both detailed and aggregate views of a system.

The theory of Causal Abstractions (CAs) addresses this challenge by formalizing causally consistent models at different granularities while also characterizing the maps that relates them. Two main frameworks have shaped this space: $\tau\text{-}\omega$ abstractions \citep{rubenstein2017causal, beckers2018abstracting}, which define a single mapping between model domains, and $(R, a, \alpha)$ abstractions \citep{rischel2020category}, which separate mappings between variables and their domains. Their relation was recently analyzed in \citep{schooltink2025aligninggraphicalfunctionalcausal}. Building on this, Causal Abstraction Learning (CAL) \citep{zennaro2022computingoptimalabstractionstructural} aims to learn abstraction maps directly from data, enabling cross-scale representation learning, evidence integration from heterogeneous sources, and more efficient modeling.

Early CAL methods established distinct computational frameworks: joint differentiable programming for $(R, a, \alpha)$ abstractions \citep{zennaro2023jointly} and multi-marginal Optimal Transport for the $\tau\text{-}\omega$ framework \citep{felekis2024causal}, followed by several approaches, see Fig. \ref{fig:casoverview}, for learning approximate abstractions \citep{dacunto2025causalabstractionlearningbased, kekić2023targeted, xia2024neural}. Parallel efforts have applied CAs to diverse domains, ranging from climate modeling \citep{chalupka2017causal} and neural network interpretability \citep{geiger2021causal} to agent-based modeling \citep{dyer2023interventionally}, bandits \citep{zennarobandits}, and causal discovery \citep{massidda2024learningcausalabstractionslinear}. A key challenge in CAL is understanding how the environmental conditions, i.e. the exogenous distributions under which an abstraction is learned, influence its generalization and causal consistency properties. While exact abstractions \citep{rubenstein2017causal, rischel2020category} assume a fixed environment without explicitly addressing how the abstraction depends on it, \citet{beckers2018abstracting} propose a notion of abstraction uniformity that requires causal consistency to hold across \emph{all} infinitely many environments; see Fig. \ref{fig:casoverview}. The latter corresponds to a mapping that captures a more fundamental relationship between models: that of their respective deterministic parts; i.e., causal bases (see Definition~\ref{def:scm}). While theoretically stronger, learning such an abstraction is computationally infeasible requiring data from infinite, potentially unattainable environments due to physical or practical constraints.

In this work, we address this challenge by introducing a framework of \emph{robust causal abstractions} that require consistency across a constrained set of relevant environments, providing a notion theoretically stronger than exact abstractions (fixed environment) yet more flexible and tractable than uniform ones (all environments). In the finite sample case, our approach links the level of modeled uncertainty to the strength of the learned abstraction: the broader the ambiguity set, the more powerful the abstraction it approximates. Overall, we make the following contributions: \textbf{(a)} We introduce the first class of \emph{robust causal abstractions}, called \(\rhoiota\)\text{-abstractions}, that requires consistency across a constrained set of environments and \textbf{(b)} propose \emph{Distributionally Robust Causal Abstractions} (\textsc{DiRoCA}), the first framework for learning robust to environmental shifts and misspecifications CAs in Additive Noise SCMs via distributionally robust optimization over Wasserstein ambiguity sets; \textbf{(c)} we derive theoretical concentration bounds to guide ambiguity radius selection and provide provable robustness guarantees; \textbf{(d)} we demonstrate the effectiveness of \textsc{DiRoCA} across diverse problems and settings, outperforming existing methods and baselines under environmental shifts, and different types of misspecifications.

\section{Background}\label{sec:main_back}
We work with \citet{pearl2009causality}'s Structural Causal Models framework, where each variable is a function of its direct causes and the exogenous noise.
\begin{defn}[Structural Causal Model]\label{def:scm}
    A $d$-dimensional \emph{Structural Causal Model (SCM)} is a pair $\scmsimple^d \coloneq (\cS^d, \rho^d)$, where $\cS^d = \langle \en, \ex, \structfuncset \rangle$ defines the deterministic \emph{causal basis}, consisting of a set of $d$ \emph{endogenous variables} $\en$, a set of $d$ \emph{exogenous variables} $\ex$ and a set of $d$ \emph{structural functions} $\structfuncset$, each defining the value of an endogenous variable as $X_i=f_i(\operatorname{PA}(X_i), U_i) \quad \forall \; i \in [d]$, where $\operatorname{PA}(X_i) \subseteq \en \setminus {X_i}$ denotes the direct causes (parents) of $X_i$. The \emph{environment} $\rho^d$ is a joint probability distribution over $\ex$.
\end{defn}

\vspace{-0.5em}
\textbf{Causal Assumptions.} We focus on \emph{Markovian} SCMs, which entail two key properties: (1) \emph{Acyclicity}, meaning $\scmsimple^d$ entails a directed acyclic graph (DAG) $\scmdag_{\scmsimple^d}$ whose nodes are the endogenous variables $\en$ and and whose edges follow the signatures of the functions $f_i$ in $\structfuncset$; and (2) \emph{Joint Independence} of the exogenous variables, i.e., $\rho^d = \prod_{i=1}^d \mathbb{P}(U_i)$. This independence implies \emph{causal sufficiency}, ensuring that there are no unobserved confounders interacting with the system. Furthermore, we assume \emph{faithfulness}, meaning that independencies in the data are captured in the graphical model. Acyclicity allows us to recursively compose the structural functions into a single deterministic map $\mathbf{g}: \dom{\ex} \to \dom{\en}$, referred to as the \emph{mixing function}. This defines the SCM’s \emph{reduced form} $\en = \mathbf{g}(\ex)$, where endogenous variables are expressed purely in terms of exogenous noise. Consequently, the induced distribution is the pushforward $\mathbb{P}_{\scmsimple^d}(\en) = \mathbf{g}_{\#}(\rho^d)$, making explicit the generative process by which exogenous uncertainty propagates through the model. We specifically consider \emph{Additive Noise Models (ANMs)}, where structural assignments take the form $X_i = f_i(\mathrm{PA}(X_i)) + U_i$ for $\forall i \in [d]$. This yields the general decomposition $\en = \cD + \cU$, where $\cD=(f_i(\mathrm{PA}(X_i)))_{i=1}^d$ is the deterministic part and $\cU=(U_i)_{i=1}^d$ is the stochastic part (see Appendix~\ref{sec:anmapp}). A special case is the \emph{Linear ANMs (LANs)}, where $\en = \mathbf{B}^\top \en + \ex$. Due to acyclicity, $\mathbf{B}$ is a weighted adjacency matrix (permutable to strictly upper triangular), and the reduced form becomes $\en = \mathbf{M} \ex$, with mixing matrix $\mathbf{M} \coloneq (I - \mathbf{B}^\top)^{-1}$.
\newline
\textbf{Interventions.} SCMs facilitate reasoning about \emph{interventions}. An exact intervention $\iota = \text{do}(\mathcal{A} = \textbf{a})$ fixes variables $\mathcal{A} \subseteq \en$ to values $\textbf{a}$, while allowing the rest of the system to evolve as usual. Graphically, this \emph{mutilates} $\scmdag_{\scmsimple^d}$ by removing incoming edges to $\mathcal{A}$. This yields a post-interventional SCM $\scmsimple^d_{\iota}$ with joint distribution $\prob_{\scmsimple^d_\iota}(\en)$.
\newline
\textbf{Causal Abstractions.} Causal Abstraction (CA) theory formalizes the relation between SCMs defined at different levels of granularity by requiring \emph{interventional consistency} between them. This enables flexible shifts in the level of representation used for causal reasoning, depending on the specific inquiry or the nature of the available data.

\begin{defn}[\citep{rubenstein2017causal}]
    Let SCMs $\scmbase$ and $\scmabst$ with fixed respective environments $\exprobbase$, $\exprobabst$ and intervention sets $\intervsetbase$, $\intervsetabst$ and a surjective and order-preserving map $\omega: \intervsetbase \to \intervsetabst$. An abstraction $\tau: \dom{\enbase} \to \dom{\enabst}$, with $\ell \geq h$, is called an \emph{exact transformation} if:
\begin{align}
    \tau_{\#}(\prob_{\scmbase_\iota}(\enbase)) = \prob_{\scmabst_{\omega(\iota)}}(\enabst),~~ \forall~\iota \in \intervsetbase
\end{align}
\end{defn}
This is illustrated in Fig.~\ref{fig:tau_approx} by comparing the {\color{teal}intervene $\rightarrow$ abstract} and {\color{violet}abstract $\rightarrow$ intervene} paths; their mismatch defines what is known as the \emph{abstraction error}.

\section{Distributionally Robust Causal Abstractions}\label{sec:main_diroca}
In this section, we introduce a new class of CAs that extends exact transformations by ensuring interventional consistency of the two models across a set of environments. \emph{This framework fills the gap between exact and uniform abstractions by introducing a stronger notion of abstraction that is environment-aware, enabling robustness to environmental misspecification and shifts during learning}. We explicitly model the uncertainty inherent in each SCM through its corresponding environment: $\exprobbase$ for the low-level model $\scmbase$ and $\exprobabst$ for the high-level model $\scmabst$. Assuming independence between them, we define the \emph{joint environment} $\exprob = \exprobbase \otimes \exprobabst$, a product measure that captures the uncertainty across abstraction levels. Formally, $\exprob \in \mathcal{P}(\exprod)$, where $\exprod = \dom{\exbase} \times \dom{\exabst}$ denotes the product measure space of the exogenous domains. In practice,  the realization of certain environments is infeasible or unrealistic. Hence, we introduce the notion of \emph{relevant environments}, denoted $\stablebase \subseteq \cP(\exbase)$ and $\stableabst\subseteq \cP(\exabst)$ for the low- and high-level SCMs, respectively. Accordingly, we define the \emph{relevant joint environment space} as $\stableprod= \stablebase \otimes \stableabst$, which represents all joint distributions formed by pairing any $\exprobbase \in \stablebase$ with any $\exprobabst \in \stableabst$. This mirrors the concept of \emph{relevant interventions} $\cI$ that is a partially ordered set that restricts attention to interventions that are semantically meaningful or practically implementable. Furthermore, to align interventions across abstraction levels, we assume the existence of a surjective and order-preserving map $\omega:\intervsetbase \to \intervsetabst$. To ensure consistent alignment of joint interventions across levels, we define $\intervsetprod =  \{(\iota, \omega(\iota)) \mid \iota \in \intervsetbase\}$, pairing each low-level intervention with its high-level counterpart. Finally, we define the \emph{abstraction context} as the pair $(\stableprod, \intervsetprod)$, which specifies the set of interventions and environments over which an abstraction is evaluated. Building on this setup, we introduce a refined notion of abstraction that operates on the reduced form of the SCM, making explicit the dependence of the abstraction on both the environment spaces and the intervention sets associated with the low- and high-level models.

\begin{defn}[$\rhoiota$-Abstraction]\label{def:rhoiota}
Let $(\stableprod, \intervsetprod)$ be an abstraction context, where $\scmbase = (\cS^\ell, \exprobbase)$, $\scmabst = (\cS^h, \exprobabst)$ the low-level and high-level SCMs with $\ell \ge h$, and let $\mixfbase$ and $\mixfabst$ denote the respective mixing functions of their reduced forms. Then a map $\tau: \dom{\enbase} \to \dom{\enabst}$ is called a \emph{$\rhoiota$-abstraction} if, for all $\exprob=\exprobbase\otimes\exprobabst \in \stableprod$ and for all $\boldsymbol{\iota} = (\iota,\omega(\iota)) \in \intervsetprod$:
\begin{equation}\label{eq:tau_rho_def}
    \tau_{\#} \left(\mixfbase_{\iota \#}(\exprobbase)\right) = \mixfabst_{\omega(\iota)\#}(\exprobabst)
\end{equation}
\end{defn}
A $\rhoiota$-abstraction ensures commutativity between interventions and abstractions under a given joint environment $\exprob = \textcolor{blue}{\exprobbase} \otimes \textcolor{orange}{\exprobabst}$: {\color{teal} intervening via $\iota$ and then abstracting via $\tau$} yields the same result as {\color{violet} abstracting first and then intervening via $\omega(\iota)$} $\forall~ \iota \in \intervsetbase$. In general, we say that $\scmabst$ and $\scmbase$ are \emph{$\rhoiota$-interventionally consistent} if there exists an abstraction context $(\stableprod, \intervsetprod)$ such that Eq.~\ref{eq:tau_rho_def} holds. Furthermore, while the underlying SCMs may be non-linear, we focus on the case of \emph{linear abstractions} \citep{massidda2024learningcausalabstractionslinear} between them, an important and well-behaved class of abstractions according to which the abstraction map can be represented through a matrix $\lintau \in \bR^{h\times \ell}$ and thus the map $\tau$ can be represented as a matrix-vector multiplication $\tau(x) = \lintau x~\in X^h,~x \in X^\ell$.
\begin{figure}[htp]
    \centering
    \resizebox{\columnwidth}{!}{%
    \begin{tikzpicture}[->, thick]

        \fill[blue!5] (-15, -1.5) rectangle (10, -5.2);
        \node[blue, anchor=east, font=\bfseries\huge] 
            at (-15, -3.75) {\huge  $\scmbase$};

        \fill[orange!5] (-15, -5.2) rectangle (10, -8.5);
        \node[orange, anchor=east, font=\bfseries\huge] 
            at (-15, -6.85) {\huge $\scmabst$};


        \node[blue, opacity=0.8] (A) at (-4.2, -3) { \huge $\exprobbase$};
        \node[orange, opacity=0.8] (B) at (-3.2, -3.7) { \huge $\exprobabst$};
        \node[blue] (D) at (-12.2, -4.5) { \Huge $\prob_{\scmbase}$};
        \node[blue] (E) at (4, -4.5) {\Huge $\prob_{\scmbase_\iota}$};
        \node[orange] (C) at (-8.2, -7) {\Huge $\prob_{\scmabst}$};
        \node[orange] (F) at (2, -7) {\Huge $\prob_{\scmabst_{\omega(\iota)}}$};
        \node[orange] (G) at (7.5, -7) {\Huge $\taupush(\prob_{\scmbase_{\iota}})$};

        \node[opacity=0.9] at (-3.8, -3.5) {\huge $\otimes$};

        \draw[blue, opacity=0.5, bend right=10] (A) to (D);
            \node[blue, opacity=0.6] at (-10.5, -3.4) {\huge $\mixfbase_{\#}$};
        \draw[blue, opacity=0.5, bend left=10] (A) to (E);
            \node[blue, opacity=0.6] at (2.4, -3.4) {\huge $\mixfbase_{\iota\#}$};
        \draw[orange, opacity=0.5, bend right=20] (B) to (C);
            \node[orange, opacity=0.6] at (-7.65, -5.5) {\huge $\mixfabst_{\#}$};
        \draw[orange, opacity=0.5, bend left=20] (B) to (F); 
            \node[orange, opacity=0.6] at (1.9, -5.5) {\huge $\mixfabst_{\omega(\iota )\#}$};      

        \node[violet] at (-10.7, -5.9) {\huge $\mathbf{\taupush}$};
        \node[teal]   at (5.9, -5.5)  {\huge $\mathbf{\taupush}$};
        
        \draw[teal,  line width=2.0pt]   (E) to (G);
        \draw[violet,  line width=2.0pt] (D) to (C);

        \node[teal] (i) at (-7.4, -4.15) {\Huge $\iota$};
        \draw[teal,  line width=2.0pt] (D) to (E);
        
        \draw[violet,  line width=2.0pt] (C) to node[above]{\Huge $\omega(\iota)$} (F);
        \draw[dotted,|-|,  line width=2.0pt] (G) to node[midway, above]{\huge $ e_{\tau}^{\exprob, \iota}$} (F);

    \end{tikzpicture}%
    }
    \caption{\emph{Computation of the environment–intervention error.} The joint environment $\exprob = {\color{blue}\exprobbase} \otimes {\color{orange}\exprobabst}$ captures the combined uncertainty from the {\color{blue}low}- and {\color{orange}high}-level SCMs. By pushing forward these components through the reduced forms $\mixfbase$ and $\mixfabst$ of the respective SCMs, we evaluate two interventional pathways: {\color{teal}(a) apply an intervention $\iota$ to $\scmbase$, then map the resulting distribution to the high-level space via $\taupush$~}: {\color{orange}$\taupush(\prob_{\scmbase_{\iota}})$}; and {\color{violet}(b) first map $\scmbase$ to $\scmabst$ via $\taupush$, then apply the corresponding intervention $\omega(\iota)$}: {\color{orange}$\prob_{\scmabst_{\omega(\iota)}}$}. The divergence $\mathcal{D}_{\enabst}$ between the resulting interventional distributions computes $e_\tau^{\exprob, \iota}$. Aggregating $e_\tau^{\exprob, \iota}$ over an $(\stableprod, \intervsetprod)$ abstraction context, recovers the $(\stableprod, \intervsetprod)$–abstraction error (Eq.~\ref{eq:totabst_error}). If zero, the diagram commutes and $\tau$ defines a $\tau$–$0$–approximate abstraction.}
    \label{fig:tau_approx}
\end{figure}
Our $\rhoiota$-framework generalizes and bridges prior notions of causal abstraction. Unlike exact transformations~\citep{rubenstein2017causal}, which assume consistency under some fixed (but unspecified) environment, and uniform transformations~\citep{beckers2018abstracting}, which require consistency across all environments, $\rhoiota$ allows for a range of finite, plausible environments. This enables principled abstraction under realistic data constraints. Further discussion on this is in Appendix~\ref{sec:caapp}.
\newline
\textbf{The $(\stableprod, \intervsetprod)$--Abstraction Error. } Def.~\ref{def:rhoiota} suggests a perfect form of abstraction where interventional consistency holds exactly, implying that under a given abstraction context $(\stableprod, \intervsetprod)$ using the high-level model entails no loss in predictive accuracy. However, in reality, such types of abstractions are rare. This stems in part from the very nature of abstraction, which reduces the size of the representation by disregarding minor differences, often leading to some degree of information loss. More importantly, in CAL, where the goal is to learn abstractions from data, all inferences are inherently approximate. Thus we need a notion of \emph{abstraction error} \citep{beckers2020approximate} tailored to our framework. The core idea is to measure the discrepancy between the two sides of Eq.~\ref{eq:tau_rho_def} while accounting for the given context $(\stableprod, \intervsetprod)$. To this end, we assume access to a metric over high-level interventional distributions, which induces a discrepancy between SCMs via  $\tau$ enabling approximate abstractions. Specifically, we introduce a novel notion of abstraction error by aggregating over environments and interventions.

\begin{defn}\label{def:absterror} Let $\tau: \dom{\enbase} \to \dom{\enabst}$ be a map between the SCMs $\scmbase$ and $\scmabst$, defined wrt context $(\stableprod, \intervsetprod)$. Let $D_{\enabst}$ also be a discrepancy measure over high-level distributions, and $q$ be a distribution over interventions in $\intervsetbase$. We define the \emph{$(\stableprod, \intervsetprod)$-Abstraction Error} as:
\begin{equation}\label{eq:totabst_error}
    e_{\tau}(\scmbase, \scmabst) =
    \underset{\exprob \in \stableprod}{\mathfrak{g}}~
    \underset{\iota \sim q}{\mathfrak{h}}~
    \left[  e_{\tau}^{\exprob, \iota}(\scmbase, \scmabst) \right],
\end{equation}
where we define the \emph{environment–intervention} approximation error induced by $\tau$ for a fixed  $\exprob \in \stableprod$ and intervention $\iota \in \intervsetbase$ as $e_{\tau}^{\exprob, \iota}(\scmbase, \scmabst) = \mathcal{D}_{\enabst} \left( \tau_{\#}(\mixfbase_{\iota \#}(\exprobbase)),~ \mixfabst_{\omega(\iota)\#}(\exprobabst) \right)$.\end{defn}
The operators $\mathfrak{g}$ and $\mathfrak{h}$ aggregate the error over environments and interventions, respectively. Prior work typically fixes the environment, aggregating only over interventions via expectation \citep{felekis2024causal, dyer2023interventionally, kekić2023targeted} or supremum \citep{zennaro2023jointly}. \citet{beckers2020approximate} proposed a similar notion of abstraction error, though their framework is defined under uniform abstractions and does not consider restrictions to subsets of the environment space, as we do. We call $\scmabst$ an \emph{approximate $\rhoiota$-abstraction} if $e_{\tau}(\scmbase, \scmabst) > 0$ (see Figure~\ref{fig:tau_approx}). The connection between Def.~\ref{def:absterror} and $\rhoiota$-abstractions now becomes apparent. In particular, when the total abstraction error vanishes for a given $\tau$, the pushforward of the low-level interventional distributions $\mixfbase_{\iota \#}(\exprobbase)$ through $\taupush$ matches the corresponding high-level ones $\mixfabst_{\omega(\iota)\#}$, $q$-almost surely. The statement and proof of this are presented in Proposition~\ref{prop:consistency_metric_abstraction}, Appendix~\ref{sec:thmsproofsapp}.

\section{Learning Distributionally Robust Causal Abstractions}\label{sec:main_diroca_learning}
We learn \emph{robust causal abstractions} when data is available from \emph{a single environment per SCM}. Consider a pair $(\scmbase, \scmabst)$ of SCMs and their joint environment $\exprob$ as before, representing the unknown data-generating process. We learn an abstraction map $\tau$ that reliably transforms low-level distributions into high-level ones, even in the presence of distributional shifts or test-time noise. We build on Wasserstein \emph{Distributionally Robust Optimization (DRO)}~\citep{kuhn2019wassersteindistributionallyrobustoptimization}, which provides robustness against distributional misspecification. The standard objective is:
\begin{align}
    \inf_{x \in \mathcal{X}} \sup_{\mathbb{Q} \in \mathbb{B}_{\varepsilon, p}(\widehat{\mathbb{P}}_N)} \mathbb{E}_{\xi \sim \mathbb{Q}}[f(x, \xi)],
\end{align}
The goal is to minimize the worst-case expected loss \(f(x, \xi): \mathbb{R}^n \times \Xi \rightarrow \mathbb{R}\), where \(x \in \mathbb{R}^n\) denotes a decision variable, \(\xi \in \Xi = \mathbb{R}^m\) a random vector representing uncertain data, over a set of plausible distributions, called the \emph{ambiguity set}. In our setting, the loss is the \emph{environment-intervention error} $e_{\tau}^{\exprob, \iota}(\scmbase, \scmabst)$, the decision variable is $\tau$, and we minimize the worst-case abstraction error over a \emph{2-Wasserstein product ambiguity set} centered at the empirical joint environment $\widehat{\exprob}$ with radius $\radprod$. More details on DRO are provided in Appendix \ref{sec:droapp}.
\newline
\textbf{Environmental Robustness via $\rhoiota$-Abstractions. }We introduce distributional robustness by setting the aggregation function $\mathfrak{g}$ in Eq. \ref{eq:totabst_error} over the environments to a supremum, and for computational practicality, the aggregation function $\mathfrak{h}$ over the interventions to an expectation. Following the DRO paradigm, and since we assume access to a single environment from each SCM, we substitute $\cP(\stableprod)$ with a \emph{joint ambiguity set} over the low- and high-level environments to incorporate distributional uncertainty in a principled way using a 2-Wasserstein product ball centered at the empirical joint environment $\widehat{\exprob}$. We denote this ambiguity set by $\bB_{\radprod, 2}(\widehat{\exprob})$, which contains all the environments that lie within a radius $\radprod$, determined by the individual radii $\radbase$ and $\radabst$, from the nominal empirical components. Formally, it is given by $\bB_{\radprod, 2}(\widehat{\exprob}):= \mathbb{B}_{\radbase, 2}(\widehat{\exprobbase}) \times \mathbb{B}_{\radabst, 2}(\widehat{\exprobabst})$,
where $\mathbb{B}_{\varepsilon_d, 2}(\widehat{\rho^d}) = \left\{ \rho \in \mathcal{P}(\ex^d) : W_2(\rho, \widehat{\rho^d}) \leq \varepsilon_d \right\}$ for $d \in \{\ell, h\}$, is a $2$-Wasserstein ball centered at $\widehat{\rho^d}$. The radius $\varepsilon_d$ controls the robustness level: larger values provide protection against broader shifts but may yield more conservative solutions. Each marginal Wasserstein ball bounds the allowable deviation from the empirical environment for each SCM. As we have sample access to the endogenous variables, we estimate the exogenous environments via abduction by inverting the reduced form of the SCM, denoted as $(\mathbf{g}^d_{\#})^{-1}$. In cases where structural functions are known, this inversion is exact; else, they must be inferred, and the abduction step recovers an approximate exogenous distribution. The role of the joint ambiguity set $ \bB_{\radprod, 2}(\widehat{\exprob})$ is to identify the worst-case joint environment $\exprob^\star = \rho^{\ell \star} \otimes \rho^{h\star}$ within the product ball that maximizes the abstraction error. This captures the most adversarial shift consistent with the estimation error. The full construction process is illustrated in Fig.~\ref{fig:ambsetconstructBOX} of Appendix \ref{sec:ambsetapp} for the case of independently sampled environments.
\newline
\textbf{The \textsc{DiRoCA} Objective. }Our objective is to learn a \emph{linear abstraction map} $\tau$, that remains reliable across all shifts inside $\bB_{\radprod, p}(\widehat{\exprob})$. Assuming access to an abstraction context $(\stableprod, \intervsetprod)$ and a divergence measure $\mathcal{D}_{\enabst}$ over high-level interventional distributions, we define the \emph{distributionally robust causal abstraction learning objective} as follows:
\begin{align}\label{eq:final_objective}
 \min_{\lintau\in \bR^{h\times \ell}}\sup_{\exprob \in \bB_{\radprod, 2}(\widehat{\exprob})}~~\underset{\iota \sim q}{\expval}\left[e_{\tau}^{\exprob, \iota}(\scmbase, \scmabst)\right]
\end{align}    
where $\lintau\in \bR^{h\times \ell}$ denotes the class of linear abstraction maps $\tau$, and $e_{\tau}^{\exprob, \iota}(\cdot, \cdot)$ is the environment–intervention error. This objective seeks the linear abstraction map that minimizes the worst-case expected abstraction error over the joint ambiguity set, ensuring robustness to distributional shifts. 
\begin{remark}\label{rmk:epsilonandcas}
Both exact and uniform abstractions arise as special cases of our framework. Specifically, the robustness radius $\radprod$ flexibly interpolates between these extremes depending on the desired level of robustness; i.e. $\radprod\!\to\!0$ approximates a $\tauomega$-transformation, while $\radprod\!\to\!\infty$ approximates a uniform abstraction.
\end{remark}

\vspace{-0.5em}
\textbf{Concentration guarantees for the joint environment.} 
Our theoretical results below show that the true joint environment $\exprob$ lies within a 2-Wasserstein product ball centered at $\widehat{\exprob}$ with high probability for an appropriate radius $\radprod$. \emph{This result offers a principled way to set the robustness parameter $\radprod$ to ensure the ambiguity set reliably covers the true environment}, thus justifying our distributionally robust CA objective for independent environments.
\newline
\textbf{Assumption 1.} For $d \in \{\ell, h\}$, $\rho^d$ is a light-tailed environment; i.e. there exist constants $\alpha > 0$ and $A > 0$ such that $\mathbb{E}^{\rho^d}\left[\exp\left(\|\xi\|_2^\alpha\right)\right] \leq A$, where $\xi \sim \rho^d$. 

\begin{theorem}[Empirical $\exprob$-Concentration]\label{thm:ca_product_concentration_emp}
Let $\widehat{\exprobbase}$ and $\widehat{\exprobabst}$ be empirical distributions, under Assumption 1, from $N_\ell$ and $N_h$ i.i.d. samples, with $\exprob \coloneq \exprobbase \otimes \exprobabst$ and $\widehat{\exprob} \coloneq \widehat{\exprobbase} \otimes \widehat{\exprobabst}$. For $d\in \{\ell, h\}$, there exist constants $c_{d,1}, c_{d,2} > 0$ depending only on the $d$-environment and confidence levels $\eta_d$. For any $\delta \in (0,1]$ with $\delta = 1-(1-\eta_\ell)(1-\eta_h)$, let $N_d(c,\eta) = \log(c_{d,1}/\eta)/c_{d,2}$. If we define:
\begin{equation*}
    \Tilde{\varepsilon}_d = 
    \begin{cases} 
        \left( \frac{\log(c_{d,1}/\eta)}{c_{d,2} N_d} \right)^{\min\{1/d, 1/2\}} & \text{if } N_d \geq N_d(c,\eta), \\
        \left( \frac{\log(c_{d,1}/\eta)}{c_{d,2} N_d} \right)^{1/\alpha_d} & \text{otherwise,}
    \end{cases}
\end{equation*}
then $\forall~\radprod \ge \sqrt{\Tilde{\varepsilon}_\ell^2 + \Tilde{\varepsilon}_h^2} \implies \mathbb{P}\left[ \mathcal{W}_2(\exprob, \widehat{\exprob}) \le \radprod \right] \ge 1 - \delta$.
\end{theorem}
This concentration bound for CA shows how the robustness radius $\radprod$ contracts with sample size to ensure high probability coverage of the true environment. In particular, \(\radprod = \mathcal{O}\left(N^{-1/d}\right)\) for dimension \(d > 2\), where \(N = \min(N_\ell, N_h)\) and \(d = \max\left(\ell, h \right)\). Under finite samples, setting $\epsilon > 0$ yields a stronger abstraction, valid over the entire Wasserstein ball. An equivalent result for the case of Elliptical/Gaussian measures, alongside the proofs, can be found in Appendix~\ref{sec:thmsproofsapp}.
\begin{remark}
When coefficients are estimated rather than known, the nominal environment $\widehat{\exprob}$ may be misspecified; however, as the sample size grows, this error vanishes, and the concentration bounds remain asymptotically valid.
\end{remark}

\vspace{-0.5em}
\textbf{Optimization.}
We now demonstrate how to solve the robust CA problem for ANMs as a min–max optimization. Given samples $X^{d, \iota} \in \mathbb{R}^{N \times d}$ for $\iota \in \cI^d$, we extract environmental observational samples $U^d \in \mathbb{R}^{N \times d}$ via abduction. During training, reusing these observational residuals for interventional generation yields a consistent pairing of batch rows across interventions. This is not an alignment assumption but a computational device enabling Frobenius-based discrepancy evaluation and avoiding repeated optimal transport computations (see Appendix~\ref{sec:thmsproofsapp}). Under the ANM decomposition (Section~\ref{sec:main_back}), the endogenous samples for an intervention $\iota \in \intervsetbase$ and level $d \in \{\ell,h\}$ are $X^{d,(\iota)} = D^{d, (\iota)} + U^d$. We define the empirical marginal environments as $\widehat{\rho^d}= 1/N\sum_{i=1}^{N} \delta_{\widehat{u}^d_i}$ for $d\in\{\ell,h\}$, and the empirical joint environment as $\widehat{\exprob}= \widehat{\exprobbase} \otimes \widehat{\exprobabst}$. By the finite-dimensional reduction of \citet{kuhn2019wassersteindistributionallyrobustoptimization}, any distribution in a 2-Wasserstein ball around an empirical measure can be represented as a perturbed empirical distribution. We thus perturb the noise samples as $U^d \mapsto U^d + \Theta_d$, where $\Theta_d \in \mathbb{R}^{N \times d}$, which induces perturbed marginals $\widehat{\rho^d}(\Theta_d)
:= 1/N\sum_{i=1}^{N} \delta_{\widehat{u}^d_i + \theta^d_i}$. The ambiguity radii $\varepsilon_\ell,\varepsilon_h$ correspond to Frobenius budgets $r_d = \varepsilon_d \sqrt{N}$, such that the condition $\|\Theta_d\|_F \le r_d, d \in \{\ell,h\}$ is equivalent to restricting the joint environment to $\bB_{\radprod,2}(\widehat{\exprob})$. To formulate the robust objective, we define the \emph{nominal misalignment} for an abstraction $\lintau \in \mathbb{R}^{h \times \ell}$ as $Z_{\lintau}^{\iota}(\mathbf{0}) \coloneq \lintau(D_\ell^{(\iota)} + U^\ell)^\top - (D_h^{(\omega(\iota))} + U^h)$. Introducing perturbations $\mathbf{\Theta} \coloneq (\Theta_\ell, \Theta_h)$, the \emph{perturbed misalignment} shifts to $Z_{\lintau}^{\iota}(\mathbf{\Theta}) \coloneq Z_{\lintau}^{\iota}(\mathbf{0}) + (\lintau\Theta_\ell^\top - \Theta_h)$. \emph{The DiRoCA objective minimizes the expected perturbed misalignment under the worst-case observational shift within the ambiguity set}:
\begin{align}\label{eq:minmax_general}
    \min_{\lintau} \sup_{\substack{\|\Theta_\ell\|_F \le r_\ell,~ \|\Theta_h\|_F \le r_h}}
    \mathbb{E}_{\iota\sim q}\!\left[
    \|Z_{\lintau}^{\iota}(\mathbf{\Theta})\|_F^2
    \right].
\end{align}
This formulation highlights that the adversary attempts to align the perturbation shift $(\lintau\Theta_\ell^\top - \Theta_h)$ with the nominal direction $Z_{\lintau}^{\iota}(\mathbf{0})$ to maximize the error. We solve~\eqref{eq:minmax_general} via alternating projected gradient descent–ascent (Alg.~\ref{alg:general_diroca}). This constitutes our primary method, applicable to \emph{any} ANM. When the underlying SCMs are Linear (LANs), the mixing functions of the reduced forms are linear operators $\mixbase_{\iota}\in\bR^{\ell \times \ell}, \mixabst_{\omega(\iota)} \in \bR^{h \times h}, \forall \iota \in \intervsetbase$. We distinguish two variations based on the available environmental information: \textbf{(a)} When only samples are available (\emph{Empirical}), we utilize the linear mixing matrices explicitly. The misalignment becomes a direct matrix operation: $Z_{\lintau}^{\iota}(\mathbf{0}) = \lintau \mixbase_{\iota} U^{\ell \top} - \mixabst_{\omega(\iota)} U^h$.The objective remains the Frobenius norm minimization from Eq.~\eqref{eq:minmax_general}, but the explicit linearity allows for faster gradient computations as the deterministic part $D$ is folded into the matrix multiplication. The optimization follows Algorithm~\ref{alg:general_diroca}, projecting perturbations $\Theta$ onto Frobenius balls. \textbf{(b)} If the environments are known to be \emph{Gaussian}, i.e., $\rho^d \sim \mathcal{N}(\mu^d, \Sigma^d)$, we exploit the property that Gaussianity is preserved under the linear transformations $\mixbase_{\iota}$ and $\mixabst_{\omega(\iota)}$ and $\lintau$. Here, instead of explicit sample perturbation, we define the ambiguity set directly over the Gaussian parameters $(\mu, \Sigma)$ using the Gelbrich distance (closed-form $\mathcal{W}_2$ for Gaussians see Appendix~\ref{sec:wassapp}). The robust objective simplifies to minimizing the worst-case Gelbrich distance misalignment $Z_{\lintau}^{\iota}(\mathbf{0})$ between $\mathcal{N}(\lintau \mixbase_\iota \mu^\ell, \lintau \mixbase_\iota \Sigma^\ell \mixbase_\iota^\top \lintau^\top)$ and $ \mathcal{N}(\mixabst_{\omega(\iota)} \mu^h, \mixabst_{\omega(\iota)} \Sigma^h \mixabst_{\omega(\iota)}^\top)$. We optimize via a proximal gradient ascent–descent scheme on the moments $(\mu, \Sigma)$, relaxing the non-smooth trace term via a variational upper bound. Illustrations of the resulting worst-case Gaussian environments appear in Appendix~\ref{sec:optapp}. All DiRoCA variants and their optimization details appear in Appendix~\ref{sec:optapp}. 
\begin{algorithm}[t]
\caption{General \textsc{DiRoCA} $~~~~~~~~~~~~~~~~~~~~~~~~[d \in \{\ell, h\}]$}\label{alg:general_diroca}
\begin{algorithmic}[1]

\REQUIRE $\omega, \mathcal{I}^d, \scmdag_{\scmsimple^d},
\{\mathbf{X}_\iota^d \in \mathbb{R}^{N \times d}\}_{\iota \in \mathcal{I}^d}, \varepsilon_d$ (Thm.~\ref{thm:ca_product_concentration_emp}).

\STATE \textit{// Infer noise and initialize empirical environments}
\STATE $\mathbf{U}^d \leftarrow \text{Abduct}(\mathbf{X}^d, \mathcal{M}^d)$;~~~$\widehat{\rho^d} \leftarrow \frac{1}{N}\sum_{i=1}^N \delta_{\widehat{u}^d_i}$

\STATE \textbf{Initialize:} $\lintau^{(0)}$, $\mathbf{\Theta}^{(0)}$,
$\rho^{d(0)} \leftarrow \widehat{\rho^d}$, $r_d \leftarrow \varepsilon_d \sqrt{N}$

\STATE $J(\lintau, \mathbf{\Theta}) \coloneq \mathbb{E}_{\iota}[\|Z_{\lintau}^{\iota}(\mathbf{\Theta})\|_F^2]$
\AlgComment{\#Objective}

\STATE \textbf{repeat until convergence:}

\STATE \hspace{0.5em}\textit{// Adversarial Projected Gradient Ascent}
\STATE \hspace{0.5em}$\Theta_{d}^{(t+1)} \leftarrow
\text{Proj}_{\|\cdot\|_F \le r_d}\!\left[\Theta_{d}^{(t)} + \eta_\theta \nabla_{\Theta_d}
J(\lintau^{(t)}, \mathbf{\Theta}^{(t)})\right]$

\STATE \hspace{0.5em}$\rho^{d(t+1)} \leftarrow \frac{1}{N}\sum_{i=1}^N
\delta_{\widehat{u}^d_i + \theta^{d(t+1)}_i}$
\AlgComment{\#Update worst-case}

\STATE \hspace{0.5em}\textit{// Abstraction Gradient Descent}
\STATE \hspace{0.5em}$\lintau^{(t+1)} \leftarrow \lintau^{(t)} -
\eta_\tau \nabla_{\lintau} J(\lintau^{(t)}, \mathbf{\Theta}^{(t+1)})$ \AlgComment{\#Update $\lintau$}

\textbf{Return:} $\lintau^\star \in \bR^{h \times \ell}$, $\exprob^\star = \rho^{\ell\star} \otimes \rho^{h\star}$

\end{algorithmic}
\end{algorithm}
Furthermore, we establish worst-case guarantees for the learned $\lintau$. While certified defenses in robust prediction typically bound a classification radius \citep{cohen2019certifiedadversarialrobustnessrandomized} or outlier influence \citep{robustGPs}, our framework certifies the abstraction error itself. Theorem~\ref{thm:provable-robustness-main} establishes that the solution $\lintau^\star$ of Eq.~\ref{eq:final_objective} is provably robust, offering a closed form bound on the abstraction error for all perturbations subject to the Frobenius constraints.
\begin{theorem}[Provable Robustness]\label{thm:provable-robustness-main}
Let $\lintau\in\mathbb{R}^{h\times \ell}$ be an abstraction matrix and define its worst-case expected loss as: $\zeta(\lintau)
\coloneq \sup_{\|\Theta_\ell\|_F \le r_\ell,\;\|\Theta_h\|_F \le r_h}
\mathbb{E}_{\iota\sim q}\!\left[
\|Z_{\lintau}^{\iota}(\mathbf{\Theta})\|_F^2
\right]$.
\begin{align}
\implies \zeta(\lintau)
\leq
\mathbb{E}_{\iota\sim q}
\Big[
\big( r_h + \|Z_{\lintau}^{\iota}(\mathbf{0})\|_F + r_\ell\|\lintau\|_2 \big)^2
\Big].
\end{align}
Consequently, the worst-case abstraction error of any minimizer $\lintau^\star$ of~\eqref{eq:minmax_general} is upper bounded by the RHS at $\lintau^\star$.
\end{theorem}

\section{Experimental Results}\label{sec:main_exps}
We examine our method's robustness to environmental shifts and misspecification versus prior art and baselines across four settings: \textbf{(a)} the \texttt{SLC} dataset, a three-variable LAN SCM whose abstraction removes a mediator; \textbf{(b)} \texttt{LUCAS} \href{https://www.causality.inf.ethz.ch/data/LUCAS.html}: a lung cancer diagnosis simulation ($\ell=6, h=3$) in linear (\texttt{LiLUCAS}) and nonlinear (\texttt{nLUCAS}) variants; \textbf{(c)} the real-world \textit{Electric Battery Manufacturing} (\texttt{EBM}) \citep{zennaro2023jointly} dataset; and \textbf{(d)} the semi-synthetic \textit{Colored MNIST} (\texttt{cMNIST}) \citep{xia2024neural} dataset, where we abstract from $32 \times 32$ RGB images to a $64$-dimensional disentangled latent space defined by the pre-trained encoder of the original work, which serves as the ground truth. We assume known DAGs, intervention sets, and map $\omega$; structural functions are known for synthetic tasks but estimated for \texttt{EBM} and \texttt{cMNIST}. Complete results, evaluation details, a roadmap of experiments, and datasets' analysis is in Appendices~\ref{sec:all_results_settings} and \ref{sec:examples_app}. We compare \textsc{DiRoCA} against the following methods: 
\textbf{(a)} \textsc{Bary}$_{\tauomega}$, a variation of the \textsc{Bary}$_{\text{OT}}$ of \citep{felekis2024causal}, which learns an abstraction between the barycenters of the interventional distributions of both abstraction levels; \textbf{(b)} \textsc{Grad}$_{\tauomega}$ that ignores environmental uncertainty and approximates a $\tauomega$ transformation via gradient descent; corresponding to \textsc{DiRoCA} with $(\radbase, \radabst)=(0,0)$ as discussed in Remark~\ref{rmk:epsilonandcas}, and for the empirical settings, we also compare against; \textbf{(c)} \textsc{Abs-LiNGAM} \citep{massidda2024learningcausalabstractionslinear}: an observational least-squares method between low- and high-level samples baseline, evaluating both perfect ($\absp$) and noisy ($\absn$) CAL variants proposed in the original work.
\newline
\textbf{Evaluation.} We report \textsc{DiRoCA}$_{\widehat{\epsilon}}$ (radii via Thm.~\ref{thm:ca_product_concentration_emp}) and \textsc{DiRoCA}$_{\epsilon^\star}$ (best-performing in each setting). Results for additional robustness radii are deferred to Appendix~\ref{sec:all_results_settings}. For nonlinear SCMs, we make no parametric assumptions on the environment. For linear SCMs, we evaluate both empirical and Gaussian environments, always using the empirical approximation of the $(\stableprod,\intervsetprod)$-abstraction error in Eq.~\eqref{eq:totabst_error}. We instantiate $\mathfrak{g}$ and $\mathfrak{h}$ as $\expval$ and report mean $\pm$ std abstraction error, with paired $t$-tests ($p < 0.05$). We use $5$-fold cross-validation for all experiments, except for \texttt{EBM} where a $90/10$ train–test split is used due to limited data. Structural functions are known for synthetic datasets and estimated for \texttt{EBM} and \texttt{cMNIST} (details in Appendix~\ref{sec:examples_app}). For \texttt{cMNIST}, robustness is enforced only in the low-level/pixel space, reflecting the sensory nature of the shifts; i.e. $\radabst=0$. Finally, at test time, abstractions are always evaluated on independently generated samples and not on the shared observational noise across interventions as in training. 

We perform a robustness analysis using a unified evaluation framework based on a Huber contamination model. This is applied to the held-out test set for each of the $k$ cross-validation folds, controlled by two key parameters: the contamination fraction $\alpha \in [0, 1]$, which controls the \textit{prevalence} of shifted samples, and the contamination strength $\tilde{\sigma}$, which controls the magnitude of their shift. Each test set consists of a collection of paired, clean endogenous datasets corresponding to each intervention $\{ (X^\ell_{ \iota}, X^h_{\omega(\iota)} )\}_{\iota \in \mathcal{I}_\ell} \in \mathbb{R}^{N^\ell_{\text{test}} \times \ell} \times \mathbb{R}^{N^h_{\text{test}} \times h}$. For each intervention $\iota$, we first generate a noisy version of the data, $\tilde{X}^d_{\iota}$, by applying an additive stochastic shift to the clean data. This is achieved by creating a noise matrix $\mathbf{N} \in \mathbb{R}^{N^d_{\text{test}} \times d}$, where each row $\mathbf{n}_i \sim \mathcal{N}(0, \tilde{\sigma}^2 \cdot \mathbf{I})$ is an independent sample from a zero-mean Gaussian distribution with a scaled covariance and thus, $\tilde{X}^d_{\iota} = X^d_{\iota} + \mathbf{N}$. While we focus on Gaussian shifts in the main text for clarity, we present analogous results for Student-t and Exponential noise distributions in Appendix~\ref{sec:all_results_settings}. The final contaminated test sets, $\bar{X}^d_{\iota}$ for every $\iota \in \cI^d$, are then formed as a convex combination:
$\bar{X}^d_{\iota} = (1-\alpha)X^d_{\iota} + \alpha \tilde{X}^d_{\iota}$. Although DiRoCA targets environmental shifts, we contaminate the endogenous samples directly, which is a valid proxy as additive exogenous shifts in ANMs propagate additively to endogenous variables. Performance is measured by the squared Frobenius norm for empirical settings and the Wasserstein distance between fitted Gaussians $\mathcal{N}(\hat{\mu}, \hat{\Sigma})$ for linear settings, averaged across all folds, samples, and interventions.
\newline
\textbf{Analysis.}
We analyze robustness by varying the outlier fraction $\alpha$ at fixed noise intensity $\tilde{\sigma}$ (Fig.~\ref{fig:alpha_main_curves}). Consistent trends are observed when varying noise intensity $\tilde{\sigma}$ at fixed $\alpha$. In linear synthetic settings (\texttt{SLC}, \texttt{LiLUCAS}), where the noise is Gaussian (explicitly or implicitly), we observe a clear crossover: at low contamination ($\alpha \approx 0$), the non-robust \textsc{Grad}$_{\tauomega}$ marginally outperforms \textsc{DiRoCA}, reflecting the \emph{cost of robustness}: \textsc{DiRoCA}'s adversarial training is conservative, prioritizing worst-case safety over best-case precision. However, as $\alpha$ increases, the error of non-robust baselines grows rapidly, whereas \textsc{DiRoCA} maintains a lower error profile under increasing shifts. In \texttt{nLUCAS} setting, this trade-off collapses: \textsc{DiRoCA} outperforms \textsc{Grad}$_{\tauomega}$ even at $\alpha=0$. This occurs because nonlinear mechanisms induce non-Gaussian data geometries, causing \textsc{Grad}$_{\tauomega}$ to overfit the nominal environment and fail to generalize to new samples, whereas \textsc{DiRoCA}'s environment-level optimization regularizes against this effect (Note in Appendix~\ref{sec:all_results_settings}). For the real-world datasets (\texttt{EBM}, \texttt{cMNIST}), we report results for $\alpha=1$ in the main text (Table~\ref{tab:battery_cmnist_unified_alpha1}), with full robustness curves deferred in Appendix~\ref{sec:all_results_settings}. Here, the structural functions estimation introduces model misspecification in addition to stochastic noise. However, \textsc{DiRoCA} treats these structural residuals as a form of environmental perturbation: the same robustness mechanism that guards against distributional shift also absorbs deviations from the assumed functional form, a perspective aligned with recent approaches that model perturbations as sparse mechanism shifts \citep{schneider2025generativeinterventionmodelscausal} and explains why it maintains stable performance even under severe shifts ($\alpha=1$). Also, \textsc{Bary}$_{\tauomega}$ shows limited robustness via aggregation but lacks the adversarial training needed to withstand worst-case shifts, while \textsc{AbsLin} fails due to observational reliance, confirming that nominal-environment tailored optimization fails to generalize under contamination.
\begin{figure*}[t]
    \centering
    \includegraphics[width=.8\textwidth]{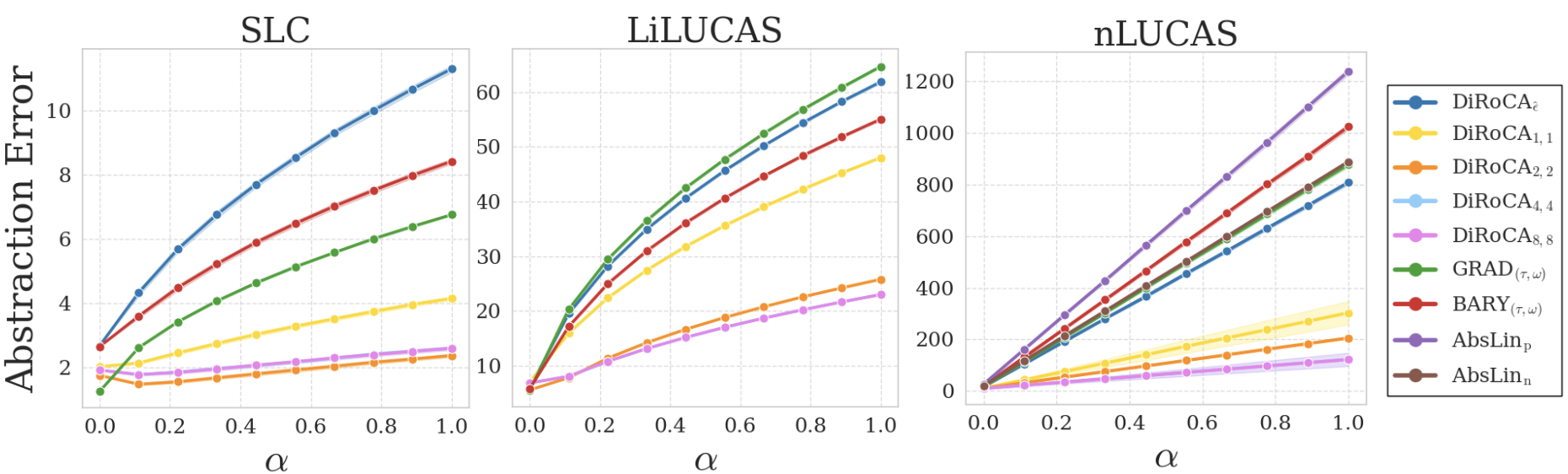}
    \caption{Robustness to outlier fraction ($\alpha$) on the \texttt{SLC} (Gaussian) and \texttt{LiLUCAS} (Gaussian) and \texttt{nLUCAS} experiments, evaluated at fixed Gaussian noise intensity ($\tilde{\sigma}=5.0$ for \texttt{SLC} and  $\tilde{\sigma}=10.0$ for \texttt{LiLUCAS} and \texttt{nLUCAS}). DiRoCA, especially with tuned ambiguity radius, achieves consistently lower error as the proportion of outlier environments increases, while non-robust methods degrade. $(\hat{\radbase},\hat{\radabst})=(0.11,0.11)$ for \texttt{SLC} and \texttt{LiLUCAS}, and $(0.47,0.45)$ for \texttt{nLUCAS}}
    \label{fig:alpha_main_curves}
\end{figure*}
\begin{table}[htp]
\caption{Average abstraction error under distribution shift, full corruption ($\alpha=1$) for \texttt{EBM} and \texttt{cMNIST}. $(\hat{\radbase},\hat{\radabst})=(0.55,0.41)$ for \texttt{EBM} and $(61.69,0)$ for \texttt{cMNIST}, while $\epsilon^\star$ corresponds to $(1,1)$ for \texttt{EBM} and $(20,0)$ for \texttt{cMNIST}.}
\label{tab:battery_cmnist_unified_alpha1}
\centering
\begin{tabular}{lcc}
\toprule
Method 
& \texttt{EBM} 
& \texttt{cMNIST} \\
\midrule
$\absp$ 
& 3034.8 $\pm$ 2129.3 
& 2951.3 $\pm$ 51.7 \\
$\absn$ 
& 3330.9 $\pm$ 3238.0 
& 470.7 $\pm$ 8.5 \\
\textsc{Bary}$_{\tauomega}$ 
& 357.4 $\pm$ 144.9 
& 15348.5 $\pm$ 629.4 \\
\textsc{Grad}$_{\tauomega}$ 
& 320.2 $\pm$ 148.2 
& 22200.7 $\pm$ 1492.7 \\
\midrule
\textsc{DiRoCA}$_{\widehat{\epsilon}}$  
& \textbf{297.2 $\pm$ 167.8} 
& 215.0 $\pm$ 8.5 \\
\textsc{DiRoCA}$_{\epsilon^\star}$   
& \textbf{268.4 $\pm$ 141.8} 
& \textbf{48.2 $\pm$ 0.6} \\
\bottomrule
\end{tabular}
\end{table}

\textbf{Beyond Environmental Robustness.} We also test robustness to violations of modeling assumptions, including: (i) Structural ($\mathcal{F}$), where we evaluate our linearly-trained models on new test data generated from SCMs with non-linear structural equations. These misspecified SCMs share the same causal graph and environments as those used during training, but the structural equation for each variable is non-linear, controlled by a strength parameter $k \in \mathbb{R}$; (ii) Intervention mapping ($\omega$), where we contaminate the ground-truth $\omega$ map by reassigning a subset of low-level interventions to different high-level ones of the same complexity, thereby preserving intervention dose while introducing realistic mapping errors; and (iii) Semantic ($\mathcal{S}$), where we utilize the \texttt{cMNIST} dataset to simulate unmodeled camera effects such as geometric rotation and photometric lighting shifts (see Figure~\ref{fig:cmnist_camera_shifts}). This setup introduces a severe structural misspecification: we task the models with learning a \textit{linear} abstraction map from high-dimensional pixels, despite the underlying generative process being inherently non-linear. An explanation of these processes can be found in Appendix~\ref{sec:all_results_settings}. Table~\ref{tab:omega_f_misspec_lilucas_nlucas} shows results for $\omega$- and $\mathcal{F}$-misspecification ($k=1$ with a sinusoidal (\texttt{sin}) function) for the empirical \texttt{LiLUCAS} and \texttt{nLUCAS}. Full robustness curves illustrating the abstraction error as a function of the non-linearity strength $k$ (both \texttt{sin} and \texttt{tanh}) are provided in Appendix~\ref{sec:all_results_settings}. For the $\cS$-misspecification task, Table~\ref{tab:cmnist_rotation_only} reports the Relative Squared Error under rotation (lighting shift results in Appendix~\ref{sec:all_results_settings}). We predict latents as $\hat{Z} = X_{\text{pix}}\lintau^\top + b$, where $b = \mathbb{E}[Z - X_{\text{pix}}\lintau^\top]$ is a test-time per-dimension mean-shift correction. This removes arbitrary intercept mismatch while leaving the learned linear map unchanged. Baselines' performance indicates a catastrophic failure to identify the causal signal whereas \textsc{DiRoCA} maintains a relatively low error. Despite the inherent misspecification of linear abstractions in this high-dimensional non-linear setting, \textsc{DiRoCA}'s remarkable stability suggests it recovers a meaningful linear approximation of the underlying geometry, ignoring high-frequency features that are unstable under rotation. Overall, by acting as an implicit regularizer, \textsc{DiRoCA}'s min–max objective extends robustness beyond environmental shifts, consistently outperforming baselines across all misspecification settings.
\begin{figure}[htp]
    \centering
    \includegraphics[width=.62\linewidth]{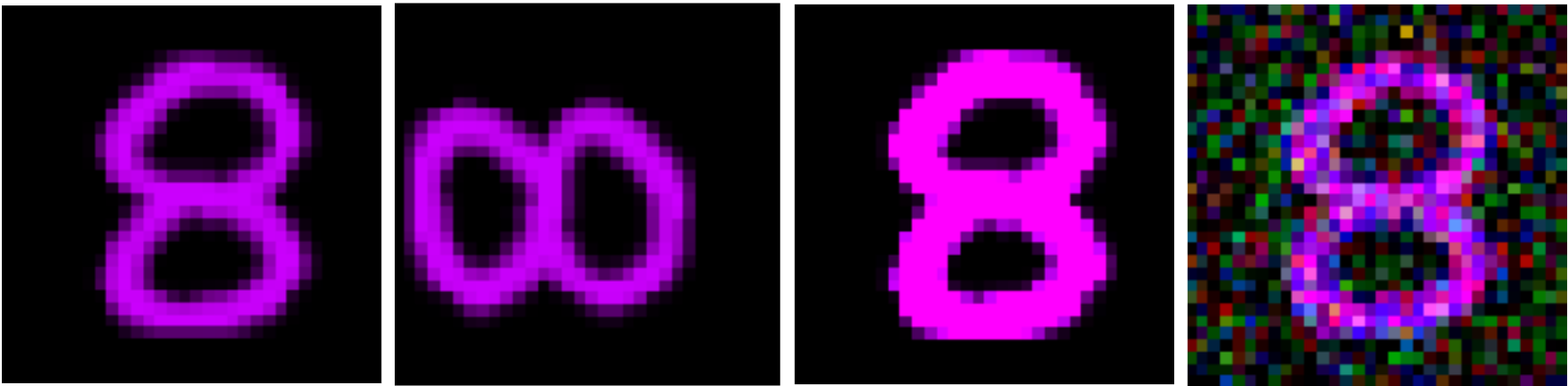}
    \caption{\texttt{cMNIST} camera shifts from left to right: Original, rotation, lighting and environmental ($\tilde{\sigma}=0.5$).}
     \label{fig:cmnist_camera_shifts}
\end{figure}
\section{Conclusion}\label{sec:main_conc}
We introduced $\rhoiota$-abstractions, a framework that bridges the gap between brittle exact abstractions and intractable uniform ones by enforcing consistency across a relevant set of environments. To learn these, we proposed \textsc{DiRoCA}, which casts CAL as a DRO problem. Our theoretical results extend DRO concentration bounds to the CA setting for the true joint environment to guide radius selection and establish provable robustness guarantees, offering a closed-form analytical bound on the worst-case abstraction error. Experiments across different problems and prior art demonstrated a consistently lower abstraction error under both distributional shifts and structural misspecification. Notably, we observe that robustness to environmental shifts often induces resilience to broader sources of misspecification, such as imperfect structural assumptions or intervention mappings. The current framework is limited to linear abstractions and assumes access to both a known intervention map and the true causal DAGs or an accurate estimate via causal discovery. These present challenges that we aim to address in future work.
\begin{table}[htp]
\caption{Average abstraction error under misspecification. 
For \texttt{LiLUCAS}, we report both $\omega$- and $\mathcal{F}$-misspecification under the empirical setting. 
\textsc{DiRoCA}$_{\widehat{\epsilon}}$ uses theoretically prescribed radii: $(\hat{\radbase},\hat{\radabst})=(0.11,0.11)$ for \texttt{LiLUCAS} and $(0.47,0.45)$ for \texttt{nLUCAS}. \textsc{DiRoCA}$_{\epsilon^\star}$ denotes the best-performing configuration, corresponding to $(4,4)$ across all settings.}
\label{tab:omega_f_misspec_lilucas_nlucas}
\centering
\begin{tabular}{lccc}
\toprule
\multirow{2}{*}{Method} 
& \multicolumn{2}{c}{\texttt{LiLUCAS}} 
& \texttt{nLUCAS} \\
\cmidrule(lr){2-3} \cmidrule(lr){4-4}
& $\omega$ & $\mathcal{F}$ & $\omega$ \\
\midrule
\textsc{Bary}$_{\tauomega}$                         
& 548.6 $\pm$ 2.5 
& 578.8 $\pm$ 1.3 
& 21.4 $\pm$ 1.0 \\
\textsc{Grad}$_{\tauomega}$                         
& 305.5 $\pm$ 1.3 
& 338.7 $\pm$ 1.0 
& 18.7 $\pm$ 0.9 \\
$\absp$                                             
& 414.6 $\pm$ 1.1 
& 430.3 $\pm$ 1.2 
& 29.3 $\pm$ 1.2 \\
$\absn$                                             
& 359.5 $\pm$ 1.4 
& 381.3 $\pm$ 1.9 
& 22.7 $\pm$ 0.9 \\
\textsc{DiRoCA}$_{\widehat{\epsilon}}$  
& 459.1 $\pm$ 2.8 
& 500.1 $\pm$ 1.4 
& 18.2 $\pm$ 0.9 \\
\textsc{DiRoCA}$_{\epsilon^\star}$                               
& \textbf{304.1 $\pm$ 7.8} 
& \textbf{327.4 $\pm$ 9.2} 
& \textbf{9.1 $\pm$ 0.3} \\
\bottomrule
\end{tabular}
\end{table}
\begin{table}[htp]
\caption{Robustness to geometric rotation on \texttt{cMNIST} (Relative L2 \%). \textsc{DiRoCA}$_{\widehat{\epsilon}}$ uses theoretically prescribed radii $(\hat{\radbase},\hat{\radabst})=(61.69,0)$, while \textsc{DiRoCA}$_{\epsilon^\star}$ denotes the best-performing configuration, corresponding to $(20,0)$ for both rotation levels.}
\label{tab:cmnist_rotation_only}
\centering
\begin{tabular}{lcc}
\toprule
Method & Low ($30^\circ$) & High ($90^\circ$) \\
\midrule
\textsc{Bary}$_{\tauomega}$             
& 2291.0 $\pm$ 196.6 
& 4135.0 $\pm$ 355.8 \\
\textsc{Grad}$_{\tauomega}$             
& 1780.9 $\pm$ 237.7 
& 3585.9 $\pm$ 643.3 \\
$\absp$                                 
& 541.3 $\pm$ 12.9  
& 604.8 $\pm$ 24.6 \\
$\absn$                                 
& 276.7 $\pm$ 6.5  
& 287.9 $\pm$ 9.3 \\
\midrule
\textsc{DiRoCA}$_{\widehat{\epsilon}}$   
& 101.7 $\pm$ 6.2 
& 106.4 $\pm$ 3.7 \\
\textsc{DiRoCA}$_{\epsilon^\star}$               
& \textbf{28.0 $\pm$ 0.4} 
& \textbf{30.7 $\pm$ 0.7} \\
\bottomrule
\end{tabular}
\end{table}

\section*{Acknowledgements}
\textbf{YF} acknowledges support by the Onassis Foundation -
Scholarship ID: F ZR 063-1/2021-2022. \textbf{TD} acknowledges support from a UKRI Turing AI acceleration Fellowship [EP/V02678X/1].

\bibliography{example_paper}

@book{villani2009optimal,
  title={Optimal transport: old and new},
  author={Villani, C{\'e}dric and others},
  volume={338},
  year={2009},
  publisher={Springer}
}

@misc{kantorovich1942,
  author       = {Kantorovich, L.},
  title        = { On the transfer of masses (in russian)},
  howpublished = {Doklady Akademii Nauk},
  month        = {August},
  year         = {1942},
}

@article{monge1781memoire,
  title={M{\'e}moire sur la th{\'e}orie des d{\'e}blais et des remblais},
  author={Monge, Gaspard},
  journal={Mem. Math. Phys. Acad. Royale Sci.},
  pages={666--704},
  year={1781}
}

@Article{peyre2019computational,
  author    = {Peyr{\'e}, Gabriel and Cuturi, Marco and others},
  journal   = {Foundations and Trends{\textregistered} in Machine Learning},
  title     = {Computational optimal transport: With applications to data science},
  year      = {2019},
  number    = {5-6},
  pages     = {355--607},
  volume    = {11},
  publisher = {Now Publishers, Inc.},
}

@InProceedings{rubenstein2017causal,
  author       = {Rubenstein, Paul K and Weichwald, Sebastian and Bongers, Stephan and Mooij, Joris M and Janzing, Dominik and Grosse-Wentrup, Moritz and Sch{\"o}lkopf, Bernhard},
  booktitle    = {33rd Conference on Uncertainty in Artificial Intelligence (UAI 2017)},
  title        = {Causal consistency of structural equation models},
  year         = {2017},
  organization = {Curran Associates, Inc.},
  pages        = {808--817},
  printed      = {printed},
  ranking      = {rank5},
  readstatus   = {read},
}

@Book{pearl2009causality,
  author    = {Pearl, Judea},
  publisher = {Cambridge University Press},
  title     = {Causality},
  year      = {2009},
  ranking   = {rank5},
  relevance = {relevant},
}

@InProceedings{zennaro2023jointly,
  author    = {Fabio Massimo Zennaro and M{\'a}t{\'e} Dr{\'a}vucz and Geanina Apachitei and W. Dhammika Widanage and Theodoros Damoulas},
  booktitle = {2nd Conference on Causal Learning and Reasoning},
  title     = {Jointly Learning Consistent Causal Abstractions Over Multiple Interventional Distributions},
  year      = {2023},
  url       = {https://openreview.net/forum?id=RNs7aMS6zDq},
}

@InProceedings{beckers2018abstracting,
  author     = {Beckers, Sander and Halpern, Joseph Y},
  booktitle  = {Proceedings of the AAAI Conference on Artificial Intelligence},
  title      = {Abstracting Causal Models},
  year       = {2019},
  pages      = {2678--2685},
  volume     = {33},
}

@mastersthesis{rischel2020category,
  author     = {Rischel, Eigil Fjeldgren},
  title      = {The Category Theory of Causal Models},
  year       = {2020},
  school     = {University of Copenhagen},
}

@InProceedings{beckers2020approximate,
  author       = {Beckers, Sander and Eberhardt, Frederick and Halpern, Joseph Y},
  booktitle    = {Uncertainty in Artificial Intelligence},
  title        = {Approximate causal abstractions},
  year         = {2020},
  organization = {PMLR},
  pages        = {606--615},
}

@Article{geiger2021causal,
  author  = {Geiger, Atticus and Lu, Hanson and Icard, Thomas and Potts, Christopher},
  journal = {Advances in Neural Information Processing Systems},
  title   = {Causal abstractions of neural networks},
  year    = {2021},
  pages   = {9574--9586},
  volume  = {34},
}

@article{felekis2024causal,
  title={Causal optimal transport of abstractions},
  author={Felekis, Yorgos and Zennaro, Fabio Massimo and Branchini, Nicola and Damoulas, Theodoros},
  journal={Proceedings of CLeaR (Causal Learning and Reasoning) 2024},
  year={2024},
  publisher={Journal of Machine Learning Research Workshop and Conference Proceedings series.}
}

@inproceedings{rischel2021compositional,
  author = {Rischel, Eigil Fjeldgren and Weichwald, Sebastian},
  booktitle = {UAI},
  editor = {de Campos, Cassio P. and Maathuis, Marloes H. and Quaeghebeur, Erik},
  ee = {https://proceedings.mlr.press/v161/rischel21a.html},
  pages = {1013-1023},
  publisher = {AUAI Press},
  series = {Proceedings of Machine Learning Research},
  title = {Compositional abstraction error and a category of causal models.},
  url = {http://dblp.uni-trier.de/db/conf/uai/uai2021.html#RischelW21},
  volume = 161,
  year = 2021
}

@article{scholkopf2021toward,
  title={Toward causal representation learning},
  author={Sch{\"o}lkopf, Bernhard and Locatello, Francesco and Bauer, Stefan and Ke, Nan Rosemary and Kalchbrenner, Nal and Goyal, Anirudh and Bengio, Yoshua},
  journal={Proceedings of the IEEE},
  volume={109},
  number={5},
  pages={612--634},
  year={2021},
  publisher={IEEE}
}

@Article{chalupka2017causal,
  author    = {Chalupka, Krzysztof and Eberhardt, Frederick and Perona, Pietro},
  journal   = {Behaviormetrika},
  title     = {Causal feature learning: an overview},
  year      = {2017},
  number    = {1},
  pages     = {137--164},
  volume    = {44},
  publisher = {Springer},
}

@inproceedings{dyer2023interventionally,
 author = {Dyer, Joel and Bishop, Nicholas and Felekis, Yorgos and Zennaro, Fabio Massimo and Calinescu, Anisoara and Damoulas, Theodoros and Wooldridge, Michael},
 booktitle = {Advances in Neural Information Processing Systems},
 editor = {A. Globerson and L. Mackey and D. Belgrave and A. Fan and U. Paquet and J. Tomczak and C. Zhang},
 pages = {21814--21841},
 publisher = {Curran Associates, Inc.},
 title = {Interventionally Consistent Surrogates for Complex Simulation Models},
 volume = {37},
 year = {2024}
}

@article{xia2024neural, title={Neural Causal Abstractions}, volume={38}, url={https://ojs.aaai.org/index.php/AAAI/article/view/30044}, DOI={10.1609/aaai.v38i18.30044}, abstractNote={The ability of humans to understand the world in terms of cause and effect relationships, as well as their ability to compress information into abstract concepts, are two hallmark features of human intelligence. These two topics have been studied in tandem under the theory of causal abstractions, but it is an open problem how to best leverage abstraction theory in real-world causal inference tasks, where the true model is not known, and limited data is available in most practical settings. In this paper, we focus on a family of causal abstractions constructed by clustering variables and their domains, redefining abstractions to be amenable to individual causal distributions. We show that such abstractions can be learned in practice using Neural Causal Models, allowing us to utilize the deep learning toolkit to solve causal tasks (identification, estimation, sampling) at different levels of abstraction granularity. Finally, we show how representation learning can be used to learn abstractions, which we apply in our experiments to scale causal inferences to high dimensional settings such as with image data.}, number={18}, journal={Proceedings of the AAAI Conference on Artificial Intelligence}, author={Xia, Kevin and Bareinboim, Elias}, year={2024}, month={Mar.}, pages={20585-20595} }

@inproceedings{
kekić2023targeted,
title={Targeted Reduction of Causal Models},
author={Armin Keki{\'c} and Bernhard Sch{\"o}lkopf and Michel Besserve},
booktitle={The 40th Conference on Uncertainty in Artificial Intelligence},
year={2024},
url={https://openreview.net/forum?id=CFHpI53xmb}
}

@inproceedings{massidda2024learningcausalabstractionslinear,
author = {Massidda, Riccardo and Magliacane, Sara and Bacciu, Davide},
title = {Learning causal abstractions of linear structural causal models},
year = {2024},
publisher = {JMLR.org},
booktitle = {Proceedings of the Fortieth Conference on Uncertainty in Artificial Intelligence},
articleno = {118},
numpages = {30},
location = {Barcelona, Spain},
series = {UAI '24}
}

@inbook{kuhn2019wassersteindistributionallyrobustoptimization,
author = {Kuhn, Daniel and Esfahani, Peyman and Nguyen, Viet and Shafieezadeh-Abadeh, Soroosh},
year = {2019},
month = {10},
pages = {130-166},
title = {Wasserstein Distributionally Robust Optimization: Theory and Applications in Machine Learning},
isbn = {978-0-9906153-3-0},
doi = {10.1287/educ.2019.0198}
}

@article{nguyen2023meancovariancerobustriskmeasurement,
  title={Mean-Covariance Robust Risk Measurement},
  author={Viet Anh Nguyen and Soroosh Shafiee and Damir Filipovi'c and Daniel Kuhn},
  journal={SSRN Electronic Journal},
  year={2021},
  url={https://api.semanticscholar.org/CorpusID:245335053}
}

@inproceedings{
dacunto2025causalabstractionlearningbased,
title={Causal Abstraction Learning based on the Semantic Embedding Principle},
author={Gabriele D'Acunto and Fabio Massimo Zennaro and Yorgos Felekis and Paolo Di Lorenzo},
booktitle={Forty-second International Conference on Machine Learning},
year={2025},
url={https://openreview.net/forum?id=J16AIOkjjY}
}

@inproceedings{
zennarobandits,
title={Causally Abstracted Multi-armed Bandits},
author={Fabio Massimo Zennaro and Nicholas George Bishop and Joel Dyer and Yorgos Felekis and Ani Calinescu and Michael J. Wooldridge and Theodoros Damoulas},
booktitle={The 40th Conference on Uncertainty in Artificial Intelligence},
year={2024},
url={https://openreview.net/forum?id=Uxrxz4X416}
}

@InProceedings{cohen2019certifiedadversarialrobustnessrandomized,
  title = 	 {Certified Adversarial Robustness via Randomized Smoothing},
  author =       {Cohen, Jeremy and Rosenfeld, Elan and Kolter, Zico},
  booktitle = 	 {Proceedings of the 36th International Conference on Machine Learning},
  pages = 	 {1310--1320},
  year = 	 {2019},
  editor = 	 {Chaudhuri, Kamalika and Salakhutdinov, Ruslan},
  volume = 	 {97},
  series = 	 {Proceedings of Machine Learning Research},
  month = 	 {09--15 Jun},
  publisher =    {PMLR},
  url = 	 {https://proceedings.mlr.press/v97/cohen19c.html}
}

@inproceedings{robustGPs,
author = {Altamirano, Matias and Briol, Fran\c{c}ois-Xavier and Knoblauch, Jeremias},
title = {Robust and conjugate Gaussian process regression},
year = {2024},
publisher = {JMLR.org},
booktitle = {Proceedings of the 41st International Conference on Machine Learning},
articleno = {49},
numpages = {31},
location = {Vienna, Austria},
series = {ICML'24}
}

@article{alvarezfixed,
title = {A fixed-point approach to barycenters in Wasserstein space},
journal = {Journal of Mathematical Analysis and Applications},
volume = {441},
number = {2},
pages = {744-762},
year = {2016},
issn = {0022-247X},
doi = {https://doi.org/10.1016/j.jmaa.2016.04.045},
url = {https://www.sciencedirect.com/science/article/pii/S0022247X16300907},
author = {Pedro C. Álvarez-Esteban and E. {del Barrio} and J.A. Cuesta-Albertos and C. Matrán}
}

@article{aguehbary,
author = {Agueh, Martial and Carlier, Guillaume},
title = {Barycenters in the Wasserstein Space},
journal = {SIAM Journal on Mathematical Analysis},
volume = {43},
number = {2},
pages = {904-924},
year = {2011},
doi = {10.1137/100805741},

URL = { 
    
        https://doi.org/10.1137/100805741
    
    

},
eprint = { 
    
        https://doi.org/10.1137/100805741
    
    

}
}

@article{Givens1984ACO,
  title={A class of Wasserstein metrics for probability distributions.},
  author={Clark R. Givens and R. M. Shortt},
  journal={Michigan Mathematical Journal},
  year={1984},
  volume={31},
  pages={231-240},
  url={https://api.semanticscholar.org/CorpusID:121338763}
}

@article{Knott1984OnTO,
  title={On the optimal mapping of distributions},
  author={Martin Knott and C. S. Smith},
  journal={Journal of Optimization Theory and Applications},
  year={1984},
  volume={43},
  pages={39-49},
  url={https://api.semanticscholar.org/CorpusID:120208956}
}

@article{Brenier1991PolarFA,
  title={Polar Factorization and Monotone Rearrangement of Vector-Valued Functions},
  author={Yann Brenier},
  journal={Communications on Pure and Applied Mathematics},
  year={1991},
  volume={44},
  pages={375-417},
  url={https://api.semanticscholar.org/CorpusID:123428953}
}

@inproceedings{NIPS2004_e0688d13,
 author = {Srebro, Nathan and Rennie, Jason and Jaakkola, Tommi},
 booktitle = {Advances in Neural Information Processing Systems},
 editor = {L. Saul and Y. Weiss and L. Bottou},
 pages = {},
 publisher = {MIT Press},
 title = {Maximum-Margin Matrix Factorization},
 url = {https://proceedings.neurips.cc/paper_files/paper/2004/file/e0688d13958a19e087e123148555e4b4-Paper.pdf},
 volume = {17},
 year = {2004}
}

@article{gelbrich1990formula,
  title={On a formula for the L2 Wasserstein metric between measures on Euclidean and Hilbert spaces},
  author={Gelbrich, Matthias},
  journal={Mathematische Nachrichten},
  volume={147},
  number={1},
  pages={185--203},
  year={1990},
  publisher={Wiley Online Library}
}

@Article{infoloss,
AUTHOR = {Fullwood, James and Parzygnat, Arthur J.},
TITLE = {The Information Loss of a Stochastic Map},
JOURNAL = {Entropy},
VOLUME = {23},
YEAR = {2021},
NUMBER = {8},
ARTICLE-NUMBER = {1021},
URL = {https://www.mdpi.com/1099-4300/23/8/1021},
PubMedID = {34441161},
ISSN = {1099-4300},
DOI = {10.3390/e23081021}
}

@inproceedings{
schooltink2025aligninggraphicalfunctionalcausal,
title={Aligning Graphical and Functional Causal Abstractions},
author={Willem Schooltink and Fabio Massimo Zennaro},
booktitle={Fourth Conference on Causal Learning and Reasoning},
year={2025},
url={https://openreview.net/forum?id=itESdHtCyO}
}

@inproceedings{
zennaro2022computingoptimalabstractionstructural,
title={Towards Computing an Optimal Abstraction for Structural Causal Models},
author={Fabio Massimo Zennaro and Paolo Turrini and Theo Damoulas},
booktitle={UAI 2022 Workshop on Causal Representation Learning},
year={2022},
url={https://openreview.net/forum?id=zGLniPvGsWT}
}

@misc{schneider2025generativeinterventionmodelscausal,
      title={Generative Intervention Models for Causal Perturbation Modeling}, 
      author={Nora Schneider and Lars Lorch and Niki Kilbertus and Bernhard Schölkopf and Andreas Krause},
      year={2025},
      eprint={2411.14003},
      archivePrefix={arXiv},
      primaryClass={cs.LG},
      url={https://arxiv.org/abs/2411.14003}, 
}
\bibliographystyle{icml2026}

\newpage
\appendix
\onecolumn

\section{Appendix}
\paragraph{Overview.} This appendix provides the theoretical foundations, formal proofs, and extended experimental details supporting the main text. In \S\ref{sec:caapp}, we position our framework within the broader causal abstraction literature. We review the necessary mathematical background on the Wasserstein metric and Distributionally Robust Optimization in \S\ref{sec:wassapp} and \S\ref{sec:droapp}, respectively. \S\ref{sec:thmsproofsapp} contains the formal proofs, including the concentration results for joint environments and the derivation of the closed-form robustness certificate. We then detail the practical implementation for general ANMs, describing the $(D, U)$ decomposition in \S\ref{sec:anmapp} and the construction of the joint ambiguity set in \S\ref{sec:ambsetapp}. Comprehensive descriptions of the datasets, causal graphs, and intervention mappings are provided in \S\ref{sec:examples_app}. In \S\ref{sec:all_results_settings}, we present the full experimental roadmap and additional robustness analyses against model misspecification. Finally, \S\ref{sec:optapp} details the optimization procedures, providing the analytical derivations for both the linear (Gaussian and empirical) and the general \textsc{DiRoCA} formulations.

\subsection{Relation to existing CA frameworks.}\label{sec:caapp}
Fig.~\ref{fig:casoverview} provides an overview of how existing causal abstraction (CA) frameworks and learning methods differ in their treatment of environmental variability. The left panel organizes CA frameworks by the subset of the joint environment space \(\cP(\exprod)\) over which they require interventional consistency. 

The $\rhoiota$ framework differs from the exact transformation formulation of \citet{rubenstein2017causal}, and the $(R, a, \alpha)$ of \citet{rischel2020category}, both of which implicitly assume a fixed environment $(\cP_{0}(\exprod))$ without addressing how it constrains or influences the abstraction. As a result, these frameworks may yield ill-defined forms of abstractions when applied beyond that fixed setting, as also noted by \citet{beckers2018abstracting}, limiting their practical utility. Our notion of $\rhoiota$-abstractions also differs from the \emph{uniform transformations} introduced by \citet{beckers2018abstracting}, which requires interventional consistency to hold across \emph{all} possible pairs of environments $(\cP_{\infty}(\exprod))$. By design, uniform transformations are environment-independent: they relate the deterministic components of the SCMs; i.e., their deterministic causal bases without regard to the specific distributions that generate observational data. While uniform abstractions offer a theoretically stronger and philosophically valid notion of abstraction, they are practically infeasible to learn, as they assume access to an unbounded number of environments, many of which may be inaccessible due to physical, data, or knowledge constraints. The $\rhoiota$ framework strikes a middle ground: it preserves the formal rigor of exact abstractions while remaining flexible and computationally tractable under real-world distributional limitations by requiring consistency over a constrained relevant subset $\cP_{m}(\exprod)$, with $0 \leq m < \infty$. Just as the notion of relevant interventions restricts attention to those that can be meaningfully abstracted and implemented at the high level, we similarly constrain focus to environments that are plausible or meaningful within a given setting. Together, these relevant environments and interventions define the \emph{abstraction context}, which determines the domain over which causal consistency is required.

As discussed, when exact abstractions are not achievable, consistency between models is relaxed to be approximate. However, in the CAL literature, various formulations of abstraction error (Eq.~\ref{eq:totabst_error}) have been proposed. The divergence term \(\mathcal{D}_{\enabst}\) has been instantiated using either the Jensen–Shannon divergence \citep{zennaro2023jointly, zennarobandits} or the KL divergence \citep{dyer2023interventionally, kekić2023targeted, dacunto2025causalabstractionlearningbased}. Regarding the aggregation function $\mathfrak{h}$, initial frameworks adopted the maximum over interventions \citep{beckers2020approximate, rischel2021compositional}, a choice followed by many subsequent CAL works \citep{zennaro2023jointly, zennarobandits}, while later approaches also considered the expectation over the intervention set \citep{felekis2024causal, dyer2023interventionally, kekić2023targeted}. Note that, especially in policy-making scenarios, some interventions matter more than others due to factors like the likelihood of implementation or cost, reflecting an implicit weighting over the intervention set. While prior works such as \citep{felekis2024causal, dyer2023interventionally} assume this distribution to be uniform, assigning equal weight to all interventions, it can be adapted to reflect practical priorities or domain-specific preferences.

As for the aggregation function $\mathfrak{g}$, no prior CAL framework has explicitly addressed variability across environments of the two SCMs. The closest suggestion came from \citet{beckers2020approximate}, who, within the setting of uniform abstractions, proposed taking an expectation over exogenous samples of the low-level model. Crucially, they showed the existence of a deterministic map $\tauu$ that maps low-level exogenous variables into their high-level counterparts. This result relies on a strong assumption, induced by the uniformity property: for any low-level exogenous sample, one can define an environment assigning it probability one, and deterministically map it to the high level (see Appendix, Theorem 3.6 in \citep{beckers2020approximate}). Their definition further aggregates over these  $\tauu$ maps via a minimum operator. In contrast, our approach does not assume the existence of a deterministic map across environments. Instead, we treat each environment independently, adopting a data-driven formulation avoiding potentially restrictive assumptions or optimality conditions imposed by a global minimum. That said, since we work with probability measures rather than individual samples, we can always estimate the pushforward of a function, even when its deterministic form is not recoverable. At best, we can approximate a stochastic map derived from the pushforward. In other words, we can always learn a measure-preserving map \(\tauupush: \exprobbase \to \exprobabst\), though not necessarily the underlying function $\tauu$. Of course, computing the pushforward measure comes at a cost: it may be complex, computationally expensive, or difficult to work with. This cost could potentially be interpreted as a form of information loss, an idea further explored in \citep{infoloss}.

Regarding the task of CAL, the right panel of Fig.~\ref{fig:casoverview} aligns learning methods with their underlying environmental assumptions. Prior CAL approaches operate in the $\epsilon = 0$ regime, implicitly assuming a fixed environment, performing learning under varying settings and assumptions. In contrast, our method introduces a robustness parameter $0 \leq \epsilon <\infty$ and explicitly models environmental uncertainty using a Wasserstein ball $\bB_{\epsilon}$ centered at the empirical joint environment, enabling generalization to distributional shifts and misspecification within a range controlled by $\epsilon$.

\subsection{The Wasserstein Distance}\label{sec:wassapp}
The Wasserstein distance strongly relates to the theory of optimal transport \citep{villani2009optimal}, \citep{peyre2019computational}, which seeks to find the most efficient way of transforming one probability distribution to another. In the classical \emph{Monge formulation} \citep{monge1781memoire}, the goal is to find a deterministic map $T : \mathbb{R}^n \to \mathbb{R}^n$ that pushes a source measure $\mu$ onto a target measure $\nu$, while minimizing the expected cost of transport, typically taken as the squared Euclidean distance $\|x - T(x)\|_2^2$. However, the Monge problem is highly non-convex and may not admit solutions in general. To address these difficulties, \citet{kantorovich1942} proposed a relaxation, allowing for \emph{couplings}, these are joint distributions $\pi$ on $\mathbb{R}^n \times \mathbb{R}^n$ with marginals $\mu$ and $\nu$ instead of deterministic maps. 

The $W_2$ Wasserstein distance between two probability measures $\mu$ and $\nu$ on $\mathbb{R}^n$ is then defined as the minimal expected squared cost over all such couplings:
\begin{align}
W_2(\mu, \nu) := \inf_{\pi \in \Pi(\mu, \nu)} \left( \mathbb{E}_{(X,Y) \sim \pi} \|X - Y\|_2^2 \right)^{1/2},
\end{align}
where $\Pi(\mu, \nu)$ denotes the set of all couplings with marginals $\mu$ and $\nu$. Intuitively, $W_2(\mu, \nu)$ measures the least amount of effort required to convert $\mu$ into $\nu$ under a given transport cost.

\paragraph{Wasserstein Distance between Gaussians.}
When $\mu$ and $\nu$ are multivariate Gaussians $\mathcal{N}(m_1, \Sigma_1)$ and $\mathcal{N}(m_2, \Sigma_2)$, the 2-Wasserstein distance admits a closed-form expression \citep{Givens1984ACO}:
\begin{align}\label{eq:gausswassd}
W_2\left( \mathcal{N}(m_1, \Sigma_1), \mathcal{N}(m_2, \Sigma_2) \right) = \sqrt{\|m_1 - m_2\|_2^2 + \operatorname{Tr}\left(\Sigma_1 + \Sigma_2 - 2\left(\Sigma_1^{1/2} \Sigma_2 \Sigma_1^{1/2}\right)^{1/2}\right)}
\end{align}
The squared $W_2$ is also known as the \emph{Gelbrich distance} \citep{gelbrich1990formula}, and it captures both the displacement between the means and the discrepancy between the covariances. In particular, when the covariance matrices commute, the trace term simplifies to the squared Frobenius norm between their matrix square roots.

\paragraph{Optimal Transport Map between Gaussians.}
When $\mu = \mathcal{N}(m_1, \Sigma_1)$ and $\nu = \mathcal{N}(m_2, \Sigma_2)$ are two Gaussian measures on $\mathbb{R}^n$, the optimal transport map $T : \mathbb{R}^n \to \mathbb{R}^n$ pushing $\mu$ to $\nu$ under the $W_2$ distance is affine and given by \citep{Knott1984OnTO}:
\begin{align}\label{eq:mongegauss}
T(x) = m_2 + A(x - m_1),
\end{align}
where the matrix $A$ is symmetric positive definite and is equal to:
\begin{align}
A = \Sigma_2^{1/2} \left( \Sigma_2^{1/2} \Sigma_1 \Sigma_2^{1/2} \right)^{-1/2} \Sigma_2^{1/2}.
\end{align}
This map is the gradient of a convex function, and is therefore the unique optimal transport map up to sets of $\mu$-measure zero\footnote{Formally, uniqueness holds $\mu$-almost everywhere: if two maps both solve the Monge problem, then they agree except possibly on a set of $\mu$-measure zero.}, by Brenier’s theorem \citep{Brenier1991PolarFA}. Moreover, it satisfies the relation
\begin{align}
T_{\sharp} \mu = \nu,
\end{align}
meaning that the pushforward of $\mu$ under $T$ is exactly $\nu$.

\paragraph{Wasserstein Barycenters.}
The Wasserstein barycenter \citep{aguehbary} provides a principled way to aggregate a collection of probability distributions into a single representative distribution by minimizing the average squared Wasserstein distance to the given distributions. Given a set of probability measures $\{\nu_j\}_{j=1}^n$ and associated weights $\{\lambda_j\}_{j=1}^n$ with $\lambda_j > 0$ and $\sum_{j=1}^n \lambda_j = 1$, the Wasserstein barycenter $\nu^\star$ is defined as the solution of the following problem:
\begin{align} 
\nu^\star = \argmin_{\nu} \sum_{j=1}^n \lambda_j W_2^2(\nu, \nu_j),
\end{align}
where $W_2(\cdot, \cdot)$ denotes the 2-Wasserstein distance between probability measures. In the special case where all $\nu_j$ are multivariate Gaussians $\mathcal{N}(\mu_j, \Sigma_j)$, it can be shown that the barycenter $\nu^\star$ is also a Gaussian $\mathcal{N}(\mu^\star, \Sigma^\star)$, where the mean $\mu^\star$ is the weighted average:
\begin{align}\label{eq:barymean} 
\mu^\star = \sum_{j=1}^n \lambda_j \mu_j
\end{align}
and the covariance $\Sigma^\star$ solves the fixed-point equation ~\citep{alvarezfixed}:
\begin{align}\label{eq:barysigma}
\Sigma^\star = \sum_{j=1}^n \lambda_j \left( \Sigma^\star{}^{1/2} \Sigma_j \Sigma^\star{}^{1/2} \right)^{1/2}.
\end{align}
Here, the square roots are matrix square roots, and the equation is understood in the positive semidefinite sense. 

\subsection{Distributionally Robust Optimization (DRO)}\label{sec:droapp}
\textit{Distributionally Robust Optimization (DRO)} is a mathematical framework for decision-making under uncertainty, designed to hedge against model misspecification and distributional shifts~\citep{kuhn2019wassersteindistributionallyrobustoptimization}. Instead of optimizing performance with respect to a single estimated distribution, DRO considers a set of plausible distributions, called the \emph{ambiguity set}, and seeks solutions that minimize the worst-case expected loss across this set.

\textbf{Formulation:} Let \(x \in \mathbb{R}^n\) denote a decision variable, \(\xi \in \Xi = \mathbb{R}^m\) a random vector representing uncertain data, and \(f(x, \xi): \mathbb{R}^n \times \Xi \rightarrow \mathbb{R}\) a loss function. The standard \textit{Wasserstein DRO} objective is to solve:
\begin{align}
    \inf_{x \in \mathcal{X}} \sup_{\mathbb{Q} \in \mathbb{B}_{\varepsilon, p}(\widehat{\mathbb{P}}_N)} \mathbb{E}_{\xi \sim \mathbb{Q}}[f(x, \xi)],
\end{align}
where \(\widehat{\mathbb{P}}_N\) is a nominal distribution estimated from data, and \(\mathbb{B}_{\varepsilon, p}(\widehat{\mathbb{P}}_N)\) denotes a Wasserstein ball of radius \(\varepsilon\) centered at \(\widehat{\mathbb{P}}_N\):
\begin{align}
    \mathbb{B}_{\varepsilon, p}(\widehat{\mathbb{P}}_N) = \left\{ \mathbb{Q} \in \mathcal{P}(\Xi) : W_p(\mathbb{Q}, \widehat{\mathbb{P}}_N) \leq \varepsilon \right\}.
\end{align}
The radius \(\varepsilon\) controls the size of the ambiguity set and reflects the desired level of robustness: a larger \(\varepsilon\) allows for greater deviations from the nominal distribution (increasing robustness), while a smaller \(\varepsilon\) yields tighter but potentially less robust decisions. Below, we present two formulations of DRO that are relevant to our work: \textbf{(a)} \emph{Elliptical DRO}, which leverages moment-based structural assumptions; and \textbf{(b)} \emph{Empirical DRO}, which builds ambiguity sets directly from the empirical data distribution.

\paragraph{Elliptical DRO.} Elliptical DRO leverages structural assumptions by modeling the data distribution as belonging to the elliptical family (e.g., Gaussian, Student-t), which is characterized by its mean and covariance. Elliptical distributions are fully characterized by their first two moments, which makes them particularly well-suited for moment-based uncertainty modeling. Specifically, a distribution \( \mathbb{Q} = \mathcal{E}_g(\mu, \Sigma) \) is called \emph{elliptical} if it admits a density of the form \( f(\xi) = C \det(\Sigma)^{-1} g\left((\xi - \mu)^\top \Sigma^{-1} (\xi - \mu)\right) \), where \(g\) is a nonnegative generator function and \(C\) is a normalizing constant. 

As we saw in the previous section, when \(p=2\), the squared Wasserstein distance between two elliptical distributions admits a closed-form lower bound known as the \emph{Gelbrich distance}:
\begin{align}\label{eq:gelbrich}
d_G^2((\mu, \Sigma), (\mu^\prime, \Sigma^\prime)) = \|\mu - \mu^\prime\|^2 + \operatorname{Tr}(\Sigma) + \operatorname{Tr}(\Sigma^\prime) - 2\operatorname{Tr}\left((\Sigma^{1/2} \Sigma^\prime \Sigma^{1/2})^{1/2}\right),
\end{align}
where \((\mu, \Sigma)\) and \((\mu^\prime, \Sigma^\prime)\) are the mean-covariance pairs of two distributions.

Assuming elliptical structure, the Wasserstein ambiguity set can be projected onto the space of first and second moments, yielding an \emph{elliptical uncertainty set}~\citep{nguyen2023meancovariancerobustriskmeasurement}:
\begin{align}
    \mathcal{U}_{\varepsilon}(\widehat{\mu}, \widehat{\Sigma}) = \left\{ (\mu, \Sigma) \in \mathbb{R}^m \times \mathbb{S}_+^m : d_G^2\left((\widehat{\mu}, \widehat{\Sigma}), (\mu, \Sigma)\right) \leq \varepsilon^2 \right\}.
\end{align}
This construction offers two key advantages: \textbf{(i)} interpretability, as uncertainty is captured through shifts in mean and covariance; and \textbf{(ii)} computational tractability, since \(\mathcal{U}_{\varepsilon}\) is a convex and compact set. Furthermore, when the nominal distribution \(\widehat{\mathbb{P}}_N\) itself is elliptical, the Wasserstein ambiguity set and the elliptical uncertainty set coincide exactly, i.e., \(\mathbb{B}_{\varepsilon, p}(\widehat{\mathbb{P}}_N) = \mathcal{U}_{\varepsilon}(\widehat{\mu}, \widehat{\Sigma})\).

\paragraph{Empirical DRO.}
In contrast, empirical DRO makes no structural assumptions about the underlying data distribution. Instead, the nominal distribution is the empirical measure:
\begin{align}
    \widehat{\mathbb{P}}_N = \frac{1}{N} \sum_{i=1}^N \delta_{\widehat{\xi}_i},
\end{align}
where \(\delta_{\widehat{\xi}_i}\) denotes the Dirac measure at the \(i\)-th training sample \(\widehat{\xi}_i\). A convenient reparameterization restricts attention to ambiguity sets of perturbed empirical distributions of the form:
\begin{align}
    \mathbb{Q}(\Theta) = \frac{1}{N} \sum_{i=1}^N \delta_{\widehat{\xi}_i + \theta_i},
\end{align}
where \(\theta_i \in \mathbb{R}^m\) is a displacement vector applied to the \(i\)-th sample, and \(\Theta = (\theta_1, \ldots, \theta_N) \in \mathbb{R}^{m \times N}\) is the perturbation matrix. The Wasserstein constraint \(W_p(\mathbb{Q}(\Theta), \widehat{\mathbb{P}}_N) \leq \varepsilon\) is then equivalent to:
\begin{align}
    \frac{1}{N} \sum_{i=1}^N \|\theta_i\|^p \leq \varepsilon^p.
\end{align}
For $p = 2$, this simplifies to $ \|\Theta\|_{\operatorname{F}} \leq \varepsilon \sqrt{N} $, where $ \|\cdot\|_F $ denotes the Frobenius norm. This formulation offers a clear geometric interpretation: empirical DRO controls the total amount of perturbation allowed on the data samples, measured in aggregate via the Frobenius norm. As a result, optimization under empirical DRO can be interpreted as finding decisions that are robust against small, localized deviations from the observed data points.

\subsection{Theorems and Proofs}\label{sec:thmsproofsapp}
\begin{proposition}[Consistency of Metric-based Abstraction Errors]\label{prop:consistency_metric_abstraction}
Let an abstraction context $(\stableprod, \intervsetprod)$ and a distribution $q$ over $\intervsetbase$ be given. Suppose $D_{\enabst}$ is a statistical divergence or distance between probability measures satisfying
$D_{\enabst}(p, q)=0 \iff p=q$. If $\scmabst$ is a $\rhoiota$-0-approximate abstraction of $\scmbase$, then for all $\exprob \in \stableprod$ and $q$-almost surely, $\scmabst$ and $\scmbase$ are $\rhoiota$-interventionally consistent; i.e. Eq.~\eqref{eq:tau_rho_def} holds.
\end{proposition}
\begin{proof}
We first consider the case where both aggregation functions $\mathfrak{g}$ and $\mathfrak{h}$ are realized as expectations. Specifically, this transforms the environment-intervention error as follows:
\begin{align}
e_{\tau}^{\exprob, \iota}(\scmbase, \scmabst) = \expval_{\exprob \in \stableprod} \expval_{\iota \sim q} \left[ \mathcal{D}_{\enabst} \left( \tau_{\#}(\mixfbase_{\iota \#}(\exprobbase)),~ \mixfabst_{\omega(\iota)\#}(\exprobabst) \right) \right].
\end{align}
By non-negativity of the divergence $\mathcal{D}_{\enabst}$, we have that for every fixed $\exprob \in \stableprod$, the inner expectation is non-negative. Since the total abstraction error is assumed to be zero, it follows that for each $\exprob \in \stableprod$,
\begin{align}
\expval_{\iota \sim q} \left[ \mathcal{D}_{\enabst} \left( \tau_{\#}(\mixfbase_{\iota \#}(\exprobbase)),~ \mixfabst_{\omega(\iota)\#}(\exprobabst) \right) \right] = 0.
\end{align}
Taking the expectation of a non-negative random variable implies that
\begin{align}
\mathcal{D}_{\enabst} \left( \tau_{\#}(\mixfbase_{\iota \#}(\exprobbase)),~ \mixfabst_{\omega(\iota)\#}(\exprobabst) \right) = 0,
\quad ~q\text{-almost surely}
\end{align}
Since $D_{\enabst}(p,q) = 0  \iff p=q$, we conclude that:
\begin{align}
\tau_{\#}(\mixfbase_{\iota \#}(\exprobbase)) = \mixfabst_{\omega(\iota)\#}(\exprobabst)
\quad~q\text{-almost surely for each fixed } \exprob \in \stableprod.
\end{align}

If either one or both of the aggregation functions $\mathfrak{g}$ and $\mathfrak{h}$ are realized as an essential supremum, the argument becomes even simpler: since the supremum of a non-negative function is zero if and only if the function itself is zero everywhere, we again conclude that
\begin{align}
\tau_{\#}\left( \mixfbase_{\iota \#}(\exprobbase) \right) = \mixfabst_{\omega(\iota)\#}(\exprobabst)
\end{align}
for all $\exprob \in \stableprod$ $q$-almost surely.
\end{proof}
\begin{remark}
The proof relies solely on the non-negativity of $D_{\enabst}$ and the fact that $D_{\enabst}(P,Q) = 0$ if and only if $P=Q$. Consequently, the proposition applies to a broad class of divergences and distances, including (but not limited to) the Kullback–Leibler (KL) divergence, the Wasserstein distance $W_p$ for any $p \geq 1$, the Total Variation distance, the Hellinger distance, and others.
\end{remark}

\paragraph{Concentration Results for Joint Environments in Causal Abstraction}\label{sec:concineqapp}
In our causal abstraction setting, each SCM is associated with its own environment: \(\exprobbase\) for the low-level model and \(\exprobabst\) for the high-level model. We assume these environments are independent and define the joint environment as the product measure \(\exprob = \exprobbase \otimes \exprobabst\). In practice, however, only empirical samples from each environment are available, leading to an empirical joint distribution \(\widehat{\exprob} = \widehat{\exprobbase} \otimes \widehat{\exprobabst}\). It is therefore essential to quantify how well the empirical product distribution approximates the true joint environment. To this end, we establish a concentration result in Wasserstein-2 distance, showing that, with high probability, the true joint environment \(\exprob\) lies within a 2-Wasserstein product ball centered at \(\widehat{\exprob}\) for an appropriately chosen radius. This result provides the theoretical foundation for defining a distributionally robust causal abstraction objective over joint uncertainty sets. In particular, the following theorems provide principled guidelines for selecting the radius of the joint ambiguity set, ensuring that it contains the true data-generating process. The results build upon: \textbf{(a)} the tensorization property of the Wasserstein distance~\citep{villani2009optimal}; \textbf{(b)} concentration inequalities for elliptical and empirical distributions~\citep[Theorems 18 and 21]{kuhn2019wassersteindistributionallyrobustoptimization}  and the independence assumption between \(\exprobbase\) and \(\exprobabst\). We first introduce the notion of a light-tailed distribution, which characterizes the tail behavior necessary for applying the concentration results:
\begin{defn}[Light-tailed distribution]
A probability distribution $\mathbb{P}$ over $\mathbb{R}^d$ is said to be \emph{$\alpha$-light-tailed} if there exist constants $\alpha > 0$ and $A > 0$ such that
$\mathbb{E}^{\mathbb{P}}\left[\exp\left(\|\xi\|_2^\alpha\right)\right] \leq A$,
where $\xi \sim \mathbb{P}$. 
\end{defn}
Unlike in~\citep{kuhn2019wassersteindistributionallyrobustoptimization}, where concentration is studied for a single distribution, our setting involves a product of two independent distributions corresponding to the low- and high-level SCMs. Consequently, our concentration theorems specifically account for this product structure.

\textbf{Assumption 1.} Both $\exprobbase$ and $\exprobabst$ are light-tailed distributions.

\textbf{Theorem 1} (Gaussian $\exprob$-Concentration) \textit{Let $\exprobbase \sim \cN(\mu_\ell, \Sigma_\ell)$ and $\exprobabst \sim \cN(\mu_h, \Sigma_h)$ under Ass.1, let $\widehat{\exprobbase}$ and $\widehat{\exprobabst}$ be the empirical distributions from $N_\ell$ and $N_h$ i.i.d.\ samples. Also, let $\exprob := \exprobbase \otimes \exprobabst$ and $\widehat{\exprob} := \widehat{\exprobbase} \otimes \widehat{\exprobabst}$. Then, $\forall~d\in \{\ell, h\}$, there exist constants $c_d> 0$, depending only on the individual $d$-environment and confidence levels $\eta_d$, such that $\forall \delta \in (0,1]$, with $\delta = 1-(1-\eta_\ell)(1-\eta_h)$:}
\begin{align*}
\mathbb{P}\left[ \mathcal{W}_2(\exprob, \widehat{\exprob}) \le \radprod \right] \ge 1-\delta,~~\forall~\radprod \ge \sqrt{\left(\frac{\log(c_\ell / \eta_\ell)}{\sqrt{N_\ell}}\right)^2 + \left(\frac{\log(c_h / \eta_h)}{\sqrt{N_h}}\right)^2}
\end{align*}
\begin{proof}
Our goal is to identify a radius \(\radprod > 0\) such that:
\begin{align}
\mathbb{P}\left[ \mathcal{W}_2(\exprob, \widehat{\exprob}) \le \radprod \right] \ge 1 - \delta.
\end{align}
Let the global event $E_p := \{\mathcal{W}_2(\exprob, \widehat{\exprob}) \le \radprod\}$. From ~\citep{villani2009optimal}, we know that the squared 2-Wasserstein distance tensorizes for product measures. That is, for $\exprob = \exprobbase \otimes \exprobabst$ and $\widehat{\exprob} = \widehat{\exprobbase} \otimes \widehat{\exprobabst}$, we have:
\begin{align}
    \mathcal{W}_2^2(\exprob, \widehat{\exprob}) = \mathcal{W}_2^2(\exprobbase, \widehat{\exprobbase}) + \mathcal{W}_2^2(\exprobabst, \widehat{\exprobabst}).
\end{align}
Therefore, the event $E_p$ is equivalent to the event:
\begin{align}
\left\{ \mathcal{W}_2^2(\exprobbase, \widehat{\exprobbase}) + \mathcal{W}_2^2(\exprobabst, \widehat{\exprobabst}) \le \radprod^2 \right\}
\end{align}
Let us now define the per-component local events:
\begin{align}
    E_\ell := \{\mathcal{W}_2(\exprobbase, \widehat{\exprobbase}) \le \radbase\} \quad \text{and} \quad E_h := \{\mathcal{W}_2(\exprobabst, \widehat{\exprobabst}) \le \radabst\}.
\end{align}
If the intersection event $E_\ell \cap E_h$ occurs, then both $\mathcal{W}_2^2(\exprobbase, \widehat{\exprobbase}) \le \radbase^2$ and $\mathcal{W}_2^2(\exprobabst, \widehat{\exprobabst}) \le \radabst^2$ hold. 

Summing these inequalities yields $\mathcal{W}_2^2(\exprobbase, \widehat{\exprobbase}) + \mathcal{W}_2^2(\exprobabst, \widehat{\exprobabst}) \le \radbase^2 + \radabst^2$. By the tensorization property, this is exactly the condition defining the event $E_p$ and it suffices to select: $\radprod^2 = \radbase^2 + \radabst^2$. Thus, the intersection event $E_\ell \cap E_h$ is a subset of the target event $E_p$:
\begin{align}
    E_\ell \cap E_h \subseteq E_p.
\end{align}
Consequently, to establish a lower bound on $\mathbb{P}(E_p)$, it suffices to lower bound the probability of the intersection:
\begin{align}\label{eq:step1}
\mathbb{P}(E_p) \ge \mathbb{P}(E_\ell \cap E_h)
\end{align}
However, by construction, the empirical measures $\widehat{\exprobbase}$ and $\widehat{\exprobabst}$ are generated independently from $\exprobbase$ and $\exprobabst$. Therefore, the events $E_\ell$ and $E_h$ are statistically independent. This allows us to factor the probability of the intersection according to the product rule for independent events:
\begin{align}\label{eq:step2}
\mathbb{P}(E_\ell \cap E_h) = \mathbb{P}(E_\ell) \mathbb{P}(E_h).
\end{align}
Next, from Theorem 21 of~\citet{kuhn2019wassersteindistributionallyrobustoptimization}, we have that for any confidence level $\eta_i \in (0,1]$:
\begin{align}\label{eq:step3}
\mathbb{P}(E_i)=\mathbb{P}(\mathcal{W}_2(\rho_i,\widehat{\rho}_i)\le\varepsilon_i)\ge1-\eta_i,
\end{align}
whenever $\varepsilon_i \geq \Tilde{\varepsilon_i}=\frac{\log(c_i/\eta_i)}{\sqrt{N_i}}$, for suitable constants \(c_i>0\) depending on the dimension and tail properties of each distribution \(\rho_i\).

Consequently, by substituting the Eq.s~\ref{eq:step2} and \ref{eq:step3} into the inequality of Eq.~\ref{eq:step1}, we obtain:
\begin{align}
    \mathbb{P}(E_p) \ge \mathbb{P}(E_\ell \cap E_h) = \mathbb{P}(E_\ell) \mathbb{P}(E_h) \ge (1 - \eta_\ell)(1 - \eta_h).
\end{align}
Recalling that $E_p =\{\mathcal{W}_2(\exprob, \widehat{\exprob}) \le \sqrt{\radbase^2 + \radabst^2}\}$ and by expressing the desired global confidence $1-\delta =(1 - \eta_\ell)(1 - \eta_h)$\footnote{A simple choice could be the uniform allocation, given by \(\eta_\ell = \eta_d =1-(1-\delta)^{1/2}\).}, we have shown:
\begin{align}
    \mathbb{P}\left[ \mathcal{W}_2(\exprob, \widehat{\exprob}) \le \radprod \right] \ge (1 - \eta_\ell)(1 - \eta_h)=1-\delta.
\end{align}
whenever $\radprod \ge \sqrt{\Tilde{\radbase}^2 + \Tilde{\radabst}^2}$, completing the proof.
\end{proof}
\textbf{Theorem 2} (Empirical $\exprob$-Concentration) \textit{Let $\widehat{\exprobbase}$ and $\widehat{\exprobabst}$ empirical distributions, under Ass.1, from $N_\ell$ and $N_h$ i.i.d. samples, with $\exprob := \exprobbase \otimes \exprobabst$ and $\widehat{\exprob} := \widehat{\exprobbase} \otimes \widehat{\exprobabst}$. Then, for $\forall~d\in \{\ell, h \}$, there exist constants $c_{d,1}, c_{d,2} > 0$, depending only on the individual $d$-environment and confidence levels $\eta_d$, such that $\forall \delta \in (0,1]$, with $\delta = 1-(1-\eta_\ell)(1-\eta_h)$, for $N_d(c,\eta) = \log(c_{d,1}/\eta)/c_{d,2}$, if we set:}
\begin{align*}
&\Tilde{\varepsilon_d} = \left( \dfrac{\log(c_{d,1}/\eta)}{c_{d,2} N_d} \right)^{\min\{1/d, 1/2\}} 
\quad \text{if } N_d \geq N_d(c,\eta), 
\quad \text{or} \quad 
\left( \dfrac{\log(c_{d,1}/\eta)}{c_{d,2} N_d} \right)^{1/\alpha_d} 
\quad \text{otherwise,} \\
&\phantom{\varepsilon_d \ge{}} \implies\quad 
\mathbb{P}\left[ \mathcal{W}_2(\exprob, \widehat{\exprob}) \le \radprod \right] 
\ge 1 - \delta, 
\quad \forall~\radprod \ge \sqrt{\Tilde{\radbase}^2 + \Tilde{\radabst}^2}
\end{align*}
\begin{proof}
The core structure of the proof follows exactly that of Theorem 1. Specifically, we define the global event $E_p := \{\mathcal{W}_2(\exprob, \widehat{\exprob}) \le \radprod\}$ and local events $E_d := \{\mathcal{W}_2(\rho^d, \widehat{\rho^d}) \le \varepsilon_d\}$, for $d \in \{\ell, h\}$. Using the tensorization property of $\mathcal{W}_2^2$ and the independence of the samples generating $\widehat{\exprobbase}$ and $\widehat{\exprobabst}$, we arrive at the same intermediate conclusion as in the proof of Theorem 1:
\[
\mathbb{P}(E_p) \ge \mathbb{P}(E_\ell \cap E_h) = \mathbb{P}(E_\ell) \mathbb{P}(E_h),~~~~\text{provided that}~~\radprod^2 = \radbase^2 + \radabst^2.
\]

The key difference from Theorem 1 lies in the specific concentration inequality used to bound the probabilities of the local events $\mathbb{P}(E_d)$. Here, since $\exprobbase_d$ are assumed to be general light-tailed distributions, we apply Theorem 18 of~\citet{kuhn2019wassersteindistributionallyrobustoptimization}. This theorem guarantees that for any $\eta_d \in (0,1]$, $\mathbb{P}(E_d) = \mathbb{P}(\mathcal{W}_2(\rho^d, \widehat{\rho^d}) \le \varepsilon_d) \ge 1-\eta_d$, provided that $\varepsilon_d$ is chosen to be at least the threshold value $\Tilde{\varepsilon_d}$ where:
\[
\Tilde{\varepsilon_d} :=
\begin{cases}
\left( \dfrac{\log(c_{d,1}/\eta_d)}{c_{d,2} N_d} \right)^{\min\{1/d, 1/2\}} & \text{if } N_d \geq N_d(c,\eta),  \\[2ex]
\left( \dfrac{\log(c_{d,1}/\eta_d)}{c_{d,2} N_d} \right)^{1/\alpha_d} & \text{otherwise.}
\end{cases}
\]
To achieve the overall confidence $1-\delta$, we select local confidence levels $\eta_\ell, \eta_h \in (0,1]$ such that $(1 - \eta_\ell)(1 - \eta_h) = 1-\delta$, and then by substituting the bounds $\mathbb{P}(E_d) \ge 1-\eta_d$ into the combined probability inequality, we yield:
\begin{align}
\mathbb{P}(E_p) \ge (1 - \eta_\ell)(1 - \eta_h) = 1-\delta.
\end{align}
Thus, the result $\mathbb{P}\left[ \mathcal{W}_2(\exprob, \widehat{\exprob}) \le \radprod \right] \ge 1-\delta$ is established once again, whenever $\radprod \ge \sqrt{\Tilde{\radbase}^2 + \Tilde{\radabst}^2}$.
\end{proof}
\begin{remark}
These results justify the use of joint Wasserstein ambiguity sets centered at the empirical product environment. They ensure that, with high probability, the true environment lies within a ball of radius $\radprod$, allowing us to formulate a robust abstraction objective that is consistent with finite-sample uncertainty. The construction is especially well-suited for our setting, where only a single environment is available from each SCM.
\end{remark}

\subsubsection{Provable Robustness}
\label{sec:analytical_props}

In this section, we prove the core theoretical guarantee of our method: the Generalized Empirical Objective is a tractable, convex min-max problem whose solution $\lintau^\star$ comes with a provable robustness certificate in terms of the abstraction error. Below, we present the main Theorem~\ref{thm:provable-robustness}, which establishes this result.

For each intervention $\iota$ and abstraction matrix $\lintau \in \mathbb{R}^{d_h \times d_\ell}$, recall the \textbf{nominal misalignment}:
\begin{align}
    Z_{\lintau}^{\iota}(\mathbf{0})
    \;\coloneq\;
    \lintau\bigl(D_\ell^{(\iota)} + U^\ell\bigr)^\top
    - \bigl(D_h^{(\omega(\iota))} + U^h\bigr)
    \;\in\; \mathbb{R}^{d_h \times n_\iota}.
\end{align}
Let $\mathbf{\Theta} = (\Theta_\ell, \Theta_h)$ denote the adversarial perturbations constrained by Frobenius budgets $r,s \ge 0$. The \textbf{perturbed misalignment} is defined as $Z_{\lintau}^{\iota}(\mathbf{\Theta}) \coloneq Z_{\lintau}^{\iota}(\mathbf{0}) + (\lintau\Theta_\ell^\top - \Theta_h)$.

We define the worst-case expected loss as:
\begin{align}
\label{eq:fdef}
\zeta(\lintau)
\;\coloneq\;
\sup_{\substack{\|\Theta_\ell\|_F \le r \\ \|\Theta_h\|_F \le s}}
\mathbb{E}_{\iota \sim q}\left[ 
\bigl\| Z_{\lintau}^{\iota}(\mathbf{\Theta}) \bigr\|_F^2
\right]
=
\sup_{\substack{\|\Theta_\ell\|_F \le r \\ \|\Theta_h\|_F \le s}}
\mathbb{E}_{\iota \sim q}\left[ 
\bigl\| Z_{\lintau}^{\iota}(\mathbf{0}) + \lintau\Theta_\ell^\top - \Theta_h \bigr\|_F^2
\right].
\end{align}
Our DiRoCA estimator solves the min-max problem:
\begin{equation}
\label{eq:anm_diroca_problem}
\min_{\lintau \in \mathbb{R}^{h \times \ell}} \zeta(\lintau).
\end{equation}
We now show that \eqref{eq:anm_diroca_problem} is a tractable convex program and that its optimum provides a formal certificate of robustness. We will need the following two lemmas.

\begin{lemma}[Inner maximization bound]\label{lem:inner-max}
Let $\lintau\in\mathbb{R}^{h\times \ell}$ and $Z\in\mathbb{R}^{h\times N}$ be fixed. Consider the maximization problem over perturbations $\Theta_\ell, \Theta_h$ within budgets $r, s$:
\begin{equation}\label{eq:inner-problem}
\max_{\|\Theta_\ell\|_F\le r,\ \|\Theta_h\|_F\le s}\ \ \big\|\, Z + \lintau\Theta_\ell^\top - \Theta_h \,\big\|_F^2.
\end{equation}
The value of this maximization is exactly:
\begin{equation}
\Big(\, s \;+\; \|Z\|_F \;+\; r\,\|\lintau\|_2 \,\Big)^{\!2}.
\end{equation}
\end{lemma}

\begin{proof}
The proof proceeds by decomposing the joint maximization into two sequential stages: first analytically solving for the optimal $\Theta_h$ given a fixed $\Theta_\ell$, and subsequently maximizing the resulting expression over $\Theta_\ell$.

\textbf{(i) Maximization of $\Theta_h$.}
Fix $\lintau$ and $\Theta_\ell$, and let $A \coloneq Z+\lintau\Theta_\ell^\top$. The adversary's problem is to choose a $\Theta_h$ that solves:
\[
\max_{\|\Theta_h\|_F\le s} \|A - \Theta_h\|_F^2.
\]
We expand the squared Frobenius norm:
\[
\|A - \Theta_h\|_F^2 = \langle A - \Theta_h, A - \Theta_h \rangle = \|A\|_F^2 - 2\langle A, \Theta_h \rangle + \|\Theta_h\|_F^2.
\]
To maximize this expression, the adversary controls both the magnitude and direction of $\Theta_h$, subject to the budget $\|\Theta_h\|_F \le s$.
\begin{itemize}
    \item \textbf{Magnitude:} The term $\|\Theta_h\|_F^2$ is non-negative. It is maximized when the adversary uses their entire budget, setting $\|\Theta_h\|_F = s$.
    \item \textbf{Direction:} The term $-2\langle A, \Theta_h \rangle$ is maximized when the inner product $\langle A, \Theta_h \rangle$ is minimized. By the Cauchy--Schwarz inequality, the minimum value is $-\|A\|_F \|\Theta_h\|_F$, achieved when $\Theta_h$ points in the opposite direction of $A$ ($\Theta_h = -c A$ for $c>0$).
\end{itemize}
Both conditions are satisfied simultaneously by choosing:
\[
\Theta_h^\star = -s \frac{A}{\|A\|_F}.
\]
Substituting this back into the objective yields:
\begin{align*}
\|A - \Theta_h^\star\|_F^2 
&= \|A\|_F^2 - 2\langle A, -s \frac{A}{\|A\|_F} \rangle + s^2 
= \|A\|_F^2 + 2s\|A\|_F + s^2 
= \Big( \|A\|_F + s \Big)^2.
\end{align*}
Replacing $A$ with $Z+\lintau\Theta_\ell$, we have:
\[
\max_{\|\Theta_h\|_F\le s} \|Z+\lintau\Theta_\ell^\top - \Theta_h\|_F^2 = (\|Z+\lintau\Theta_\ell^\top\|_F + s)^2.
\]

\textbf{(ii) Maximization of $\Theta_\ell$.}
The full inner problem reduces to maximizing $\big(\|Z+\lintau\Theta_\ell^\top\|_F + s\big)^2$ over $\|\Theta_\ell\|_F\le r$. Since the function $h(x) = (x+s)^2$ is monotonically increasing for $x \ge 0$, this is equivalent to maximizing the term $\|Z+\lintau\Theta_\ell^\top\|_F$.

Using the triangle inequality and the definition of the spectral norm, we obtain the upper bound directly:
\begin{align*}
\max_{\|\Theta_\ell\|_F\le r} \|Z+\lintau\Theta_\ell^\top\|_F 
&\leq \|Z\|_F + \max_{\|\Theta_\ell\|_F\le r} \|\lintau\Theta_\ell^\top\|_F \\
&\leq \|Z\|_F + \max_{\|\Theta_\ell\|_F\le r} \|\lintau\|_2 \|\Theta_\ell\|_F \\
&= \|Z\|_F + r\|\lintau\|_2.
\end{align*}
Adding back $s$ and squaring gives the final result: $(\, s + \|Z\|_F + r\|\lintau\|_2 \,)^2$.
\end{proof}

\begin{lemma}[Convexity of the Outer Minimization]\label{lem:outer-min}
The objective function $\zeta(\lintau)$ defined in \eqref{eq:fdef} is convex in $\lintau$.
\end{lemma}
\begin{proof}
For any fixed perturbation tuple $c = (\Theta_\ell,\Theta_h)$, define the loss function $g_c(\lintau) \coloneq \mathbb{E}_{\iota \sim q}[ \|\, Z_{\lintau}^{\iota}(\mathbf{0}) + \lintau\Theta_\ell^\top - \Theta_h \,\|_F^2 ]$.
For a fixed $\iota$, the term inside the expectation is the squared Frobenius norm of an affine function of $\lintau$, which is convex. Since the expectation is a linear operator preserving convexity, $g_c(\lintau)$ is convex in $\lintau$.
The worst-case loss $\zeta(\lintau)$ is the pointwise supremum of the family $\{g_c(\lintau)\}_c$. Since the supremum of convex functions is convex, $\zeta(\lintau)$ is convex.
\end{proof}

\begin{theorem}[Provable Robustness]\label{thm:provable-robustness}
Let $\zeta(\lintau)$ be the worst-case expected loss defined in \eqref{eq:fdef}. Then, for every abstraction matrix $\lintau \in \mathbb{R}^{d_h \times d_\ell}$:
\begin{equation}\label{eq:global-upper-bound}
\zeta(\lintau)
\;\le\;
\mathbb{E}_{\iota \sim q}\Big[ \big( s + \|Z_{\lintau}^{\iota}(\mathbf{0})\|_F + r\|\lintau\|_2 \big)^2 \Big].
\end{equation}
Consequently, minimizing this upper bound (via our proposed algorithm) minimizes a valid upper bound on the worst-case abstraction error.
\end{theorem}

\begin{proof}
The proof relies on the exchange inequality $\sup \mathbb{E}[\cdot] \le \mathbb{E} \sup[\cdot]$. Let $\mathcal{C}$ denote the feasible set of perturbations $\{(\Theta_\ell, \Theta_h) : \|\Theta_\ell\|_F \le r, \|\Theta_h\|_F \le s\}$.
We start with the definition of $\zeta(\lintau)$:
\begin{align}
\zeta(\lintau) &= \sup_{(\Theta_\ell, \Theta_h) \in \mathcal{C}} \mathbb{E}_{\iota \sim q} \left[ \bigl\| Z_{\lintau}^{\iota}(\mathbf{0}) + \lintau\Theta_\ell^\top - \Theta_h \bigr\|_F^2 \right] \\
&\le \mathbb{E}_{\iota \sim q} \left[ \sup_{(\Theta_\ell, \Theta_h) \in \mathcal{C}} \bigl\| Z_{\lintau}^{\iota}(\mathbf{0}) + \lintau\Theta_\ell^\top - \Theta_h \bigr\|_F^2 \right] \label{eq:inequality_step}
\end{align}
The inequality in \eqref{eq:inequality_step} relaxes the problem by allowing the adversary to choose optimal perturbations specific to each intervention $\iota$ independently, rather than a single fixed perturbation for the entire distribution.
We can now apply Lemma~\ref{lem:inner-max} pointwise to the inner maximization term inside the expectation. For a fixed $\iota$, Lemma~\ref{lem:inner-max} guarantees:
\[
\sup_{(\Theta_\ell, \Theta_h) \in \mathcal{C}} \bigl\| Z_{\lintau}^{\iota}(\mathbf{0}) + \lintau\Theta_\ell^\top - \Theta_h \bigr\|_F^2
\;=\;
\bigl( s + \|Z_{\lintau}^{\iota}(\mathbf{0})\|_F + r\|\lintau\|_2 \bigr)^2.
\]
Substituting this back into \eqref{eq:inequality_step} yields:
\[
\zeta(\lintau) \le \mathbb{E}_{\iota \sim q} \left[ \bigl( s + \|Z_{\lintau}^{\iota}(\mathbf{0})\|_F + r\|\lintau\|_2 \bigr)^2 \right],
\]
which concludes the proof.
\end{proof}

\paragraph{Justification of the Generalized Objective.}
We adopt the squared Frobenius norm to quantify the discrepancy between abstracted low-level samples and their high-level counterparts. This choice is theoretically grounded in optimal transport: it corresponds to the transport cost under the \emph{empirical identity coupling}. Unlike generic optimal transport solvers, which optimize over \emph{all} admissible couplings (potentially allowing arbitrary re-pairings of units), the identity coupling enforces a strict, sample-wise correspondence. Crucially, in our framework this correspondence is not arbitrary. Through \textbf{abduction}, the $i$-th observational datum induces specific exogenous noise values $(u_i^\ell, u_i^h)$ for the low- and high-level SCMs. Consequently, the $i$-th interventional samples at both levels represent counterfactual outcomes of the \emph{same underlying unit}. The identity coupling thus reflects the causally mandated alignment between SCMs.

\begin{proposition}[Frobenius Norm as the Transport Cost of the Identity Coupling]
\label{prop:frobenius_as_cost}
Let $\mathbf{A}, \mathbf{B} \in \mathbb{R}^{n \times d_h}$ be matrices of samples where the $i$-th rows correspond to the same observational unit (via abduction). Let $\hat{\mu}_A$ and $\hat{\mu}_B$ be the associated empirical measures. Then the identity coupling $\pi^\ast = \frac{1}{n}\sum_{i=1}^n \delta_{(a_i,b_i)}$ is a valid element of $\Pi(\hat{\mu}_A,\hat{\mu}_B)$ and its transport cost satisfies:
\begin{equation}
\int \|x-y\|_2^2 \, d\pi^\ast(x,y)
\;=\; \frac{1}{n}\|\mathbf{A}-\mathbf{B}\|_F^2.
\end{equation}
Moreover, since $\mathcal{W}_2^2$ is the infimum over all couplings:
\begin{equation}
\mathcal{W}_2^2(\hat{\mu}_A, \hat{\mu}_B)
\;\le\; \frac{1}{n}\|\mathbf{A}-\mathbf{B}\|_F^2.
\end{equation}
\end{proposition}

\begin{proof}
The equality follows directly from the construction of the identity coupling: the transport cost is the average squared Euclidean distance between paired rows, which is exactly the squared Frobenius norm scaled by $1/n$. For the inequality, recall that the Wasserstein distance is defined as the infimum of the transport cost over the set of all admissible couplings $\Pi(\hat{\mu}_A, \hat{\mu}_B)$. Since $\pi^\ast$ matches the marginals $\hat{\mu}_A$ and $\hat{\mu}_B$, it is a feasible coupling ($\pi^\ast \in \Pi$). Therefore, the cost associated with this specific coupling serves as a valid upper bound on the infimum (the Wasserstein distance).
\end{proof}

\begin{remark}[Frobenius as a Tractable and Causally Correct Surrogate] 
\label{rem:frobenius_as_cost_unified} 
Proposition~\ref{prop:frobenius_as_cost} establishes that the Frobenius objective is not a heuristic; it is the quadratic transport cost under the \emph{causally correct} coupling induced by abduction. In linear Gaussian DiRoCA, this logic leads to the Gelbrich formula (as the optimal coupling is linear). In the general ANM case, solving for the optimal coupling at each step is computationally prohibitive and may ignore unit identity. Minimizing the cost of the identity coupling provides a tractable surrogate that respects causal alignment and upper bounds the Wasserstein distance. Driving the Frobenius norm to zero, therefore guarantees convergence of the interventional distributions in the Wasserstein sense, under the correct unit-level correspondence.
\end{remark}
\subsection{Additive Noise Models and the $(D, U)$ Decomposition}\label{sec:anmapp}
We work with \emph{Markovian} Structural Causal Models (SCMs) $\mathcal{M}=(S,\rho)$, where $S=(\en, \ex, \structfuncset)$ comprises endogenous variables $\en=\{X_i\}_{i=1}^d$, exogenous noises $\ex=\{U_i\}_{i=1}^d$, and structural functions $\structfuncset=\{f_i\}_{i=1}^d$. Markovianity, defined as the mutual independence of the exogenous variables $U_i$ under the environment $\rho$ implies acyclicity, entailing a DAG $\scmdag_{\mathcal{M}}$. We assume faithfulness and causal sufficiency throughout. The  mixing function $\mathbf{g}: \dom{\ex}\to\dom{\en}$ is obtained by recursively composing the structural functions along the topological order of $\scmdag_{\mathcal{M}}$ and yields the reduced form of the SCM:
\begin{equation}
    \en = \mathbf{g}(\ex).
\end{equation}
Every intervention mutilates the SCM and induces a new reduced form $\mathbf{g}_{\iota}$. For the CA setting, we denote the corresponding reduced forms by $\mixfbase_{\iota}$ (low-level) and $\mixfabst_{\omega(\iota)}$ (high-level). The associated interventional observable distributions are the pushforward measures:
\begin{equation}
    \mathbb{P}_{\mathcal{M}^\ell_{\iota}}(\en^\ell) = (\mixfbase_{\iota})_{\#}(\exprobbase)
    \quad \text{and} \quad
    \mathbb{P}_{\mathcal{M}^h_{\omega(\iota)}}(\en^h) = (\mixfabst_{\omega(\iota)})_{\#}(\exprobabst).
\end{equation}
An \emph{Additive Noise Model (ANM)} specifies that the structural assignment for each node is additive in the noise:
\begin{equation}
    X_i = f_i\big(\mathrm{PA}(X_i)\big) + U_i, \quad i\in[d],
\end{equation}
A \emph{Linear Additive Noise (LAN)} model is a special case of an ANM where the structural functions are linear.
\begin{equation}
    \en = \mathbf{B}^\top \en + \ex,
\end{equation}
where $\mathbf{B}$ is (permutable to) a strictly upper triangular matrix. The reduced form is linear:
\begin{equation}
    \en = \mathbf{M} \ex, \qquad \text{with } \mathbf{M} := (\mathbf{I} - \mathbf{B}^\top)^{-1}.
\end{equation}
\textbf{Note:} If the environments $\exprobbase, \exprobabst$ are Gaussian, the resulting pushforwards $(\mixfbase_{\iota})_{\#}(\exprobbase)$ and $(\mixfabst_{\omega(\iota)})_{\#}(\exprobabst)$ remain Gaussian.

In general ANMs, the functions $f_i$ may be arbitrary (e.g., polynomials, kernels, neural networks, etc). Consequently, the reduced form $\mathbf{g}$ is nonlinear in $\ex$, and thus we rely on the $(D, U)$ decomposition described below, which remains valid regardless of structural assumptions.

\subsubsection{The $(D, U)$ Decomposition}
\label{sec:anm_decomposition}

We leverage the property that in an SCM with additive noise, the value of any endogenous variable $X_i$ is the sum of a deterministic mechanism $f_i$ and a stochastic residual $U_i$. Given a known causal graph, identifying the SCM reduces to a set of supervised regression problems. This allows us to decompose any observed data, observational or interventional, into a deterministic component matrix $\mathbf{D}$ (representing the causal mechanisms) and a stochastic component matrix $\mathbf{U}$ (representing the specific noise realizations).

\paragraph{Procedure: Estimation and Abduction.}
Assume the DAG is given. The decomposition proceeds in two phases: \emph{estimation} (learning the mechanisms once) and \emph{abduction} (applying them to data).
\begin{enumerate}
    \item \textbf{Fit Mechanisms (If not known):} For each node $X_i$, we estimate the structural function $\hat{f}_i: \dom{\mathrm{PA}(X_i)}\to\mathbb{R}$ by regressing $X_i$ on its parents $\mathrm{PA}(X_i)$ using the observational data.
    \begin{itemize}
        \item For Linear ANMs, this utilizes standard OLS or Ridge regression.
        \item For Nonlinear ANMs, we employ nonparametric regressors (e.g., Random Forests, Neural Networks).
    \end{itemize}
    
    \item \textbf{Compute Deterministic Components ($\mathbf{D}$):} For any dataset (observational or interventional), we calculate the deterministic matrix $\mathbf{D}$ by evaluating the \emph{fixed} learned functions on the current parent values: $D_i := \hat{f}_i(\mathrm{PA}(X_i))$.
    \begin{itemize}
        \item \emph{Intervention Handling:} If variable $X_j$ is under intervention $\mathrm{do}(X_j=a)$, the learned structural mechanism is overridden. The column for $X_j$ in $\mathbf{D}$ is set deterministically to the constant $a$.
    \end{itemize}
    
    \item \textbf{Abduce Residuals ($\mathbf{U}$):} We recover the specific exogenous noise realizations for the current samples via subtraction: $U_i := X_i - D_i$.
    \begin{itemize}
        \item Under a hard intervention on $X_j$, the variable is fully determined by the intervention; we set the corresponding residual to zero to preserve the additive identity.
    \end{itemize}
\end{enumerate}

Stacking these column-wise yields the matrices $\mathbf{D}^{(\iota)}$ and $\mathbf{U}^{(\iota)}$ for any specific interventional environment $\iota$. Note that while the \emph{distribution} of non-intervened noise remains invariant, the specific \emph{samples} $\mathbf{U}^{(\iota)}$ differ across datasets.

\paragraph{Illustrative Example}\label{sec:anm_example}
We illustrate this decomposition using a simple nonlinear chain $X \to Y \to Z$ with $N=5$ samples. Assume we have already performed the estimation step and obtained:
\begin{align*}
    \hat{f}_X(\emptyset) &= 0, & \hat{f}_Y(X) &= \sin(X), & \hat{f}_Z(Y) &= Y^2.
\end{align*}

\textbf{Case 1: Observational Data ($\iota_0 = \emptyset$).}
We observe data matrix $\mathbf{X}^{(0)} \in \mathbb{R}^{5 \times 3}$. We apply the fitted functions to these samples to separate signal from noise:
\[
\mathbf{X}^{(0)} =
\begin{pmatrix}
x_1 & y_1 & z_1 \\
\vdots & \vdots & \vdots \\
x_5 & y_5 & z_5
\end{pmatrix}
\implies
\mathbf{D}^{(0)} =
\begin{pmatrix}
0 & \sin(x_1) & y_1^2 \\
\vdots & \vdots & \vdots \\
0 & \sin(x_5) & y_5^2
\end{pmatrix}
\]
The specific noise realizations for this batch are $\mathbf{U}^{(0)} = \mathbf{X}^{(0)} - \mathbf{D}^{(0)}$.

\textbf{Case 2: Interventional Data ($\iota_1 = \mathrm{do}(X=1.0)$).}
We collect a \emph{new} batch of 5 samples $\mathbf{X}^{(1)}$. Even though the underlying noise distribution is the same, these are new realizations. We use the \textbf{same} functions $\hat{f}$ to decompose this new data, but override the intervened column. Note that since $X$ is fixed to $1.0$, the parent input for $Y$ is now $1.0$:
\[
\mathbf{X}^{(1)} =
\begin{pmatrix}
1.0 & y'_1 & z'_1 \\
\vdots & \vdots & \vdots \\
1.0 & y'_5 & z'_5
\end{pmatrix}
\implies
\mathbf{D}^{(1)} =
\begin{pmatrix}
\mathbf{1.0} & \sin(1.0) & (y'_1)^2 \\
\vdots & \vdots & \vdots \\
\mathbf{1.0} & \sin(1.0) & (y'_5)^2
\end{pmatrix}
\]
The residual matrix $\mathbf{U}^{(1)} = \mathbf{X}^{(1)} - \mathbf{D}^{(1)}$ captures the new noise samples for $Y$ and $Z$, while the column for $X$ is strictly zero.

\paragraph{The \textsc{DiRoCA} Protocol.} While the standard sampling above involves varying noise across datasets, the \textsc{DiRoCA} optimization (Appendix~\ref{sec:optapp}) utilizes a shared Residuals protocol, as we saw in the main text. During training, we explicitly reuse the observational residuals $\mathbf{U}^{(0)}$ from Case 1 to generate interventional batches as: $\mathbf{X}^{(\iota)} = \mathbf{D}^{(\iota)} + \mathbf{U}^{(0)}$. This protocol is mandated by the empirical DRO formulation defined in Section 4 since we solve for a single worst-case perturbation matrix. As discussed in the main text, this reuse yields a consistent pairing of batch rows across interventions; it is not an alignment assumption but a computational device enabling efficient Frobenius-based discrepancy evaluation and avoiding repeated optimal transport computations necessary for the robust objective since we solve for a single worst-case perturbation of the underlying noise distribution, we must use a single base set of residuals to define the ambiguity set consistently across all environments. Although this technically fixes the noise profile rather than sampling new realizations for every intervention, it effectively forces the model to learn the stable deterministic mechanism $\mathbf{D}^{(\iota)}$ invariant to specific noise samples; empirically, this regularization provides the extra robustness observed during test time, where DiRoCA successfully generalizes to unseen noise realizations. When tested on actual samples where noise naturally varies, the robust DiRoCA method is able to successfully capture the underlying mechanisms, whereas other baselines fail, as demonstrated in the results below.

\paragraph{Adversarial Perturbation Strategy.}
While linear models admit a simplified surrogate through the reduced form of the SCM, this $(D, U)$ decomposition suggests a universal strategy for nonlinear ANMs. By separating the deterministic mechanisms from the stochastic noise, we can inject adversarial perturbations $\Theta$ directly into the abduced residuals $U$. This allows us to evaluate robustness of the abstraction map $\lintau$ by propagating perturbed noise $U + \Theta$ through the fixed (estimated) mechanisms $\hat{f}$, without requiring re-estimation of structural functions during the optimization loop.

\paragraph{Structural Misspecification as Residual Perturbation.}
The same decomposition also clarifies why DiRoCA can remain stable when structural assumptions are violated (e.g., imperfect functional estimation, mild mechanism deviations, or errors in the intervention mapping). If the true mechanism $f_i$ differs from the estimated mechanism $\hat{f}_i$, then for non-intervened nodes we can write
\[
X_i
= f_i(\mathrm{PA}(X_i)) + U_i
= \hat{f}_i(\mathrm{PA}(X_i)) + \underbrace{\big(U_i + \delta_i(\mathrm{PA}(X_i))\big)}_{\tilde U_i},
\qquad
\delta_i(\cdot) := f_i(\cdot) - \hat{f}_i(\cdot).
\]
Thus, from the perspective of the learned model, structural error $\delta_i(\mathrm{PA}(X_i))$ is observationally absorbed into an \emph{effective residual} $\tilde U_i$. Consequently, adversarial perturbations $\Theta$ applied to abduced residuals $U$ can be interpreted as covering both (i) distributional shifts in exogenous noise and (ii) deviations induced by mechanism misspecification or estimation error. This offers a simple explanation for why optimizing against worst-case perturbations in a Wasserstein ambiguity set can improve robustness beyond environmental shift, and it aligns with recent perspectives that model mechanism deviations as structured perturbations of causal systems \citep{schneider2025generativeinterventionmodelscausal}.

\subsection{The Joint Ambiguity Set}\label{sec:ambsetapp}
Here, we provide a visual and formal illustration of the construction of the \emph{joint ambiguity set} used in our framework, as described in the main text. Our goal is to introduce distributional robustness at the level of environments by following the Distributionally Robust Optimization (DRO) paradigm. Since we assume access to a single observed environment from each SCM (low-level and high-level), we model distributional uncertainty using a $2$-Wasserstein product ball centered at the empirical joint environment $\widehat{\exprob}$ that is formally defined as the joint ambiguity set:
\begin{align}
    \bB_{\radprod, 2}(\widehat{\exprob}) := \mathbb{B}_{\radbase, 2}(\widehat{\exprobbase}) \times \mathbb{B}_{\radabst, 2}(\widehat{\exprobabst}),
\end{align}
where for each domain \(d \in \{|\scmbase|, |\scmabst|\}\), the marginal ambiguity set $\mathbb{B}_{\varepsilon_d, 2}(\widehat{\rho^d})$ is a $2$-Wasserstein ball defined as
\begin{align}
    \mathbb{B}_{\varepsilon_d, 2}(\widehat{\rho^d}) = \left\{ \rho \in \mathcal{P}(\ex^d) : W_2(\rho, \widehat{\rho^d}) \leq \varepsilon_d \right\}.
\end{align}
Each radius $\varepsilon_d$ controls the robustness level for the corresponding SCM, with larger values providing robustness to broader shifts but may lead to more conservative solutions.

\paragraph{Construction Steps.}
Since we do not observe exogenous environments directly, we reconstruct them via abduction. Specifically, given samples from the endogenous variables, we invert the reduced form of the SCMs to recover approximate exogenous samples:
\begin{align}
(\mathbf{g}^d_{\#})^{-1}(\widehat{\en^d}) \approx \widehat{\ex^d}, \quad \text{for } d \in \{\ell, h\}.
\end{align}
With known structural functions, this inversion is exact. Otherwise, if the structural functions must be estimated, the quality of the abduction depends on the estimation accuracy of the selected method (e.g. Linear Regression for LAN models, Random Forest/ Neural Networks for general ANMs). After performing abduction, we form the empirical exogenous distributions \(\widehat{\exbase}\) and \(\widehat{\exabst}\) via empirical measures:
\begin{align}
\widehat{\exbase} = \frac{1}{N} \sum_{i=1}^N \delta_{\widehat{\contbase_i}}, \quad
\widehat{\exabst} = \frac{1}{M} \sum_{j=1}^M \delta_{\widehat{\contabst_j}},
\end{align}
where \(\widehat{\contbase_i}\) and \(\widehat{\contabst_j}\) denote the reconstructed exogenous samples at the low and high levels, respectively. Next, we form the nominal joint distribution \(\widehat{\exprob} := \widehat{\exprobbase} \otimes \widehat{\exprobabst}\) by combining the empirical exogenous distributions. Finally, we build the joint ambiguity set \(\bB_{\radprod, 2}(\widehat{\exprob})\) by taking the product of two marginal Wasserstein balls. The ambiguity set \(\bB_{\radprod, 2}(\widehat{\exprob})\) allows us to find the worst-case joint environment \(\exprob^\star = \rho^{\ell\star} \otimes \rho^{h\star}\) that maximizes the abstraction error within the specified robustness radii. This captures the most adversarial distributional shift consistent with the observed data and estimated environments. The radii \(\radbase\) and \(\radabst\) quantify the maximum deviation we aim to protect against at each abstraction level. The full construction process, including the abduction, nominal estimation, and radius selection steps, is illustrated in Fig.~\ref{fig:ambsetconstructBOX} for the case where low- and high-level environments are sampled independently.

\begin{figure}[h]
    \centering
    \begin{tikzpicture}
        \node (A) at (0,0) {\large$\widehat{\enbase}$};
        \node (B) at (2.4,0) {\large$\widehat{\exbase}$};
        \node (C) at (5.2,0) {\large$\widehat{\exprobbase}$};
        \node (D) at (8.7,0) {\large$\mathbb{B}_{\radbase, 2}(\widehat{\exprobbase})$};

        \node (A1) at (0,-2) {\large$\widehat{\enabst}$};
        \node (B1) at (2.4,-2) {\large$\widehat{\exabst}$};
        \node (C1) at (5.2,-2) {\large$\widehat{\exprobabst}$};
        \node (D1) at (8.7,-2) {\large$\mathbb{B}_{\radabst, 2}(\widehat{\exprobabst})$};

        \draw[->] (A) to node[above,font=\small]{$(\scmbase)^{-1}$} (B);
        \draw[->] (B) to node[above,font=\small]{$\frac{1}{N}\sum_{i=1}^{N}\delta_{\widehat{\contbase_i}}$} (C);
        \draw[->] (C) to node[above,font=\small]{$\radbase$} (D);

        \draw[->] (A1) to node[above,font=\small]{$(\scmabst)^{-1}$} (B1);
        \draw[->] (B1) to node[above,font=\small]{$\frac{1}{M}\sum_{j=1}^{M}\delta_{\widehat{\contabst_j}}$} (C1);
        \draw[->] (C1) to node[above,font=\small]{$\radabst$} (D1);

        \node (E) at (10,-1) {\large$\bB_{\radprod, 2}(\widehat{\exprob})$};
        \draw[->] (D) to[bend right=25] (E);
        \draw[->] (D1) to[bend left=25] (E);

        \draw[red, thick, dotted, rounded corners] (-0.6,0.9) rectangle (2.9,-2.6);
        \draw[red, thick, dotted, rounded corners] (3.0,0.9) rectangle (5.6,-2.6);
        \draw[red, thick, dotted, rounded corners] (5.7,0.9) rectangle (10.7,-2.6);

        \node at (1.2,-3.2) {\small \textbf{Abduction}};
        \node at (4.3,-3.2) {\small \textbf{Nominal Estimation}};
        \node at (8.2,-3.2) {\small \textbf{Radius Selection}};
    \end{tikzpicture}
    \caption{Construction of the joint ambiguity set \(\bB_{\radprod, 2}(\widehat{\exprob})\) under the assumption of independently sampled environments. The process involves abduction of exogenous variables, nominal empirical distribution estimation, and selection of robustness radii.}
    \label{fig:ambsetconstructBOX}
\end{figure}

\subsection{Datasets}\label{sec:examples_app}
\subsubsection{Simple Lung Cancer (\texttt{SLC}) Dataset}
The Simple Lung Cancer (\texttt{SLC}) dataset models the causal relationships between three continuous variables: \(S\) denotes \emph{Smoking}, \(T\) denotes \emph{Tar deposition} in the lungs, and \(C\) denotes \emph{Cancer} development. In the low-level representation, smoking causes tar accumulation, which, in turn, causes cancer. The high-level abstraction removes the intermediate variable \(T\), resulting in a direct causal link between smoking and cancer. The experiment includes a set of 6 distinct low-level and 3 high-level binary relevant interventions: \(\intervsetbase = \{ \iota_0, \dots, \iota_5 \}\) and \(\intervsetabst = \{ \eta_0, \eta_1, \eta_2 \}\). Table~\ref{tab:slc_interventions} demonstrates the map $\omega:\intervsetbase \to \intervsetabst$. The corresponding low-level and high-level causal graphs are shown in Fig.~\ref{fig:low_high_graphs_slc}.

\begin{figure}[h]
\centering
\begin{tikzpicture}[->,>=latex,thick]
\node[circle, draw, minimum size=0.9cm] (S) at (0,0) {\(S\)};
\node[circle, draw, minimum size=0.9cm] (T) at (2,0) {\(T\)};
\node[circle, draw, minimum size=0.9cm] (C) at (4,0) {\(C\)};
\draw (S) -- (T);
\draw (T) -- (C);

\node[circle, draw, minimum size=0.9cm] (Sprime) at (7,0) {\(S'\)};
\node[circle, draw, minimum size=0.9cm] (Cprime) at (9,0) {\(C'\)};
\draw (Sprime) -- (Cprime);

\node[align=center, above=0.5cm of T] {\textbf{Low-Level:} $\scmdag_{\scmbase}$};
\node[align=center, above=0.5cm of Sprime, xshift=1cm] {\textbf{High-Level:} $\scmdag_{\scmabst}$};
\end{tikzpicture}
\caption{Low-level (\(\scmdag_{\scmbase}\)) and high-level (\(\scmdag_{\scmabst}\)) causal graphs for the Simple Lung Cancer (\texttt{SLC}) dataset. In the high-level abstraction, the intermediate variable \(T\) (tar deposition) is omitted, resulting in a direct causal link between smoking and cancer.}
\label{fig:low_high_graphs_slc}
\end{figure}

\begin{table}[h]
\centering
\caption{Intervention definitions and the \(\omega\) map for the \texttt{SLC} dataset. The map \(\omega\) aggregates low-level interventions \(\iota\) into high-level interventions \(\eta\).}
\label{tab:slc_interventions}
\begin{tabular}{@{}ll ll@{}}
\toprule
\multicolumn{2}{c}{\textbf{High-Level ($\eta$)}} & \multicolumn{2}{c}{\textbf{Low-Level ($\iota$) mapped to $\eta$}} \\ 
\cmidrule(r){1-2} \cmidrule(l){3-4}
\textbf{Label} & \textbf{Definition} & \textbf{Label} & \textbf{Definition} \\ \midrule
$\eta_0$ & $\emptyset$ (Observational) & $\iota_0$ & $\emptyset$ (Observational) \\ \midrule
\multirow{2}{*}{$\eta_1$} & \multirow{2}{*}{$do(S'=0)$} & $\iota_1$ & $do(S=0)$ \\
 & & $\iota_2$ & $do(S=0, T=1)$ \\ \midrule
\multirow{3}{*}{$\eta_2$} & \multirow{3}{*}{$do(S'=1)$} & $\iota_3$ & $do(S=1)$ \\
 & & $\iota_4$ & $do(S=1, T=0)$ \\
 & & $\iota_5$ & $do(S=1, T=1)$ \\ \bottomrule
\end{tabular}
\end{table}

\subsubsection{LUCAS Datasets (\texttt{LiLUCAS} \& \texttt{nLUCAS})}
The LUng CAncer dataset (LUCAS)\footnote{\href{http://www.causality.inf.ethz.ch/data/LUCAS.html}{http://www.causality.inf.ethz.ch/data/LUCAS.html}}, originally designed to simulate realistic causal relationships in lung cancer diagnosis. We consider two versions sharing the same causal structure and abstraction logic but differing in mechanism complexity: the \emph{Linearized LUCAS} (\texttt{LiLUCAS}) dataset, which uses linear mechanisms, and the \texttt{nLUCAS} dataset, which employs non-linear relationships on the low-level model. The low-level model involves several continuous variables: \(\text{SM}\) denotes \emph{Smoking}, \(\text{GE}\) denotes \emph{Genetics}, \(\text{LC}\) denotes \emph{Lung Cancer}, \(\text{AL}\) denotes \emph{Allergy}, \(\text{CO}\) denotes \emph{Coughing}, and \(\text{FA}\) denotes \emph{Fatigue}. In the high-level abstraction, groups of variables are clustered into broader concepts: \(\text{EN}'\) (Environment), \(\text{GE}'\) (Genetics), and \(\text{LC}'\) (Lung Cancer). While sharing the same abstraction logic, the datasets differ in their interventional configurations. \texttt{LiLUCAS} utilizes a comprehensive set of 21 relevant binary low-level interventions (including the observational state) mapped to 11 high-level interventions: \(\intervsetbase = \{ \iota_0, \dots, \iota_{20} \}\) and \(\intervsetabst = \{ \eta_0, \dots, \eta_{10} \}\). In contrast, \texttt{nLUCAS} employs a subset of 11 low-level interventions while keeping the same high-level ones. Table~\ref{tab:lucas_interventions} demonstrates the full map $\omega:\intervsetbase \to \intervsetabst$ used in the linear setting. The corresponding low-level and high-level causal graphs are shown in Fig.~\ref{fig:lucas_graphs}.









\begin{figure}[h]
\centering
\begin{tikzpicture}[->,>=latex,thick, node distance=1.8cm]


\node[circle, draw, minimum size=0.9cm] (SM) at (0, 0) {SM};
\node[circle, draw, minimum size=0.9cm] (LC) at (1, -2) {LC};
\node[circle, draw, minimum size=0.9cm] (CO) at (-1, -2) {CO};
\node[circle, draw, minimum size=0.9cm] (FA) at (3, -2) {FA};

\node[circle, draw, minimum size=0.9cm] (GE) at (2,0) {GE};
\node[circle, draw, minimum size=0.9cm] (AL) at (-3,-2) {AL};

\draw (SM) -- (LC);
\draw (GE) -- (LC);
\draw (LC) -- (CO);
\draw (LC) -- (FA);
\draw(CO) to[bend right=35] (FA); 
\draw (AL) -- (CO); 

\node[circle, draw, minimum size=0.9cm] (ENp) at (7,0) {EN'};
\node[circle, draw, minimum size=0.9cm] (GEp) at (9,0) {GE'};
\node[circle, draw, minimum size=0.9cm] (LCp) at (8,-2) {LC'};

\draw (ENp) -- (LCp);
\draw (GEp) -- (LCp);

\node[align=center, above=0.5cm of SM, xshift=1cm] {\textbf{Low-Level:} $\scmdag_{\scmbase}$};
\node[align=center, above=0.5cm of ENp, xshift=1cm] {\textbf{High-Level:} $\scmdag_{\scmabst}$};

\end{tikzpicture}
\caption{Low-level (\(\scmdag_{\scmbase}\)) and high-level (\(\scmdag_{\scmabst}\)) causal graphs for the LUCAS datasets (\texttt{LiLUCAS} and \texttt{nLUCAS}). The abstraction groups variables related to environment, genetics, and disease outcomes, resulting in a simplified causal structure.}
\label{fig:lucas_graphs}
\end{figure}

\begin{table}[h]
\centering
\caption{Intervention definitions and the \(\omega\) map for the LUCAS datasets. The map \(\omega\) aggregates low-level interventions \(\iota\) into high-level interventions \(\eta\).}
\label{tab:lucas_interventions}
\begin{tabular}{@{}ll ll@{}}
\toprule
\multicolumn{2}{c}{\textbf{High-Level ($\eta$)}} & \multicolumn{2}{c}{\textbf{Low-Level ($\iota$) mapped to $\eta$}} \\ 
\cmidrule(r){1-2} \cmidrule(l){3-4}
\textbf{Label} & \textbf{Definition} & \textbf{Label} & \textbf{Definition} \\ \midrule
$\eta_0$ & $\emptyset$ (Observational) & $\iota_0$ & $\emptyset$ (Observational) \\ \midrule
$\eta_1$ & $do(\text{EN}'=0)$ & $\iota_1$ & $do(\text{SM}=0)$ \\ \midrule
$\eta_2$ & $do(\text{EN}'=1)$ & $\iota_2$ & $do(\text{SM}=1)$ \\ \midrule
\multirow{2}{*}{$\eta_3$} & \multirow{2}{*}{$do(\text{GE}'=0)$} & $\iota_3$ & $do(\text{LC}=0)$ \\
 & & $\iota_{12}$ & $do(\text{GE}=1, \text{SM}=1)$ \\ \midrule
\multirow{2}{*}{$\eta_4$} & \multirow{2}{*}{$do(\text{GE}'=1)$} & $\iota_4$ & $do(\text{LC}=1)$ \\
 & & $\iota_{13}$ & $do(\text{GE}=0, \text{SM}=1)$ \\ \midrule
\multirow{3}{*}{$\eta_5$} & \multirow{3}{*}{$do(\text{EN}'=0, \text{GE}'=0)$} & $\iota_5$ & $do(\text{SM}=0, \text{LC}=0)$ \\
 & & $\iota_{10}$ & $do(\text{GE}=1)$ \\
 & & $\iota_{16}$ & $do(\text{AL}=1)$ \\ \midrule
\multirow{3}{*}{$\eta_6$} & \multirow{3}{*}{$do(\text{EN}'=1, \text{GE}'=1)$} & $\iota_6$ & $do(\text{SM}=1, \text{LC}=1)$ \\
 & & $\iota_{11}$ & $do(\text{GE}=0, \text{SM}=0)$ \\
 & & $\iota_{17}$ & $do(\text{CO}=0)$ \\ \midrule
$\eta_7$ & $do(\text{EN}'=0, \text{GE}'=1)$ & $\iota_{15}$ & $do(\text{AL}=0)$ \\ \midrule
\multirow{2}{*}{$\eta_8$} & \multirow{2}{*}{$do(\text{EN}'=1, \text{GE}'=0)$} & $\iota_{14}$ & $do(\text{GE}=1, \text{SM}=0)$ \\
 & & $\iota_{20}$ & $do(\text{CO}=0, \text{FA}=1)$ \\ \midrule
\multirow{2}{*}{$\eta_9$} & \multirow{2}{*}{$do(\text{LC}'=0)$} & $\iota_7$ & $do(\text{SM}=0, \text{LC}=1)$ \\
 & & $\iota_{19}$ & $do(\text{CO}=1, \text{FA}=0)$ \\ \midrule
\multirow{3}{*}{$\eta_{10}$} & \multirow{3}{*}{$do(\text{LC}'=1)$} & $\iota_8$ & $do(\text{SM}=1, \text{LC}=0)$ \\
 & & $\iota_{18}$ & $do(\text{CO}=1, \text{FA}=1)$ \\
\bottomrule
\end{tabular}%
\end{table}

\subsubsection{Electric Battery Manufacturing (\texttt{EBM})  \citep{zennaro2023jointly}}
The Electric Battery Manufacturing (\texttt{EBM}) dataset applies our framework to a real-world scenario involving lithium-ion battery production. This dataset contains data collected from two distinct experimental settings representing different levels of granularity. The first setting (WMG) has been modeled through a low-level SCM that captures the effect of a control variable (comma gap) on an output (mass loading) at multiple spatial locations. The second setting (LRCS) is modeled through a high-level SCM that relates the same control variable to a single aggregated output.  The low-level variables are \(\mathbf{X}_L = [\text{CG}, \text{ML}_1, \text{ML}_2]^\top\), where \(\text{CG}\) denotes the \emph{Comma Gap} (in \(\mu m\)) and \(\text{ML}_{1,2}\) denote \emph{Mass Loading} measurements (in \(mg/cm^2\)) at two distinct spatial locations. The high-level variables are \(\mathbf{X}_H = [\text{CG}^\prime, \text{ML}]^\top\), where mass loading is treated as a single scalar.

\paragraph{SCM Construction and Abduction}
Unlike the synthetic datasets where structural functions are known, here we must infer them from data. A key methodological divergence from the original work by \citet{zennaro2023jointly}, which used a non-parametric empirical SCM, is our use of a parametric ANM This choice allows us to fully utilize our generalized optimization objective. We fit linear mechanisms with intercepts to both datasets, respecting the causal graph structures (\(\text{CG} \to \text{ML}\)):
\begin{equation}
    X_i = \sum_{j \in \mathrm{Pa}(i)} \beta_{ji} X_j + c_i + U_i,
\end{equation}
where \(\beta_{ji}\) are the learned coefficients and \(c_i\) are the intercepts. We perform exact abduction to recover the noise terms by computing the residuals: \(\hat{U}_i = X_i - (c_i + \sum \beta_{ji} X_j)\). To ensure robustness, the noise is centered to be mean-zero per intervention bucket. The deterministic component \(D^{(\iota)}\) for an intervention \(\iota = do(\text{CG}=c)\) is then defined as:
\begin{equation}
    D_i^{(\iota)} =
    \begin{cases}
    c & \text{if } i = \text{CG}, \\
    c_i + \sum_{j \in \mathrm{Pa}(i)} \beta_{ji} \cdot c & \text{if } i \in \mathrm{Ch}(\text{CG}), \\
    c_i & \text{otherwise}.
    \end{cases}
\end{equation}
Finally, we use the set of real interventions performed during the data collection. The continuous Comma Gap values are aligned via the map \(\omega\), as detailed in Table~\ref{tab:battery_interventions}.
\begin{figure}[h]
\centering
\begin{tikzpicture}[->,>=latex,thick, node distance=1.5cm]

\node[circle, draw, minimum size=0.9cm] (CG) at (0, 0) {CG};
\node[circle, draw, minimum size=0.9cm] (ML1) at (-1.2, -2) {$\text{ML}_1$};
\node[circle, draw, minimum size=0.9cm] (ML2) at (1.2, -2) {$\text{ML}_2$};

\draw (CG) -- (ML1);
\draw (CG) -- (ML2);

\node[circle, draw, minimum size=0.9cm] (CGp) at (6, 0) {CG$^\prime$};
\node[circle, draw, minimum size=0.9cm] (MLp) at (6, -2) {ML};

\draw (CGp) -- (MLp);

\node[align=center, above=0.5cm of CG] {\textbf{Low-Level:} $\scmdag_{\scmbase}$};
\node[align=center, above=0.5cm of CGp] {\textbf{High-Level:} $\scmdag_{\scmabst}$};

\end{tikzpicture}
\caption{Low-level (\(\scmdag_{\scmbase}\)) and high-level (\(\scmdag_{\scmabst}\)) causal graphs for the \texttt{EBM} dataset. The low-level model captures spatial granularity ($\text{ML}_1, \text{ML}_2$), while the high-level model represents an aggregated view ($\text{ML}$).}
\label{fig:battery_graphs}
\end{figure}

\begin{table}[h]
\centering
\caption{Intervention definitions and the \(\omega\) map for the \texttt{EBM} dataset. Interventions correspond to setting the Comma Gap (\(\text{CG}\)) to specific values (in \(\mu m\)).}
\label{tab:battery_interventions}
\begin{tabular}{@{}ll ll@{}}
\toprule
\multicolumn{2}{c}{\textbf{High-Level ($\eta$)}} & \multicolumn{2}{c}{\textbf{Low-Level ($\iota$) mapped to $\eta$}} \\ 
\cmidrule(r){1-2} \cmidrule(l){3-4}
\textbf{Label} & \textbf{Definition} & \textbf{Label} & \textbf{Definition} \\ \midrule
$\eta_0$ & $\emptyset$ (Observational) & $\iota_0$ & $\emptyset$ (Observational) \\ \midrule
$\eta_1$ & $do(\text{CG}=75)$ & $\iota_1$ & $do(\text{CG}^\prime=75)$ \\ \midrule
$\eta_2$ & $do(\text{CG}=100)$ & $\iota_2$ & $do(\text{CG}^\prime=110)$ \\ \midrule
\multirow{2}{*}{$\eta_3$} & \multirow{2}{*}{$do(\text{CG}=200)$} & $\iota_3$ & $do(\text{CG}^\prime=180)$ \\
 & & $\iota_4$ & $do(\text{CG}^\prime=200)$ \\
\bottomrule
\end{tabular}
\end{table}

\subsubsection{Colored MNIST (\texttt{cMNIST}) \citep{xia2024neural}}
This experiment tests our framework on a high-dimensional, non-linear task based on the Colored MNIST (\texttt{cMNIST}) dataset. The objective is to learn a robust \emph{linear} abstraction $\lintau$ that aligns a complex pixel-level SCM with a disentangled latent-space SCM. While both levels operate on the same graph structure visualized in Fig.~\ref{fig:cmnist_graph} with parents Digit $D$ and Color $C$, which are correlated ($p=0.85$), they are distinguished by their collider targets. The low-level target $X^\ell$ is instantiated as the high-dimensional image $I_P \in \mathbb{R}^{3072}$, while the high-level target $X^h$ is the compact latent code $z \in \mathbb{R}^{64}$ obtained via the pre-trained encoder $E_\phi$ \citep{xia2024neural}.

We utilize a set of 10 distinct interventions $\mathcal{I}$ alongside the observational data. The set includes atomic interventions (e.g., $\text{do}(D=6)$) and joint interventions (e.g., $\text{do}(D=4, C=4)$). These interventions are critical because the observational data is heavily confounded (e.g., digit 0 is strongly correlated with red). Interventions break these spurious correlations, creating counterfactual combinations (e.g., digit 0 with blue) that are absent in the training distribution. Since both SCMs share identical parent definitions, the alignment map is the identity: $\omega(\iota) = \iota$.
\begin{figure}[h]
\centering
\begin{tikzpicture}[->,>=latex,thick]
\node[circle, draw, minimum size=0.9cm] (D) at (0,1.5) {\(D\)};
\node[circle, draw, minimum size=0.9cm] (C) at (0,-1.5) {\(C\)};
\node[circle, draw, minimum size=1.6cm] (I) at (2,0) {\(I_P\)};
\draw (D) -- (I);
\draw (C) -- (I);

\node[circle, draw, minimum size=0.9cm] (D2) at (6,1.5) {\(D\)};
\node[circle, draw, minimum size=0.9cm] (C2) at (6,-1.5) {\(C\)};
\node[circle, draw, minimum size=0.9cm] (Z) at (8,0) {\(z\)};
\draw (D2) -- (Z);
\draw (C2) -- (Z);

\node[align=center, above=0.5cm of I, yshift=1cm] {\textbf{Low-Level:} $\scmdag_{\scmbase}$};
\node[align=center, above=0.5cm of Z, yshift=1cm] {\textbf{High-Level:} $\scmdag_{\scmabst}$};
\end{tikzpicture}
\caption{Causal graphs for \texttt{cMNIST}. Both levels share parents $D$ (Digit) and $C$ (Color). The low-level generates pixels $I_P$, while the high-level generates latent codes $z$.}
\label{fig:cmnist_graph}
\end{figure}

\paragraph{SCM Construction and Abduction}
The original generative process is non-additive ($z = G(D, C, \epsilon)$), preventing direct abduction in our additive noise framework. We explicitly construct and fit two valid ANMs to model the data:
\begin{itemize}

\item \emph{Low-Level Model ($f_L$): U-Net with FiLM.}
We model the pixel generation as $I_P = f_L(\text{Shape}, D, C) + U^\ell$. We employ a \emph{U-Net} architecture that takes a grayscale ``Shape'' image as input to define structural content. To inject causal control, we use \emph{FiLM (Feature-wise Linear Modulation)} layers. These layers map the parents $(D, C)$ to affine parameters $(\bm{\gamma}, \bm{\beta})$, which modulate the internal feature maps of the U-Net via $\mathbf{h}' = \bm{\gamma} \cdot \mathbf{h} + \bm{\beta}$. This architecture analytically separates structural content (processed by convolutions) from style (injected by FiLM), allowing for noise abduction.

  \item \emph{High-Level Model ($f_H$): Vector-Valued Cell-Means.}
We model the latent mechanism as $z = f_H(D, C) + U^h$. Given the discrete $10 \times 10$ parent space, we fit a \emph{vector-valued cell-means model}. To ensure robustness against data scarcity in specific $(d,c)$ buckets (caused by confounding), we apply shrinkage regularization. The deterministic vector $m_{d,c}$ is a weighted average of the empirical cell mean and a stable additive baseline ($\hat{\mu}_{\text{base}} = \hat{\mu}_d + \hat{\mu}_c - \hat{\mu}_{\text{global}}$).
\end{itemize}

\paragraph{Abduction Procedure.}
Following our framework, we (1) fit $f_L$ and $f_H$ on observational data; (2) abduce anchor noise matrices $\mathbf{U}^\ell$ and $\mathbf{U}^h$ as residuals; and (3) synthesize the deterministic components $\mathbf{D}_\ell^{(\iota)}$ and $\mathbf{D}_h^{(\iota)}$ by evaluating the fitted functions under the interventional settings defined in $\mathcal{I}$.

A visual inspection of the dataset (Fig.~\ref{fig:cmnist_samples}) reveals the core causal challenge: strong spurious correlations in the observational distribution. For instance, specific digits are consistently paired with specific colors (e.g., digit 3 is predominantly green). Interventions are critical for breaking these correlations, generating counterfactual combinations (e.g., digit 3 as red) that expose the true causal mechanism.
\begin{figure}[h!]
    \centering
    \includegraphics[width=.8\linewidth]{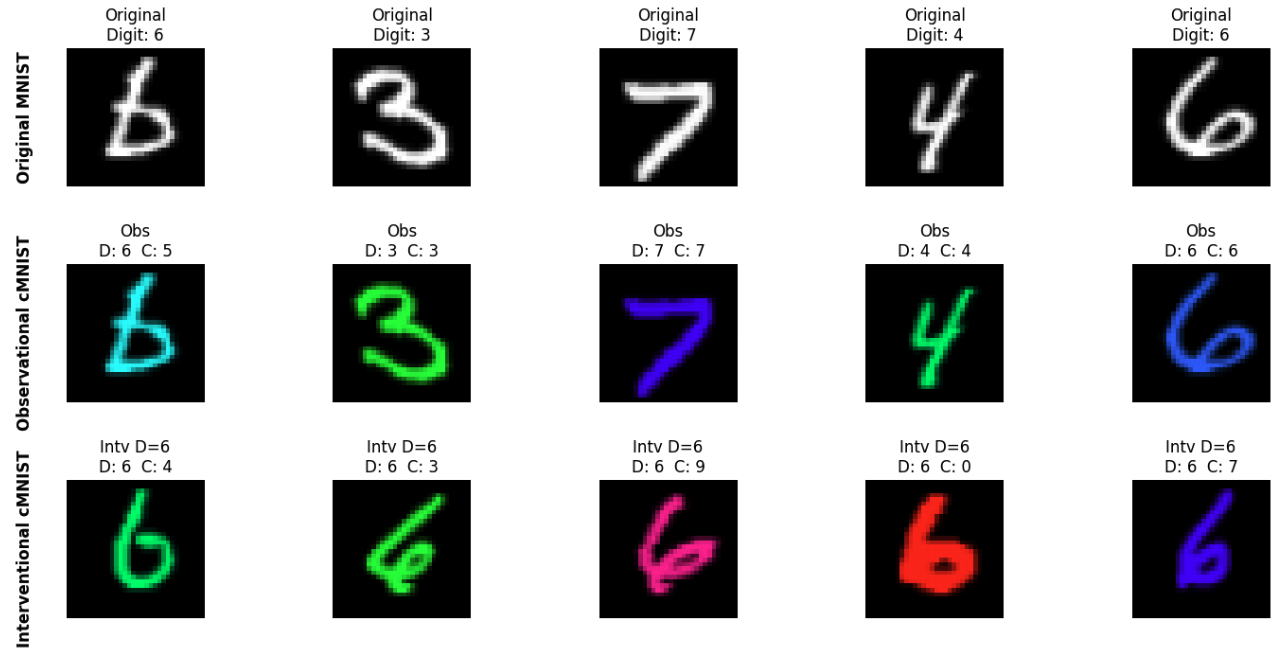}
    \caption{\texttt{cMNIST} data samples. \textbf{Top:} Original grayscale MNIST. \textbf{Middle:} Observational data, showing the spurious correlation (e.g., $D=3, C=3$ and $D=7, C=7$). \textbf{Bottom:} Interventional data for $do(D=6)$, where the correlation is broken, and the digit 6 is paired with novel colors.}
    \label{fig:cmnist_samples}
\end{figure}

\paragraph{Latent Space Structure.}
We analyze the geometry of the high-level latent space $z \in \mathbb{R}^{64}$ by projecting it to two dimensions using PCA (Figure \ref{fig:cmnist_pca}). The visualizations show that the latent space $z$ (learned by Xia et al.'s encoder $E_\phi$) is highly structured. The plots show that both the 10 digits and the 10 colors form distinct clusters. This suggests that the latent space is \emph{disentangled}, with different sets of dimensions corresponding to ``digit" and ``color" information. 
\begin{figure}[ht]
    \centering
    \includegraphics[width=0.49\linewidth]{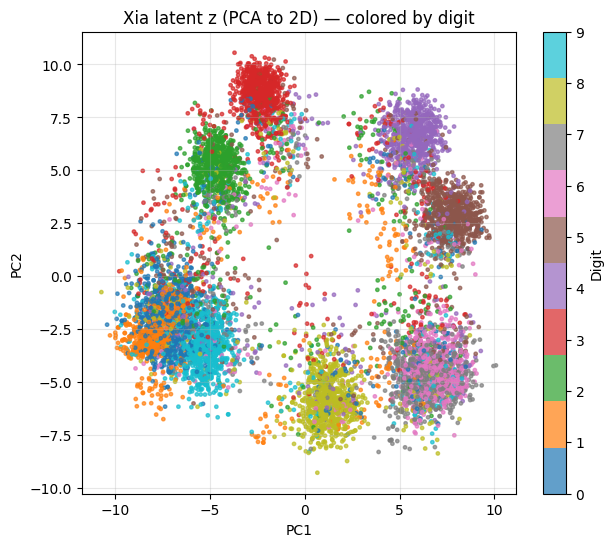}
    \includegraphics[width=0.49\linewidth]{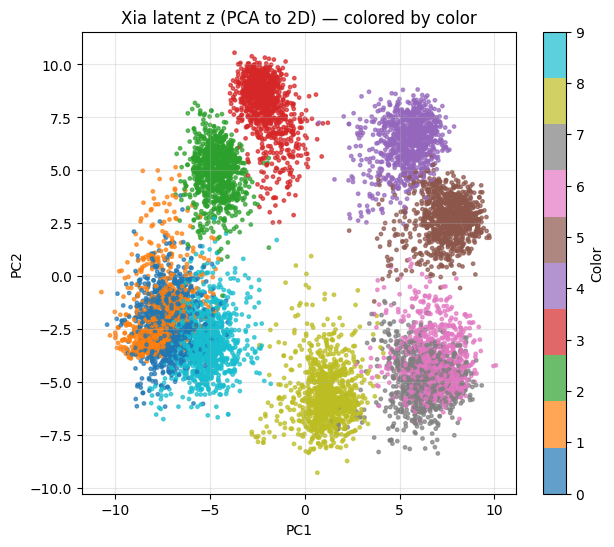}
    \caption{PCA projection of the latent space $z \in \mathbb{R}^{64}$ from the observational dataset. \textbf{Left:} Colored by digit $D$. \textbf{Right:} Colored by color $C$.}
    \label{fig:cmnist_pca}
\end{figure}

\paragraph{Interventional Geometry.}
Finally, we analyze the effect of interventions on the latent distribution $P(z)$, as shown in Figure \ref{fig:cmnist_hists}. The distribution of the $l_2$ norm of $z$, $\|z\|_2$, (left plot) shifts significantly under different interventions. The observational distribution (blue) is distinct from the interventional ones (orange, green, purple). This confirms that interventions induce a distributional shift in the high-level space. The right plot quantifies this shift, showing the $l_2$ distance $\|z_{obs} - z_{int}\|_2$ between the latent vector of an observational sample and its interventional counterpart. The distribution is centered far from zero (around 15-20), confirming that interventional samples lie in a different region of the latent space than their observational counterparts. 
\begin{figure}[ht]
    \centering
    \includegraphics[width=0.5\linewidth]{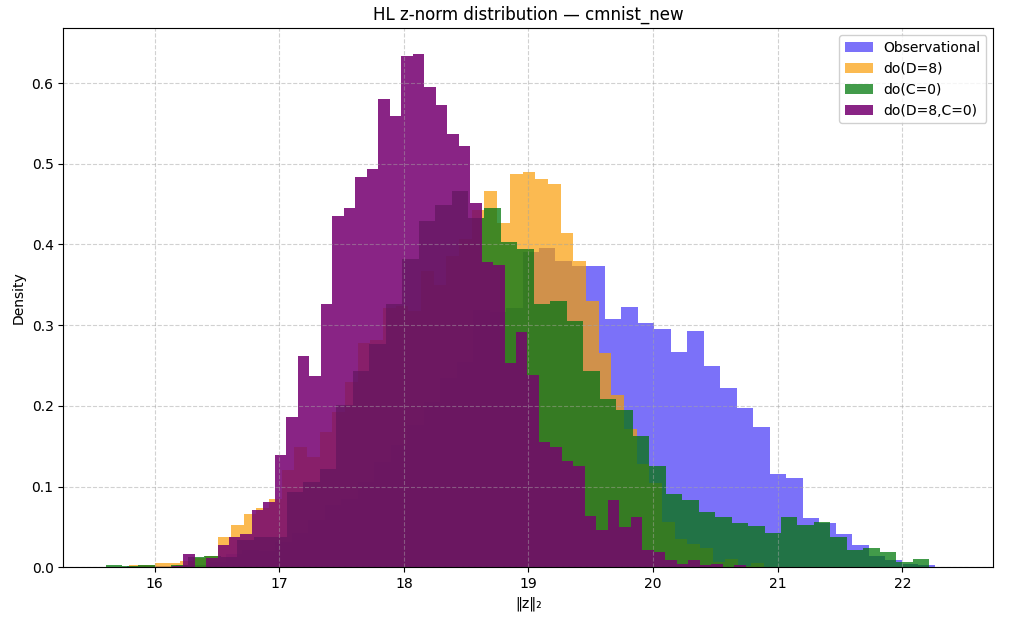}
    \includegraphics[width=0.4\linewidth]{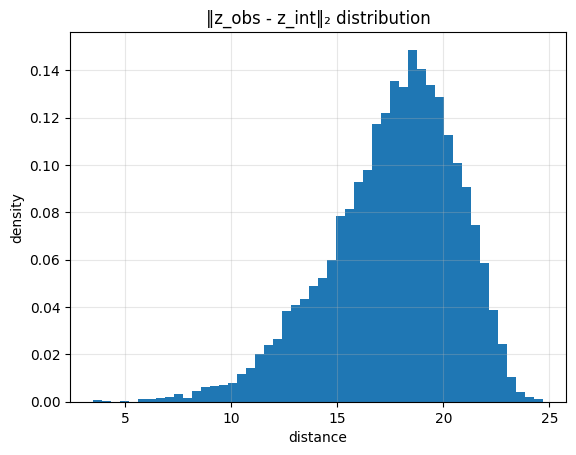}
    \caption{\textbf{Left:} Distribution of the $l_2$ norm of $z$ under observational and various interventional distributions. \textbf{Right:} Distribution of the $l_2$ distance between observational samples $z_{obs}$ and their interventional counterfactuals $z_{int}$.}
    \label{fig:cmnist_hists}
\end{figure}

\subsection{Additional experimental results and settings}\label{sec:all_results_settings}
We evaluate our framework across a diverse set of experimental settings of different representations of the environment (Gaussian vs. empirical) and datasets. Fig.~\ref{fig:exp_roadmap} provides an overview of the datasets, methods, and evaluation metrics used in each setting.
\begin{figure}[htp]
\begin{center}
\begin{tikzpicture}[
    node distance=0.6cm and 0.4cm,
    box/.style={draw, rounded corners, align=center, font=\footnotesize, minimum height=1.2em, text width=2.5cm, fill=white},
    env/.style={
        draw,
        thick,
        rounded corners,
        fill=gray!15,
        font=\small\bfseries,
        text width=3.0cm,
        align=center,
        inner sep=3pt
    },
    cat/.style={
        draw=none,
        font=\small\bfseries,
        text width=2.5cm,
        align=center
    },
    subcat/.style={
        draw=none,
        font=\scriptsize\itshape,
        text width=2.0cm,
        align=center,
        color=darkgray
    },
    arrow/.style={thick, -{Stealth[length=2mm]}}
]

\node[font=\large\bfseries] (root) {$\exprob = \exprobbase \otimes \exprobabst$} ;

\node[env, below=1.0cm of root, xshift=-4.5cm] (gauss) {Gaussian};
\node[env, below=1.0cm of root, xshift=3.5cm] (emp) {Empirical};

\draw[arrow] (root) -| (gauss);
\draw[arrow] (root) -| (emp);

\node[cat, below=0.8cm of gauss] (gauss_lin_label) {Linear};
\node[box, below=0.1cm of gauss_lin_label] (gauss_datasets) {\texttt{SLC}, \texttt{LiLUCAS}};

\draw[-] (gauss) -- (gauss_lin_label);
\draw[-] (gauss_lin_label) -- (gauss_datasets);


\coordinate[below=0.8cm of emp] (emp_split);
\draw[-] (emp) -- (emp_split);

\node[cat, below=0.8cm of emp, xshift=-3.5cm] (emp_lin_label) {Linear};
\node[box, below=0.1cm of emp_lin_label] (emp_lin_datasets) {\texttt{SLC}, \texttt{LiLUCAS}};

\draw[-] (emp_split) -| (emp_lin_label);
\draw[-] (emp_lin_label) -- (emp_lin_datasets);

\node[cat, below=0.8cm of emp, xshift=2.0cm] (emp_nonlin_label) {Non-Linear};
\draw[-] (emp_split) -| (emp_nonlin_label);

\node[subcat, below=0.4cm of emp_nonlin_label, xshift=-1.4cm] (known_label) {Known $f$};
\node[box, below=0.0cm of known_label] (known_ds) {\texttt{nLUCAS}};

\node[subcat, below=0.4cm of emp_nonlin_label, xshift=1.4cm] (est_label) {Estimated $f$};
\node[box, below=0.0cm of est_label] (est_ds) {\texttt{EBM}, \texttt{cMNIST}};

\coordinate[below=0.1cm of emp_nonlin_label] (nl_split);
\draw[-] (emp_nonlin_label) -- (nl_split);
\draw[-] (nl_split) -| (known_label);
\draw[-] (nl_split) -| (est_label);
\draw[-] (known_label) -- (known_ds);
\draw[-] (est_label) -- (est_ds);

\end{tikzpicture}
\caption{\textbf{Roadmap of Experimental Datasets.} Evaluation is divided into Gaussian (parametric) and Empirical (data-driven) environments. Under the Empirical setting, we distinguish between \textbf{Linear} and \textbf{Non-Linear} benchmarks. We further categorize Non-Linear tasks by structural knowledge: \textbf{Known} functions (\texttt{nLUCAS}) versus \textbf{Estimated} functions (\texttt{EBM}, \texttt{cMNIST}).}
\label{fig:exp_roadmap}
\end{center}
\end{figure}

\subsubsection{Misspecifications and settings} 
\paragraph{Robustness to Distributional Shifts.}
We evaluate robustness using a unified and flexible framework based on a Huber-style contamination model, applied to the held-out test data from each of the $k$ cross-validation folds. The methodology consists of data contamination, error computation, and a multi-level aggregation protocol.
\begin{enumerate}
    \item{Data Contamination.}
    For a given clean interventional test set, $X_{\text{test}, \iota} \in \mathbb{R}^{n_{\text{test}} \times d}$, we first generate a noisy outlier set, $\tilde{X}_{\iota}$, by applying an additive stochastic shift. This is achieved by generating a noise matrix $\mathbf{N} \in \mathbb{R}^{n_{\text{test}} \times d}$, where each row is an independent sample from a specified distribution. While the main text focuses on Gaussian noise ($\mathbf{n}_i \sim \mathcal{N}(0, \sigma_{\text{noise}}^2 \cdot \mathbf{I})$), our framework also supports other distributions, including heavy-tailed Student-t and skewed Exponential. The final contaminated test set, $\bar{X}_{\iota}$, is then formed as a convex combination of the clean and noisy data, governed by the contamination fraction $\alpha \in [0, 1]$:
    \begin{align}
        \bar{X}_{\iota} = (1-\alpha)X_{\text{test}, \iota} + \alpha \tilde{X}_{\iota}.
    \end{align}
    This process is controlled by two parameters: the fraction $\alpha$ (prevalence) and the noise scale $\sigma_{\text{noise}}$ (magnitude).
    \item{Error Computation.}
    For a single contaminated test set and a given abstraction map $\lintau$, the abstraction error is computed as the average distance between the transformed low-level and high-level samples, averaged across all interventions $\iota \in \mathcal{I}_\ell$:
    \begin{align}
        \mathcal{E}(\tau, \alpha, \sigma_{\text{noise}}) = \mathbb{E}_{\iota \sim q} \left[ \mathcal{D}\left( \lintau \bar{X}_{\iota}^{\ell}, \bar{X}_{\omega(\iota)}^{h} \right) \right].
    \end{align}
    \item{Experimental Protocol.}
    To ensure statistically robust results, we perform a multi-level evaluation. The full experiment is a nested procedure that iterates through a grid of contamination parameters ($\alpha$ and $\tilde{\sigma}$). For each point on this grid, our protocol yields a total of $m \times k$ individual error measurements (one for each of the $m$ random samples on each of the $k$ data splits). The final "Robustness Curves" visualize the mean error, averaged over both samples and folds, as a function of the contamination parameters.
\end{enumerate}

\paragraph{Note.} For the \texttt{nLUCAS} dataset, we conduct an additional demonstration to validate the impact of the shared residuals optimization strategy discussed in Appendix~\ref{sec:anm_decomposition}. While our training minimizes discrepancies using a fixed noise realization $\mathbf{U}$ across interventions, we verify that the learned abstraction is not an artifact of this device. We compare performance under two distinct protocols: \textbf{a) Generalization}, which evaluates the abstraction on actual samples with their own (varying) noise samples (Different $\mathbf{U}$), and \textbf{b) Consistency}, which does the same thing as training, deriving samples from the shared residuals (Shared $\mathbf{U}$). In the results below, we explicitly observe the consequence of this distinction. Under the consistency protocol, the non-robust baseline \textsc{Grad}$_{\tauomega}$ outperforms \textsc{DiRoCA} in the clean setting ($\alpha=0$), as it aggressively fits the specific noise geometry used during optimization. However, under the generalization one, this advantage disappears and \textsc{DiRoCA} outperforms \textsc{Grad}$_{\tauomega}$ even at $\alpha=0$. This confirms that the baseline overfits the nominal environment's specific noise realization, whereas \textsc{DiRoCA}'s adversarial training acts as a regularizer—forcing the model to learn the true, noise-invariant mechanism and preventing the generalization failure described in the main text. Unless otherwise stated, all other results in the main text and appendices report performance under the generalization protocol, as it represents the realistic evaluation setting.

\paragraph{Robustness to Model Misspecification.}
Beyond robustness to environmental shifts, we evaluate our method's resilience to violations of core modeling assumptions. We challenge two critical assumptions: (a) the linearity of the structural functions ($\cF$-misspecification), and (b) the correctness of the intervention mapping ($\omega$-misspecification).

\begin{itemize}
    \item \textbf{$\mathcal{F}$-misspecification.} To test robustness against violations of the linearity assumption, we evaluate our linearly-trained models on test data generated from fundamentally non-linear SCMs. Crucially, these new SCMs share the same causal graph, intervention sets, and underlying environments as those used in training. However, we define a new structural non-linear equations $f_{\text{NL}}$, controlled by a strength parameter $k \in \bR$:
    \begin{align}
        X_j = k \cdot f_{\text{NL}}(\text{Pa}(X_j)) + U_j.
    \end{align}
    For each strength level $k$, we simulate a brand new, ground-truth non-linear test set from this SCM and compute the abstraction error. The main text reports the results for $k=1$ using a sinusoidal function. Here we also provide full robustness curves showing the error as a function of $k \in [0,100]$ for both \texttt{sin} and \texttt{tanh} non-linearities in Figures~\ref{fig:error_vs_k_sin} and \ref{fig:error_vs_k_tanh} respectively.

    \item \textbf{$\omega$-misspecification.} To test robustness to an incorrect intervention mapping, we contaminate the ground-truth $\omega$ map using a semantic corruption strategy. For a given number of misalignments, we randomly select a subset of low-level interventions and re-assign each one to an incorrect high-level counterpart of similar complexity (i.e., one that intervenes on a similar number of variables), with a fallback to the nearest complexity neighbor if an exact match is not available within a specified delta. An illustration of this semantic $\omega$-misspecification process is provided in Figure~\ref{fig:omega_contamination}.
\end{itemize}

\begin{figure}
    \centering
    \begin{tikzpicture}[>=stealth,scale=1.2]
\node (l0) at (0,0) {$\emptyset$};
\node (l1) at (0,1) {$\iota_1$};
\node (l2) at (0,2) {$\iota_2$};
\node (l3) at (0,3) {$\iota_3$};

\node (r0) at (2.5,0) {$\emptyset$};
\node (r1) at (2.5,1.5) {$\eta_1$};
\node (r2) at (2.5,3) {$\eta_2$};

\draw[dotted,->] (l0) -- (r0) node[midway,below] {\scriptsize $\omega(\emptyset) = \emptyset$};
\draw[dotted,->] (l1) -- (r1) node[midway,below] {\scriptsize $\omega(\iota_1) = \eta_1$};
\draw[dotted,->] (l2) -- (r1) node[midway,above] {\scriptsize $\omega(\iota_2) = \eta_1$};
\draw[dotted,->] (l3) -- (r2) node[midway,above] {\scriptsize $\omega(\iota_3) = \eta_2$};

\node at (1.25,3.8) {\textbf{Ground-Truth $\omega$}};

\node (l0c) at (5,0) {$\emptyset$};
\node (l1c) at (5,1) {$\iota_1$};
\node (l2c) at (5,2) {$\iota_2$};
\node (l3c) at (5,3) {$\iota_3$};

\node (r0c) at (7.5,0) {$\emptyset$};
\node (r1c) at (7.5,1.5) {$\eta_1$};
\node (r2c) at (7.5,3) {$\eta_2$};

\draw[dotted,->] (l0c) -- (r0c) node[midway,below] {\scriptsize $\omega(\emptyset) = \emptyset$};
\draw[dotted,red,->] (l1c) -- (r0c) node[midway,above] {\scriptsize $\omega(\iota_1) = \emptyset$}; 
\draw[dotted,->] (l2c) -- (r1c) node[midway,below] {\scriptsize $\omega(\iota_2) = \eta_1$};
\draw[dotted,red,->] (l3c) -- (r1c) node[midway,above] {\scriptsize $\omega(\iota_3) = \eta_1$}; 

\node at (6.25,3.8) {\textbf{Contaminated $\omega$}};

\end{tikzpicture}
    \caption{Illustration of intervention map ($\omega$) contamination. Left: the ground-truth mapping between low-level interventions ($\iota$) and high-level interventions ($\eta$). Right: a contaminated version where interventions have been re-assigned to incorrect counterparts of similar complexity (e.g., $\omega(\iota_1)=\eta_2$ instead of $\eta_1$) shown in red.}
    \label{fig:omega_contamination}
\end{figure}

\paragraph{Robustness to Camera Shifts.}
In addition to Huber-style pixel corruptions for the \texttt{cMNIST} dataset, we evaluate the robustness of each learned abstraction under geometric distribution shifts. Unlike pixel noise, which degrades the signal-to-noise ratio, these transformations test the stability of the learned linear operator $\lintau$ under realistic camera shifts. Given a clean image $x \in [-1,1]^{3\times 32 \times 32}$ we consider two distinct shifts:
\begin{itemize}
    \item \emph{Geometric}, which rotates the data manifold in the input space. We generate a rotated version $\tilde{x} = \mathcal{R}_{\theta}(x)$ at distinct levels of severity $\theta \in \{30^{\circ}, 90^{\circ}\}$.
    \item \emph{Photometric}, which scales pixel intensities to emulate changing environmental lighting. We generate shifted versions $\tilde{x} = \mathcal{B}_{\lambda}(x)$ by adjusting brightness by a factor $\lambda \in \{0.6, 3.0\}$, representing dim and overexposed conditions, respectively.
\end{itemize}
A purely linear abstraction trained on canonical orientations may suffer from catastrophic intercept mismatch when the input distribution rotates. To isolate the quality of the learned geometry (the matrix $\lintau$) from simple mean shifts for a batch of shifted inputs $X_{\text{pix}}$ and corresponding ground-truth latents $Z$, we formulate the prediction as:
\begin{equation}
\hat{Z} = X_{\text{pix}}\lintau^\top + b,
\end{equation}
where $b = \mathbb{E}[Z - X_{\text{pix}}T^\top]$ is a test-time mean-shift correction vector. This correction effectively re-centers the predicted point cloud to match the first moment of the ground truth, ensuring that the reported \emph{Relative Squared Error} measures strictly the alignment of the learned subspace with the causal variables, independent of the translational offset induced by the rotation.
\begin{figure}
    \centering
    \includegraphics[width=.55\linewidth]{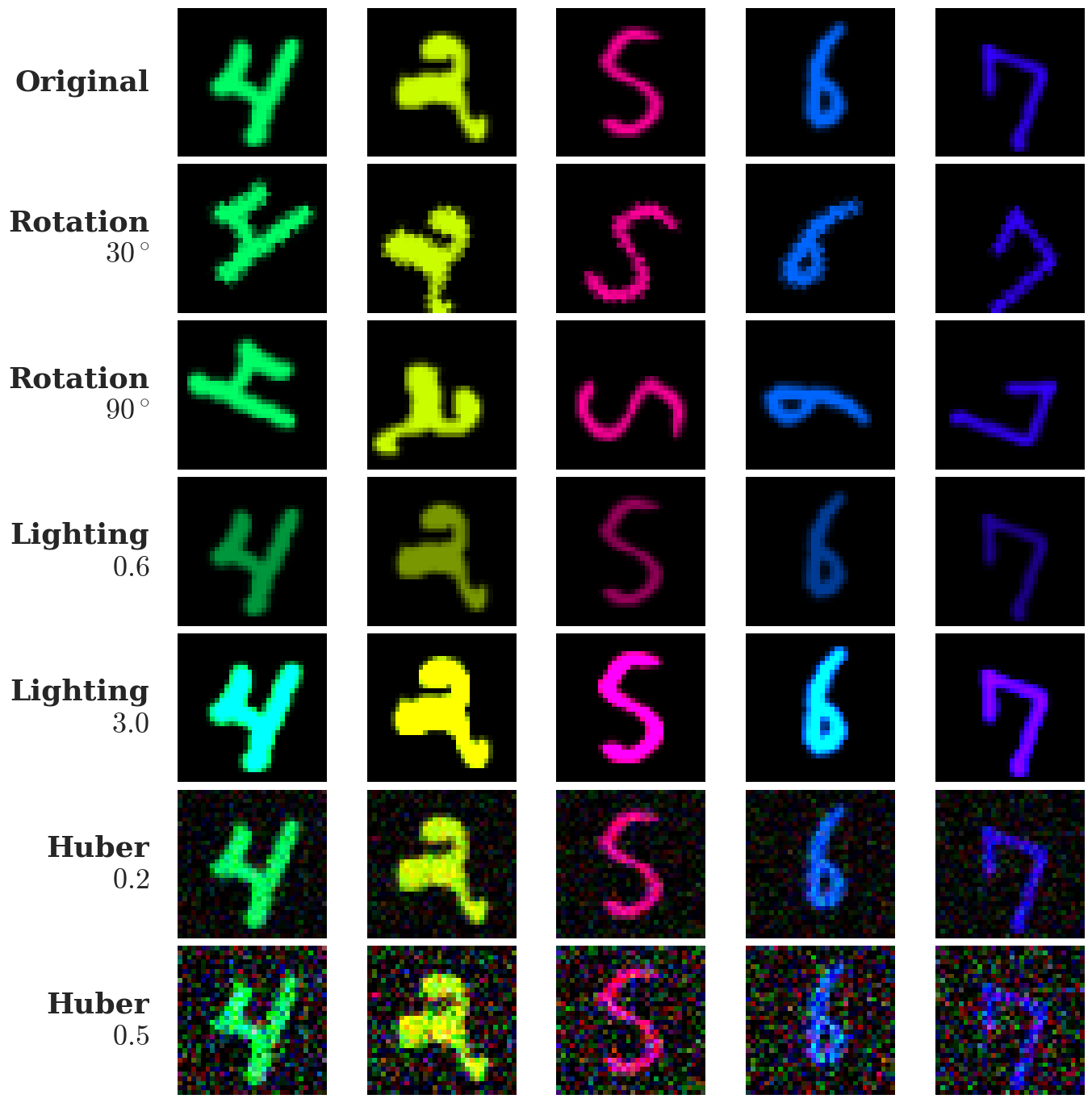}
    \caption{Visualizing distributional shifts on \texttt{cMNIST}. Top row: Original samples from the test distribution. Subsequent rows display the specific perturbations evaluated in our robustness analysis: Geometric Shifts (rotations of $30^\circ$ and $90^\circ$), Photometric Shifts (lighting adjustments of $\lambda=0.6$ and $\lambda=3.0$), and Huber Contamination (additive noise with $\tilde{\sigma}=0.2$ and $\tilde{\sigma}=0.5$). These transformations test the stability of learned abstractions against structured geometric changes, environmental lighting variations, and unstructured pixel noise, respectively.}
    \label{fig:placeholder}
\end{figure}

\paragraph{Implementation and Optimization Details.} All models were trained on a MacBook Air (13-inch, M1, 2020) equipped with an Apple M1 chip and 16\,GB of unified memory, running macOS Monterey version 12.3. Performance is evaluated using a $k$-fold cross-validation scheme (with $k=5$) on the 10,000 generated samples for each intervention in \texttt{LiLUCAS} and \texttt{SLC} and 5,000 samples for \texttt{nLUCAS}. In \textsc{DiRoCA} each iteration comprises 5 minimization and 2 maximization steps. Finally, we monitored convergence by tracking the absolute change in the objective value across successive epochs, and the optimization is terminated once this change falls below a fixed threshold of $10^{-4}$, ensuring numerical stability while avoiding unnecessary computation. Regarding the operationalization of the theoretical radius (Theorems 1, 2), we treated the concentration constants as hyperparameters fixed for each dataset type. We consistently set the tail exponent $\alpha_d=2.0$ (assuming sub-Gaussian tails) and confidence level $\eta=0.05$. The scaling constants were set to $c_{d,1}=1.0, c_{d,2}=1.0$ for the linear synthetic datasets (\texttt{SLC}, \texttt{LiLUCAS}), while for the more complex non-linear and high-dimensional tasks (\texttt{nLUCAS}, \texttt{EBM}, \texttt{cMNIST}), we increased the pre-factor $c_{d,1}$ (ranging from $10$ to $1000$) to account for looser theoretical bounds in higher dimensions, while keeping $c_{d,2}=1.0$.

\subsubsection{Results}
Below, we present the full experimental results as outlined in the experiments roadmap (Fig.~\ref{fig:exp_roadmap}). All bolded values in the tables indicate statistically significant differences, determined using paired t-tests with a significance threshold of $p < 0.05$.
\begin{table}[htp]
\caption{Abstraction error and std under $\alpha=0$ and $\alpha=1$ for \texttt{SLC} and \texttt{LiLUCAS}. Top: Gaussian setting with $(\hat{\radbase}, \hat{\radabst})=(0.11, 0.11)$. Bottom: Empirical setting with $(\hat{\radbase}, \hat{\radabst})=(0.10, 0.03)$}
\label{tab:all_0_rho_shift}
\centering
\begin{tabular}{lcccc}
\toprule
\multirow{2}{*}{Method} & $\alpha=0$ & $\alpha=1$ & $\alpha=0$ & $\alpha=1$ \\
\cmidrule(lr){2-3} \cmidrule(lr){4-5}
 & \multicolumn{2}{c}{\texttt{SLC}} & \multicolumn{2}{c}{\texttt{LiLUCAS}} \\
\midrule
\multicolumn{5}{c}{Gaussian} \\
\midrule
\textsc{Bary}$_{\tauomega}$                     & 2.63 $\pm$ 0.01 & 8.43 $\pm$ 0.01 & 5.75 $\pm$ 0.01 & 55.11 $\pm$ 0.02 \\
\textsc{Grad}$_{\tauomega}$                       & \textbf{1.25 $\pm$ 0.01} & 6.74 $\pm$ 0.00 & \textbf{5.45 $\pm$ 0.02} & 64.73 $\pm$ 0.03 \\
\textsc{DiRoCA}$_{\widehat{\epsilon}}$ & 2.67 $\pm$ 0.02 & 11.28 $\pm$ 0.01 & 5.86 $\pm$ 0.01 & 61.86 $\pm$ 0.02 \\
\textsc{DiRoCA}$_{1,1}$                            & 2.01 $\pm$ 0.01 & 4.14 $\pm$ 0.01 & 6.97 $\pm$ 0.03 & 48.05 $\pm$ 0.02 \\
\textsc{DiRoCA}$_{2,2}$                            & 1.74 $\pm$ 0.01 & \textbf{2.37 $\pm$ 0.01} & 5.53 $\pm$ 0.03 & 25.79 $\pm$ 0.01 \\
\textsc{DiRoCA}$_{4,4}$                            & 1.90 $\pm$ 0.01 & 2.59 $\pm$ 0.01 & 6.81 $\pm$ 0.03 & \textbf{22.99 $\pm$ 0.01} \\
\textsc{DiRoCA}$_{8,8}$                            & 1.90 $\pm$ 0.01 & 2.59 $\pm$ 0.01 & 6.81 $\pm$ 0.03 & \textbf{23.00 $\pm$ 0.01} \\
\midrule
\multicolumn{5}{c}{Empirical} \\
\midrule
\textsc{Bary}$_{\tauomega}$                     & 86.46 $\pm$ 0.15 & 650.33 $\pm$ 0.69 & 561.19 $\pm$ 3.49 & 3804.18 $\pm$ 4.39 \\
\textsc{Grad}$_{\tauomega}$                       & \textbf{57.04 $\pm$ 0.28} & 392.95 $\pm$ 0.30 & \textbf{309.82 $\pm$ 1.71} & 1290.65 $\pm$ 1.52 \\
$\absp$                                            & 89.29 $\pm$ 0.24 & 678.89 $\pm$ 0.71 & 399.21 $\pm$ 1.63 & 2189.48 $\pm$ 2.15 \\
$\absn$                                            & 91.67 $\pm$ 0.29 & 640.45 $\pm$ 0.68 & 354.31 $\pm$ 1.46 & 1878.96 $\pm$ 1.55 \\
\textsc{DiRoCA}$_{\widehat{\epsilon}}$ & 88.39 $\pm$ 0.14 & 674.40 $\pm$ 0.74 & 470.13 $\pm$ 3.65 & 2928.40 $\pm$ 3.41 \\
\textsc{DiRoCA}$_{1,1}$                            & 59.08 $\pm$ 0.33 & 388.73 $\pm$ 0.31 & 332.96 $\pm$ 2.98 & 1385.36 $\pm$ 1.58 \\
\textsc{DiRoCA}$_{2,2}$                            & 58.86 $\pm$ 1.26 & \textbf{357.93 $\pm$ 1.72} & 315.21 $\pm$ 7.93 & 1102.80 $\pm$ 18.15 \\
\textsc{DiRoCA}$_{4,4}$                            & 58.87 $\pm$ 1.25 & \textbf{357.93 $\pm$ 1.72} & 311.19 $\pm$ 8.01 & \textbf{981.61 $\pm$ 18.29} \\
\textsc{DiRoCA}$_{8,8}$                            & 58.89 $\pm$ 1.21 & 358.20 $\pm$ 1.67 & 311.43 $\pm$ 8.18 & 987.53 $\pm$ 17.81 \\
\bottomrule
\end{tabular}
\end{table}
\begin{figure}
    \centering
    \includegraphics[width=.85\linewidth]{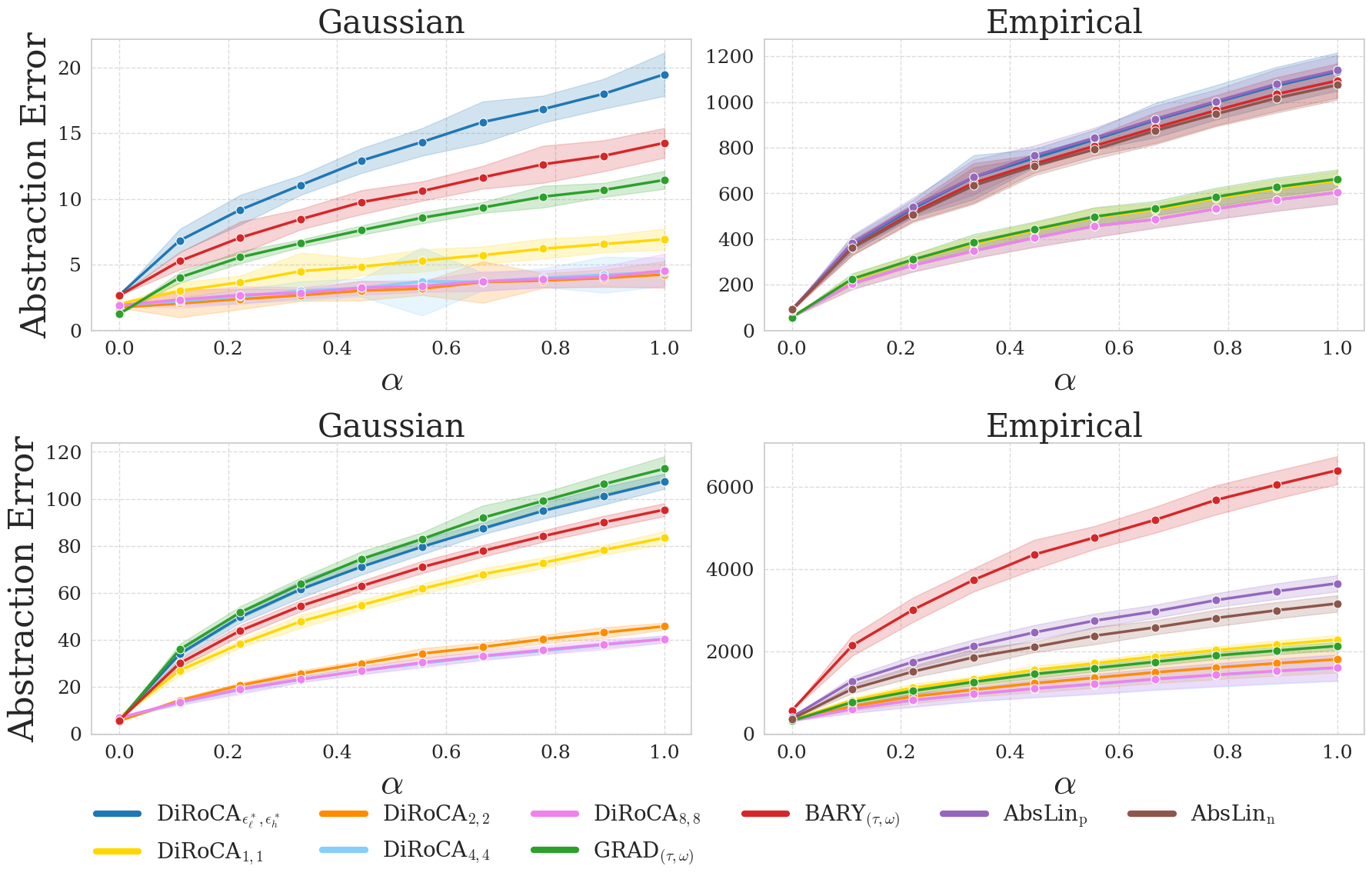}
    \caption{Robustness to outlier fraction ($\alpha$) on the \texttt{SLC} (top) and \texttt{LiLUCAS} (bottom) experiments for the Gaussian (left) and Empirical (right) settings. The evaluation is performed at a fixed \textit{Gaussian} noise intensity ($\tilde{\sigma}=5.0$ for \texttt{SLC} and $\tilde{\sigma}=10.0$ for \texttt{LiLUCAS})}
\label{fig:error_vs_alpha_complete}
\end{figure}
\begin{figure}
    \centering
    \includegraphics[width=.85\linewidth]{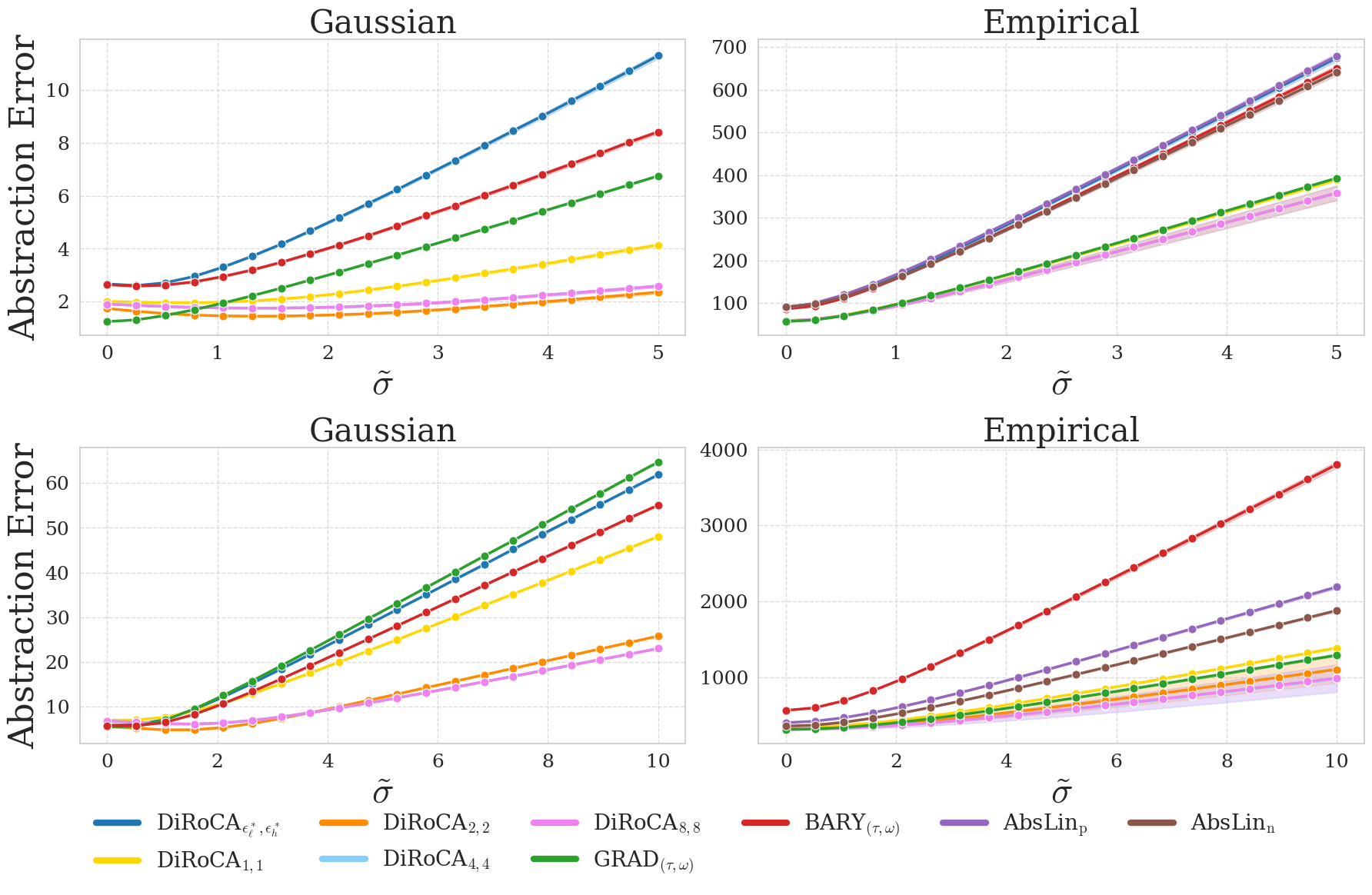}
    \caption{Robustness to \textit{Gaussian} noise intensity ($\tilde{\sigma}$) on the \texttt{SLC} (top) and \texttt{LiLUCAS} (bottom) experiments for the Gaussian (left) and Empirical (right) settings. The evaluation is performed at a fixed outlier fraction of $\alpha=1.0$ (fully noisy data).}
\label{fig:error_vs_sigma_complete}
\end{figure}

\begin{table}[htp]
\caption{Average abstraction error (mean $\pm \sigma$) under \emph{model} $\cF$-misspecification with \texttt{sin} nonlinearity for the Gaussian setting with $(\hat{\radbase},\hat{\radabst})=(0.11,0.11)$ and the empirical setting with $(\hat{\radbase},\hat{\radabst})=(0.10,0.03)$.}
\label{tab:unified_f_misspec_sin}
\centering
\begin{tabular}{lcccc}
\toprule
\multirow{2}{*}{Method} & \multicolumn{2}{c}{\texttt{SLC}} & \multicolumn{2}{c}{\texttt{LiLUCAS}} \\
\cmidrule(lr){2-3} \cmidrule(lr){4-5}
 & Gaussian & Empirical & Gaussian & Empirical \\
\midrule
\textsc{Bary}$_{\tauomega}$                         & 1.10 $\pm$ 0.02 & 185.84 $\pm$ 1.50 & 4.25 $\pm$ 0.01 & 578.80 $\pm$ 1.25 \\
\textsc{Grad}$_{\tauomega}$                         & 2.96 $\pm$ 0.01 & 138.33 $\pm$ 0.96 & 5.86 $\pm$ 0.03 & 338.74 $\pm$ 1.01 \\
$\absp$                                             & n/a             & 198.39 $\pm$ 1.70 & n/a             & 430.34 $\pm$ 1.24 \\
$\absn$                                             & n/a             & 190.58 $\pm$ 1.67 & n/a             & 381.32 $\pm$ 1.92 \\
\textsc{DiRoCA}$_{\widehat{\epsilon}}$  & \textbf{0.40 $\pm$ 0.02} & 191.68 $\pm$ 1.54 & 6.01 $\pm$ 0.03 & 500.10 $\pm$ 1.44 \\
\textsc{DiRoCA}$_{1,1}$                             & 1.79 $\pm$ 0.02 & 137.37 $\pm$ 0.92 & 5.98 $\pm$ 0.03 & 353.61 $\pm$ 1.83 \\
\textsc{DiRoCA}$_{2,2}$                             & 1.43 $\pm$ 0.01 & 135.34 $\pm$ 0.51 & \textbf{3.99 $\pm$ 0.03} & 333.02 $\pm$ 9.86 \\
\textsc{DiRoCA}$_{4,4}$                             & 2.19 $\pm$ 0.02 & 135.34 $\pm$ 0.50 & 4.63 $\pm$ 0.03 & \textbf{327.42 $\pm$ 9.15} \\
\textsc{DiRoCA}$_{8,8}$                             & 2.19 $\pm$ 0.02 & \textbf{135.32 $\pm$ 0.52} & 4.63 $\pm$ 0.03 & 327.64 $\pm$ 9.00 \\
\bottomrule
\end{tabular}
\end{table}

\begin{figure}
    \centering
    \includegraphics[width=.85\linewidth]{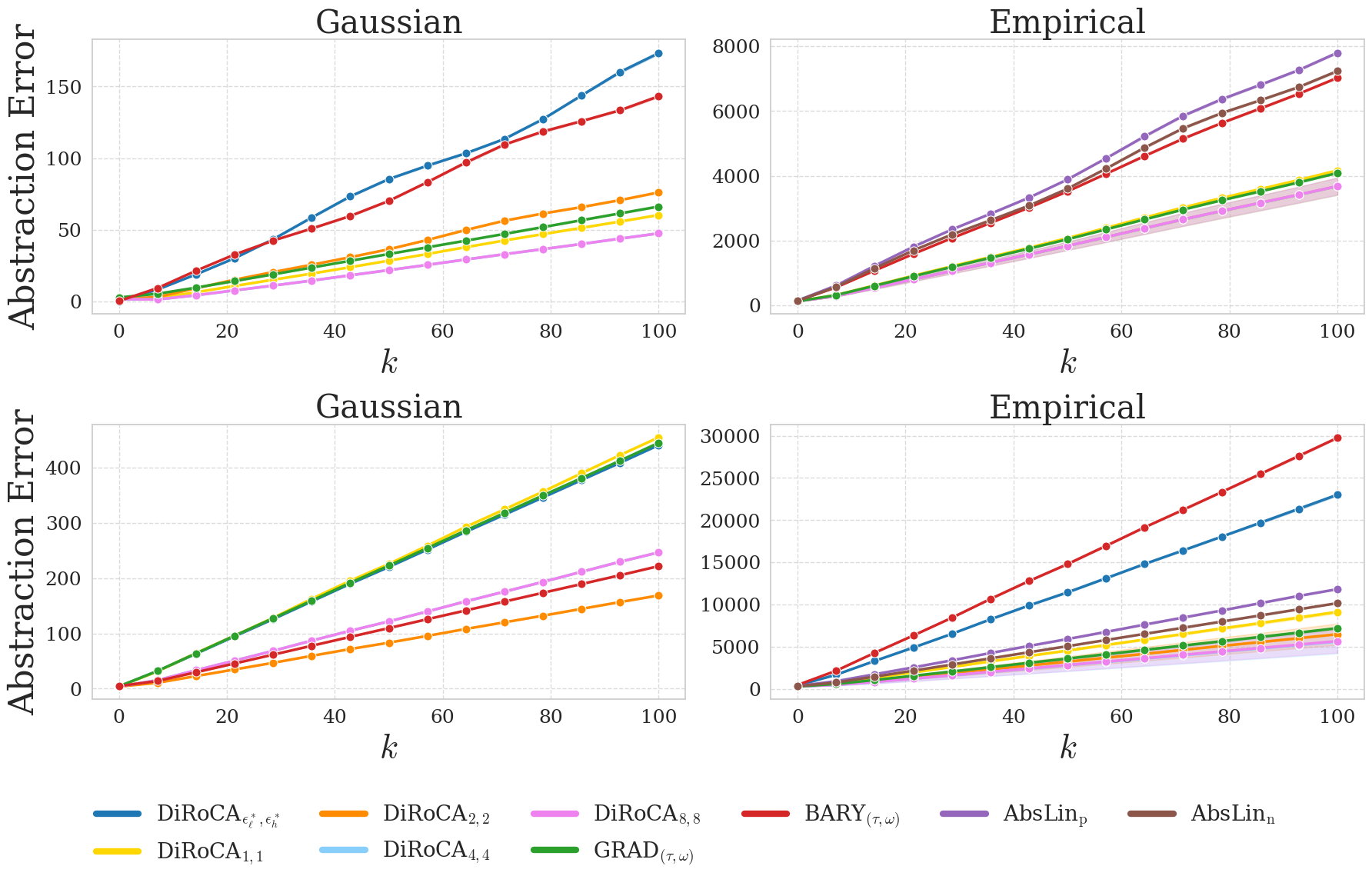}
    \caption{Abstraction error comparison as a function of the nonlinearity strength parameter $k$, which controls the strength of the nonlinear component (\texttt{sin}) in the data generation process. The figure shows results for both Gaussian (left) and empirical (right) data distributions in the \texttt{SLC} (top) and \texttt{LiLUCAS} (bottom) settings.
}
\label{fig:error_vs_k_sin}
\end{figure}

\begin{table}[htp]
\caption{Average abstraction error (mean $\pm \sigma$) under \emph{model} $\cF$-misspecification with \texttt{tanh} nonlinearity for the Gaussian setting with $(\hat{\radbase},\hat{\radabst})=(0.11,0.11)$ and the empirical setting with $(\hat{\radbase},\hat{\radabst})=(0.10,0.03)$.}
\label{tab:unified_f_misspec_tanh}
\centering
\begin{tabular}{lcccc}
\toprule
\multirow{2}{*}{Method} & \multicolumn{2}{c}{\texttt{SLC}} & \multicolumn{2}{c}{\texttt{LiLUCAS}} \\
\cmidrule(lr){2-3} \cmidrule(lr){4-5}
 & Gaussian & Empirical & Gaussian & Empirical \\
\midrule
\textsc{Bary}$_{\tauomega}$                         & 1.22 $\pm$ 0.01 & 186.94 $\pm$ 1.39 & \textbf{3.85 $\pm$ 0.01} & 564.09 $\pm$ 1.81 \\
\textsc{Grad}$_{\tauomega}$                         & 2.99 $\pm$ 0.01 & 140.07 $\pm$ 0.87 & 5.42 $\pm$ 0.04 & 338.69 $\pm$ 0.93 \\
$\absp$                                             & n/a             & 195.23 $\pm$ 1.62 & n/a             & 444.13 $\pm$ 1.27 \\
$\absn$                                             & n/a             & 191.79 $\pm$ 1.57 & n/a             & 386.72 $\pm$ 2.03 \\
\textsc{DiRoCA}$_{\widehat{\epsilon}}$  & \textbf{0.55 $\pm$ 0.01} & 192.73 $\pm$ 1.44 & 5.61 $\pm$ 0.04 & 486.29 $\pm$ 1.89 \\
\textsc{DiRoCA}$_{1,1}$                             & 2.05 $\pm$ 0.02 & 138.89 $\pm$ 0.84 & 5.86 $\pm$ 0.04 & 352.50 $\pm$ 1.84 \\
\textsc{DiRoCA}$_{2,2}$                             & 1.57 $\pm$ 0.01 & 136.98 $\pm$ 0.42 & 4.08 $\pm$ 0.03 & 333.73 $\pm$ 8.63 \\
\textsc{DiRoCA}$_{4,4}$                             & 2.32 $\pm$ 0.02 & 136.99 $\pm$ 0.42 & 4.96 $\pm$ 0.03 & \textbf{328.82 $\pm$ 8.26} \\
\textsc{DiRoCA}$_{8,8}$                             & 2.32 $\pm$ 0.02 & \textbf{136.92 $\pm$ 0.45} & 4.96 $\pm$ 0.03 & 329.04 $\pm$ 8.12 \\
\bottomrule
\end{tabular}
\end{table}

\begin{figure}
    \centering
    \includegraphics[width=.85\linewidth]{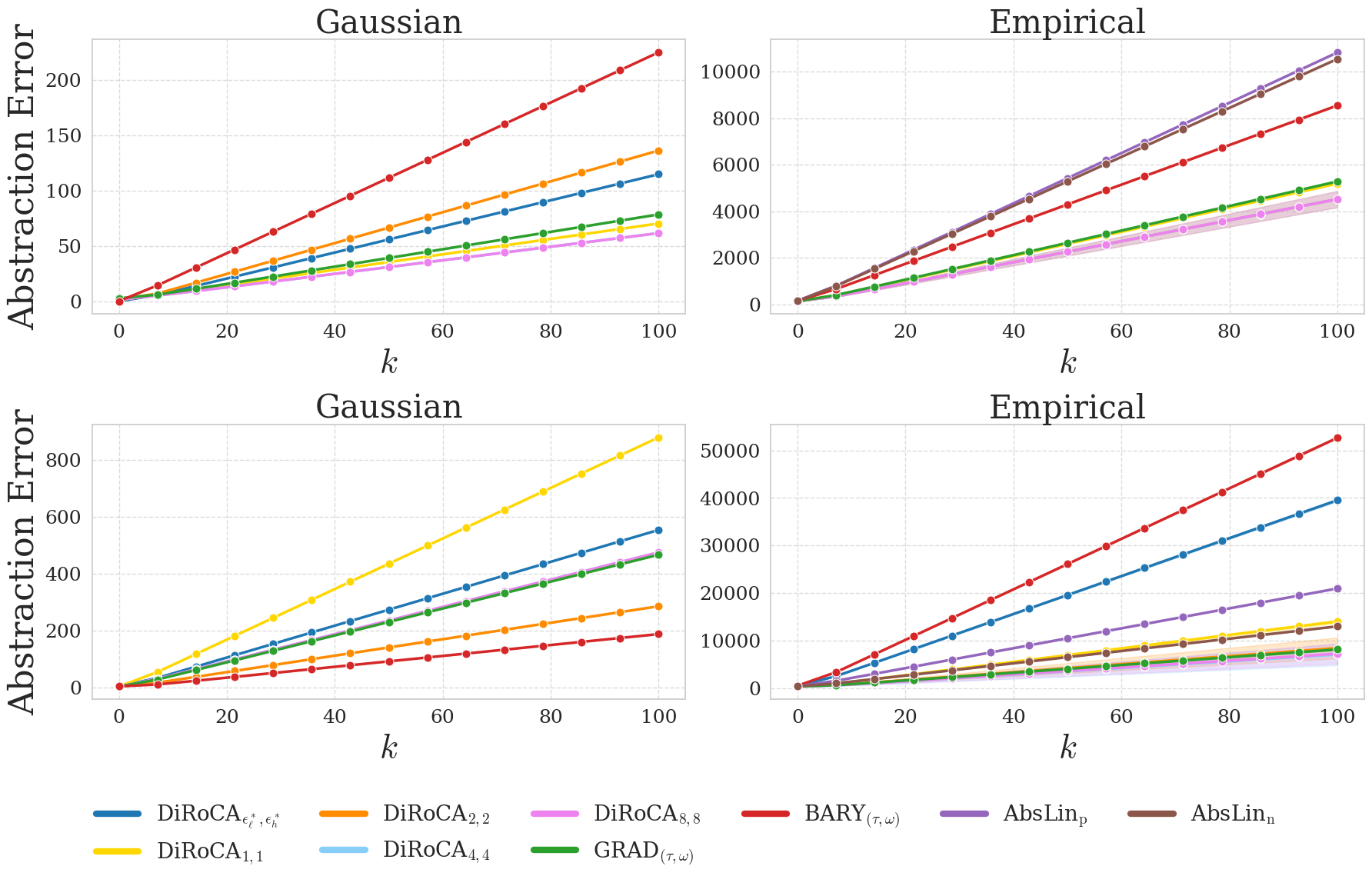}
    \caption{Abstraction error comparison as a function of the nonlinearity strength parameter $k$, which controls the strength of the nonlinear component (\texttt{tanh}) in the data generation process. The figure shows results for both Gaussian (left) and empirical (right) data distributions in the \texttt{SLC} (top) and \texttt{LiLUCAS} (bottom) settings.}
\label{fig:error_vs_k_tanh}
\end{figure}

\begin{table}[htp]
\caption{Average abstraction error (mean $\pm \sigma$) under \textit{intervention mapping} $\omega$-misspecification for the Gaussian setting with $(\hat{\radbase},\hat{\radabst})=(0.11,0.11)$ and the empirical setting with $(\hat{\radbase},\hat{\radabst})=(0.10,0.03)$.}
\label{tab:unified_omega_misspec}
\centering
\begin{tabular}{lcccc}
\toprule
\multirow{2}{*}{Method} & \multicolumn{2}{c}{\texttt{SLC}} & \multicolumn{2}{c}{\texttt{LiLUCAS}} \\
\cmidrule(lr){2-3} \cmidrule(lr){4-5}
 & Gaussian & Empirical & Gaussian & Empirical \\
\midrule
\textsc{Bary}$_{\tauomega}$                         & 2.77 $\pm$ 0.00 & 100.98 $\pm$ 0.24 & 5.86 $\pm$ 0.03 & 548.55 $\pm$ 2.48 \\
\textsc{Grad}$_{\tauomega}$                         & \textbf{0.71 $\pm$ 0.00} & 50.39 $\pm$ 0.15  & 5.63 $\pm$ 0.03 & 305.46 $\pm$ 1.31 \\
$\absp$                                             & n/a             & 111.24 $\pm$ 0.17 & n/a             & 414.62 $\pm$ 1.09 \\
$\absn$                                             & n/a             & 105.50 $\pm$ 0.14 & n/a             & 359.45 $\pm$ 1.44 \\
\textsc{DiRoCA}$_{\widehat{\epsilon}}$  & 2.57 $\pm$ 0.02 & 105.03 $\pm$ 0.24 & 6.08 $\pm$ 0.03 & 459.06 $\pm$ 2.80 \\
\textsc{DiRoCA}$_{1,1}$                             & 1.54 $\pm$ 0.00 & 46.21 $\pm$ 0.14  & 7.16 $\pm$ 0.03 & 326.56 $\pm$ 2.59 \\
\textsc{DiRoCA}$_{2,2}$                             & 1.48 $\pm$ 0.01 & \textbf{39.16 $\pm$ 2.75}  & \textbf{5.55 $\pm$ 0.03} & 308.52 $\pm$ 8.54 \\
\textsc{DiRoCA}$_{4,4}$                             & 1.43 $\pm$ 0.00 & \textbf{39.16 $\pm$ 2.75}  & 6.86 $\pm$ 0.02 & \textbf{304.11 $\pm$ 7.76} \\
\textsc{DiRoCA}$_{8,8}$                             & 1.43 $\pm$ 0.00 & 39.28 $\pm$ 2.51  & 6.86 $\pm$ 0.02 & 304.32 $\pm$ 7.83 \\
\bottomrule
\end{tabular}
\end{table}

\begin{figure}
    \centering
    \includegraphics[width=.85\linewidth]{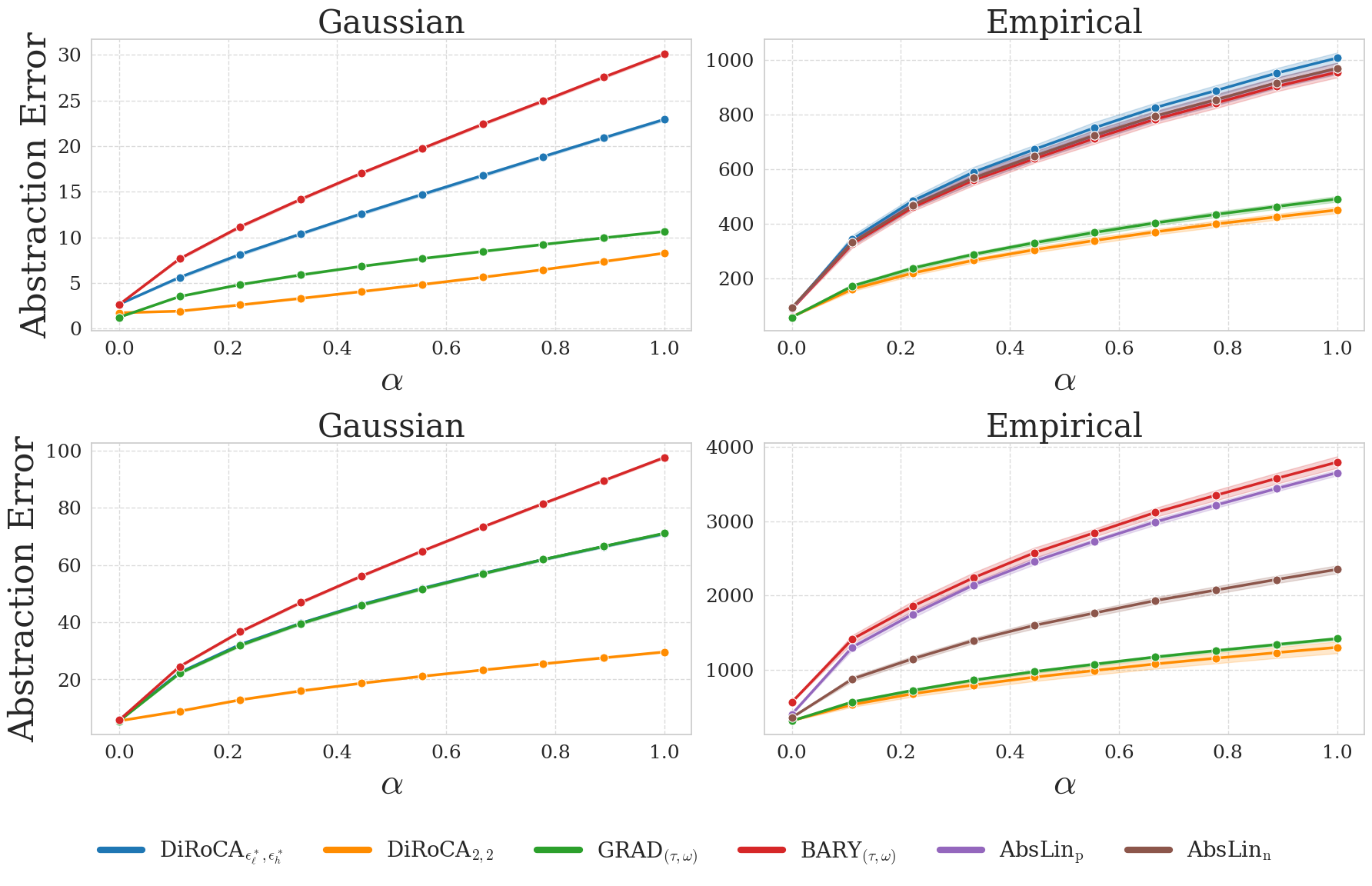}
    \caption{Robustness to outlier fraction ($\alpha$) on the \textbf{\texttt{SLC}} (top) and \texttt{LiLUCAS} (bottom) experiments for the Gaussian (left) and Empirical (right) settings. The evaluation is performed at a fixed \textit{Exponential} noise intensity ($\tilde{\sigma}=5.0$ for \texttt{SLC} and $\tilde{\sigma}=10.0$ for \texttt{LiLUCAS})}
\label{fig:error_vs_alpha_expo}
\end{figure}

\begin{figure}
    \centering
    \includegraphics[width=.85\linewidth]{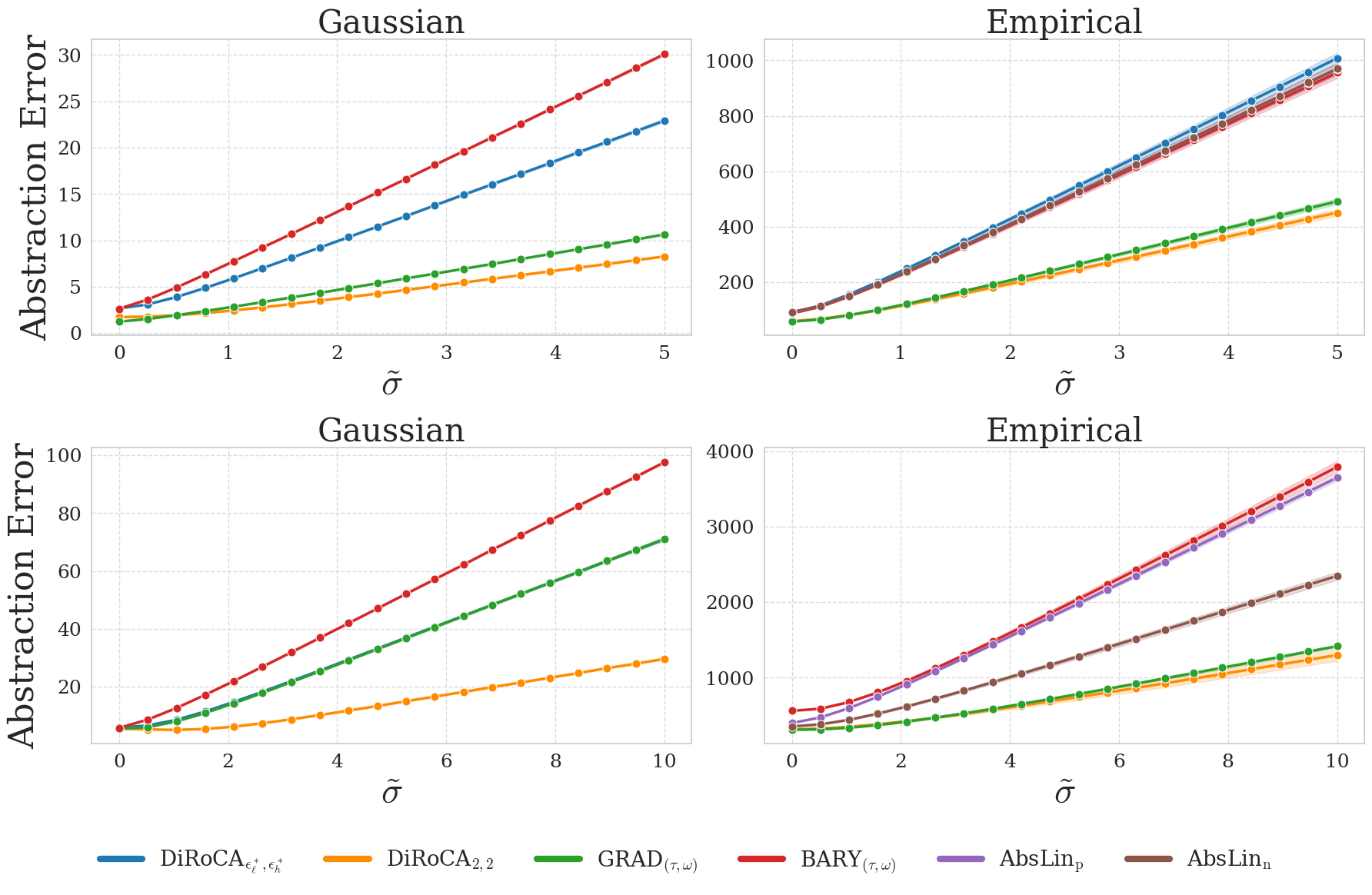}
    \caption{Robustness to \textit{Exponential} noise intensity ($\tilde{\sigma}$) on the \texttt{SLC} (top) and \texttt{LiLUCAS} (bottom) experiments for the Gaussian (left) and Empirical (right) settings. The evaluation is performed at a fixed outlier fraction of $\alpha=1.0$ (fully noisy data).}
\label{fig:error_vs_sigma_expo}
\end{figure}

\begin{figure}
    \centering
    \includegraphics[width=.85\linewidth]{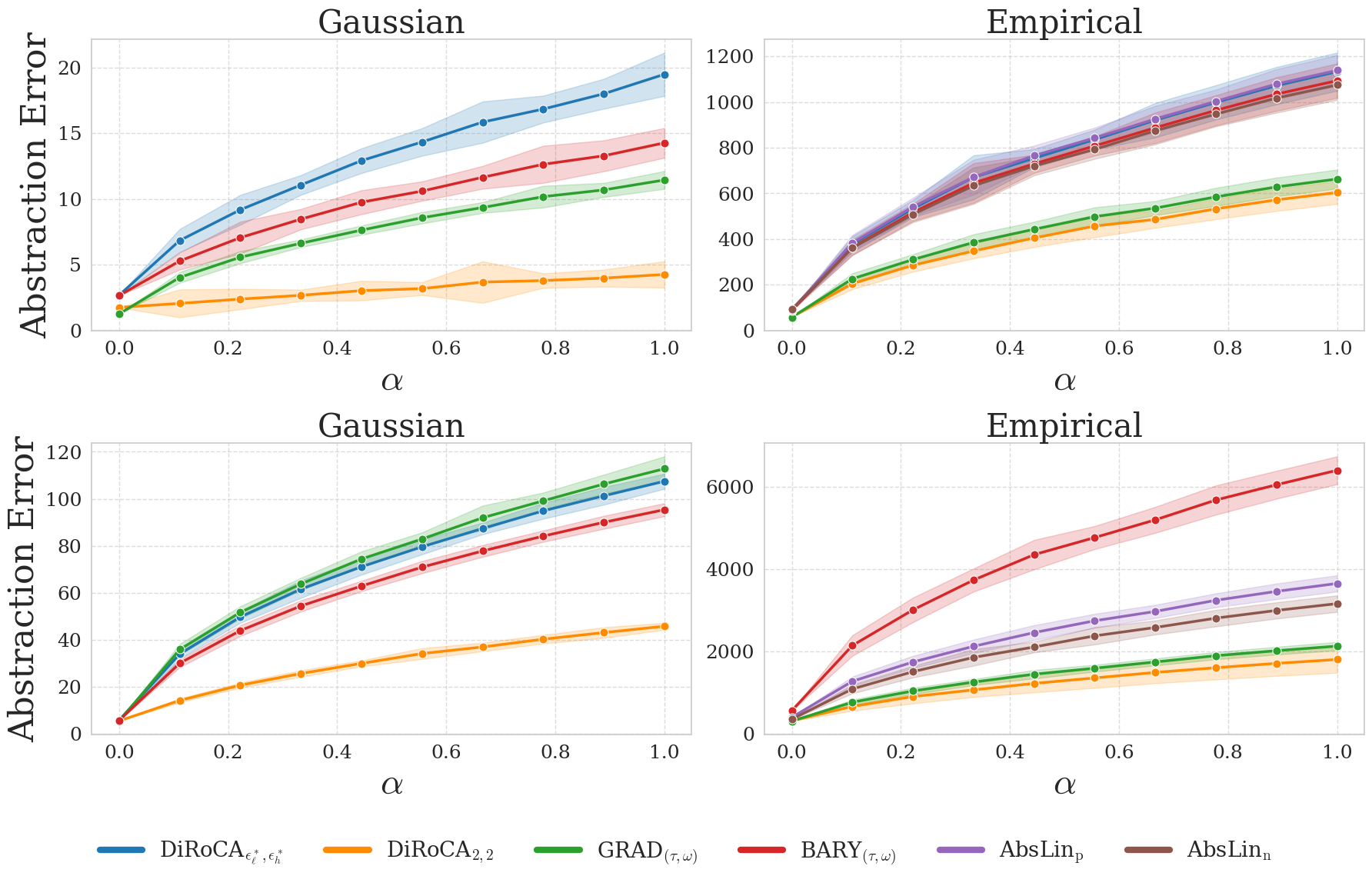}
    \caption{Robustness to outlier fraction ($\alpha$) on the \texttt{SLC} (top) and \texttt{LiLUCAS} (bottom) experiments for the Gaussian (left) and Empirical (right) settings. The evaluation is performed at a fixed \textit{Student-t} noise intensity ($\tilde{\sigma}=5.0$ for \texttt{SLC} and $\tilde{\sigma}=10.0$ for \texttt{LiLUCAS})}
\label{fig:error_vs_alpha_student}
\end{figure}

\begin{figure}
    \centering
    \includegraphics[width=.85\linewidth]{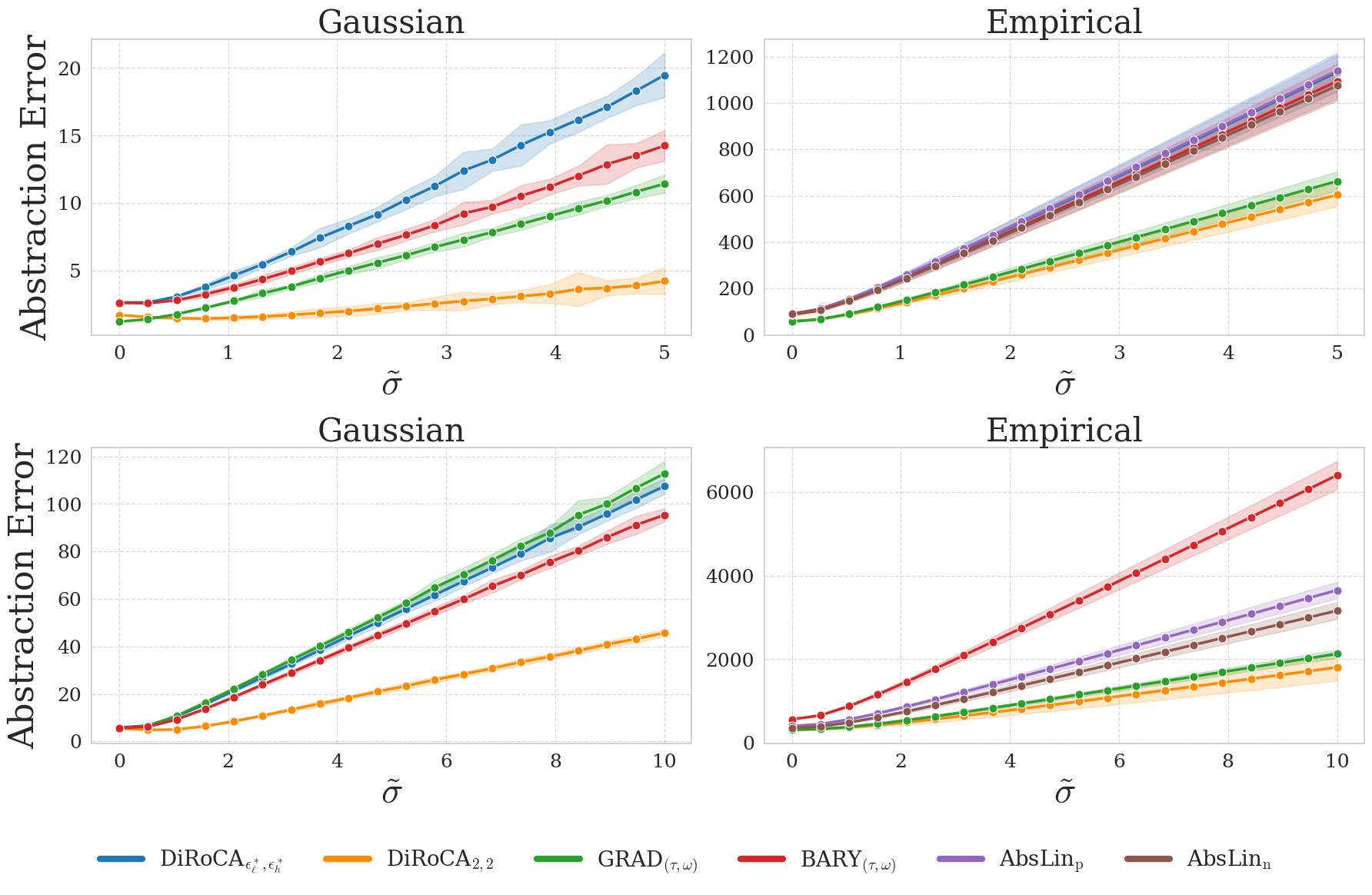}
    \caption{Robustness to \textit{Student-t} noise intensity ($\tilde{\sigma}$) on the \texttt{SLC} (top) and \texttt{LiLUCAS} (bottom) experiments for the Gaussian (left) and Empirical (right) settings. The evaluation is performed at a fixed outlier fraction of $\alpha=1.0$ (fully noisy data).}
\label{fig:error_vs_sigma_student}
\end{figure}

\begin{table}[htp]
\caption{Abstraction error and std under $\alpha=0$ for \texttt{nLUCAS} with $(\hat{\radbase}, \hat{\radabst})=(0.47, 0.45)$ comparing two evaluation protocols: \textbf{Consistency} (evaluating on the shared noise samples used during training) and \textbf{Generalization} (evaluating on new, independent noise samples). See Appendix~\ref{sec:all_results_settings} for details.}
\label{tab:lucas_alpha0_protocols}
\centering
\begin{tabular}{lcc}
\toprule
Method & Consistency & Generalization \\
\midrule
\textsc{Bary}$_{\tauomega}$ & 2.98 $\pm$ 0.02 & 18.74 $\pm$ 0.11 \\
\textsc{Grad}$_{\tauomega}$ & \textbf{2.42 $\pm$ 0.05} & 16.39 $\pm$ 0.13 \\
$\absp$ &5.94 $\pm$ 0.04  &26.09 $\pm$ 0.19\\
$\absn$ &7.93 $\pm$ 0.07 &20.36 $\pm$ 0.14\\
\textsc{DiRoCA}$_{\widehat{\epsilon}}$ & 2.51 $\pm$ 0.05 & 15.97 $\pm$ 0.15 \\
\textsc{DiRoCA}$_{1,1}$ & 6.35 $\pm$ 0.92 & 10.27 $\pm$ 0.53 \\
\textsc{DiRoCA}$_{2,2}$ & 8.90 $\pm$ 0.23 & 9.67 $\pm$ 0.11 \\
\textsc{DiRoCA}$_{4,4}$ & 13.48 $\pm$ 2.33 & \textbf{8.98 $\pm$ 0.10} \\
\textsc{DiRoCA}$_{8,8}$ & 13.49 $\pm$ 2.33 & 8.98 $\pm$ 0.10 \\
\bottomrule
\end{tabular}
\end{table}

\begin{table}[htp]
\caption{Abstraction error and std under $\alpha=1, \tilde{\sigma}=10.0$ (Gaussian noise) for \texttt{nLUCAS} with $(\hat{\radbase}, \hat{\radabst})=(0.47, 0.45)$ comparing the Consistency and Generalization evaluation protocols}
\label{tab:lucas_gaussian_alpha1}
\centering
\begin{tabular}{lcc}
\toprule
Method & Consistency & Generalization \\
\midrule
\textsc{Bary}$_{\tauomega}$ & 1008.04 $\pm$ 10.37 & 1024.79 $\pm$ 10.89 \\
\textsc{Grad}$_{\tauomega}$ & 863.91 $\pm$ 8.79 & 877.45 $\pm$ 9.65 \\
$\absp$ &1216.24 $\pm$ 12.32 &1238.59 $\pm$ 12.69\\
$\absn$ &876.45 $\pm$ 8.77 &889.06 $\pm$ 7.01\\
\textsc{DiRoCA}$_{\widehat{\epsilon}}$ & 794.99 $\pm$ 8.57 & 808.21 $\pm$ 8.38 \\
\textsc{DiRoCA}$_{1,1}$ & 298.52 $\pm$ 44.55 & 302.46 $\pm$ 45.85 \\
\textsc{DiRoCA}$_{2,2}$ & 203.82 $\pm$ 2.17 & 204.80 $\pm$ 2.35 \\
\textsc{DiRoCA}$_{4,4}$ & \textbf{126.14 $\pm$ 23.17} & \textbf{121.55 $\pm$ 25.34} \\
\textsc{DiRoCA}$_{8,8}$ & \textbf{126.17 $\pm$ 23.18} & \textbf{121.54 $\pm$ 25.55} \\
\bottomrule
\end{tabular}
\end{table}
\begin{table}[htp]
\caption{Robustness to intervention mapping misspecification on \texttt{nLUCAS}: Abstraction error under corrupted intervention map ($\omega$). Results reported for the Generalization protocol. \textsc{DiRoCA}$_{\widehat{\epsilon}}$ uses radii $(0.47, 0.45)$.}
\label{tab:lucas_omega}
\centering
\begin{tabular}{lc}
\toprule
Method & Error \\
\midrule
\textsc{Bary}$_{\tauomega}$ & 21.37 $\pm$ 1.01 \\
\textsc{Grad}$_{\tauomega}$ & 18.65 $\pm$ 0.88 \\
$\absp$ & 29.28 $\pm$ 1.20 \\
$\absn$ & 22.65 $\pm$ 0.89 \\
\textsc{DiRoCA}$_{\widehat{\epsilon}}$ & 18.17 $\pm$ 0.86 \\
\textsc{DiRoCA}$_{1,1}$ & 11.47 $\pm$ 0.78 \\
\textsc{DiRoCA}$_{2,2}$ & 10.17 $\pm$ 0.24 \\
\textsc{DiRoCA}$_{4,4}$ & \textbf{9.13 $\pm$ 0.29} \\
\textsc{DiRoCA}$_{8,8}$ & \textbf{9.13 $\pm$ 0.29} \\
\bottomrule
\end{tabular}
\end{table}

\begin{figure}
    \centering
    \includegraphics[width=0.85\linewidth]{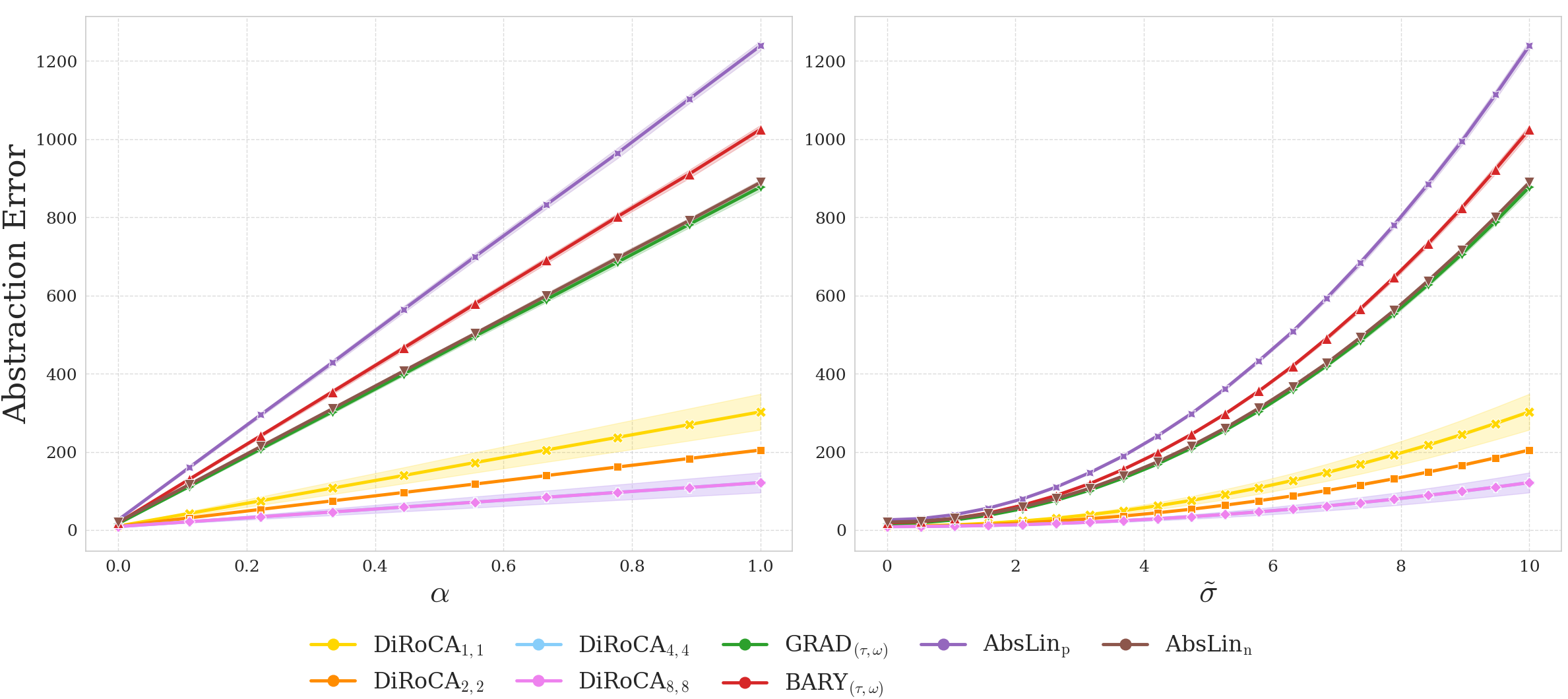}
    \caption{For the \texttt{nLUCAS} experiment. \textbf{Left:} Robustness to outlier fraction ($\alpha$) at a fixed \textit{Gaussian} noise intensity ($\tilde{\sigma}=10.0$). \textbf{Right:} Robustness to \textit{Gaussian} noise intensity ($\tilde{\sigma}$) at a fixed outlier fraction of $\alpha=1.0$.}
    \label{fig:placeholder}
\end{figure}
\begin{figure}
    \centering
    \includegraphics[width=0.85\linewidth]{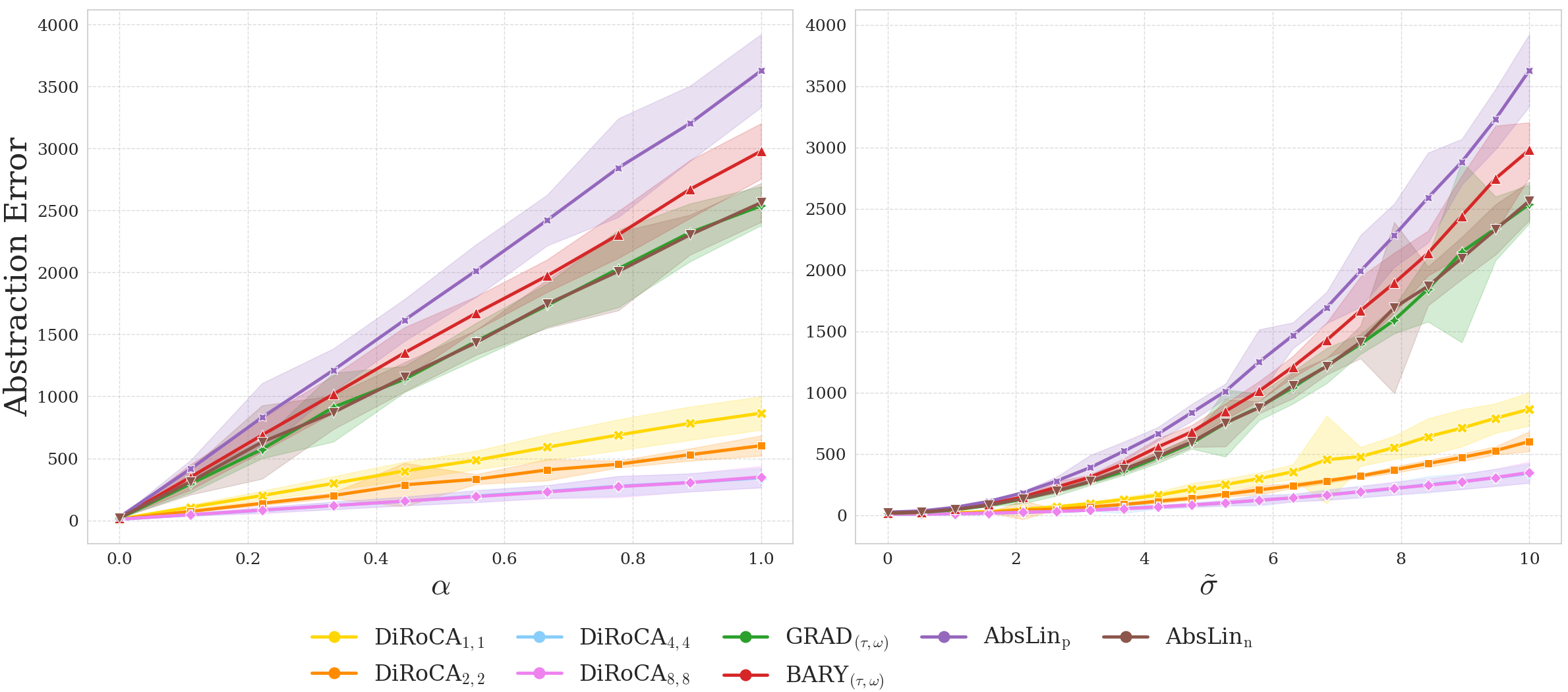}
    \caption{For the \texttt{nLUCAS} experiment. \textbf{Left:} Robustness to outlier fraction ($\alpha$) at a fixed \textit{Student-t} noise intensity ($\tilde{\sigma}=10.0$). \textbf{Right:} Robustness to \textit{Student-t} noise intensity ($\tilde{\sigma}$) at a fixed outlier fraction of $\alpha=1.0$.}
    \label{fig:placeholder}
\end{figure}
\begin{figure}
    \centering
    \includegraphics[width=0.85\linewidth]{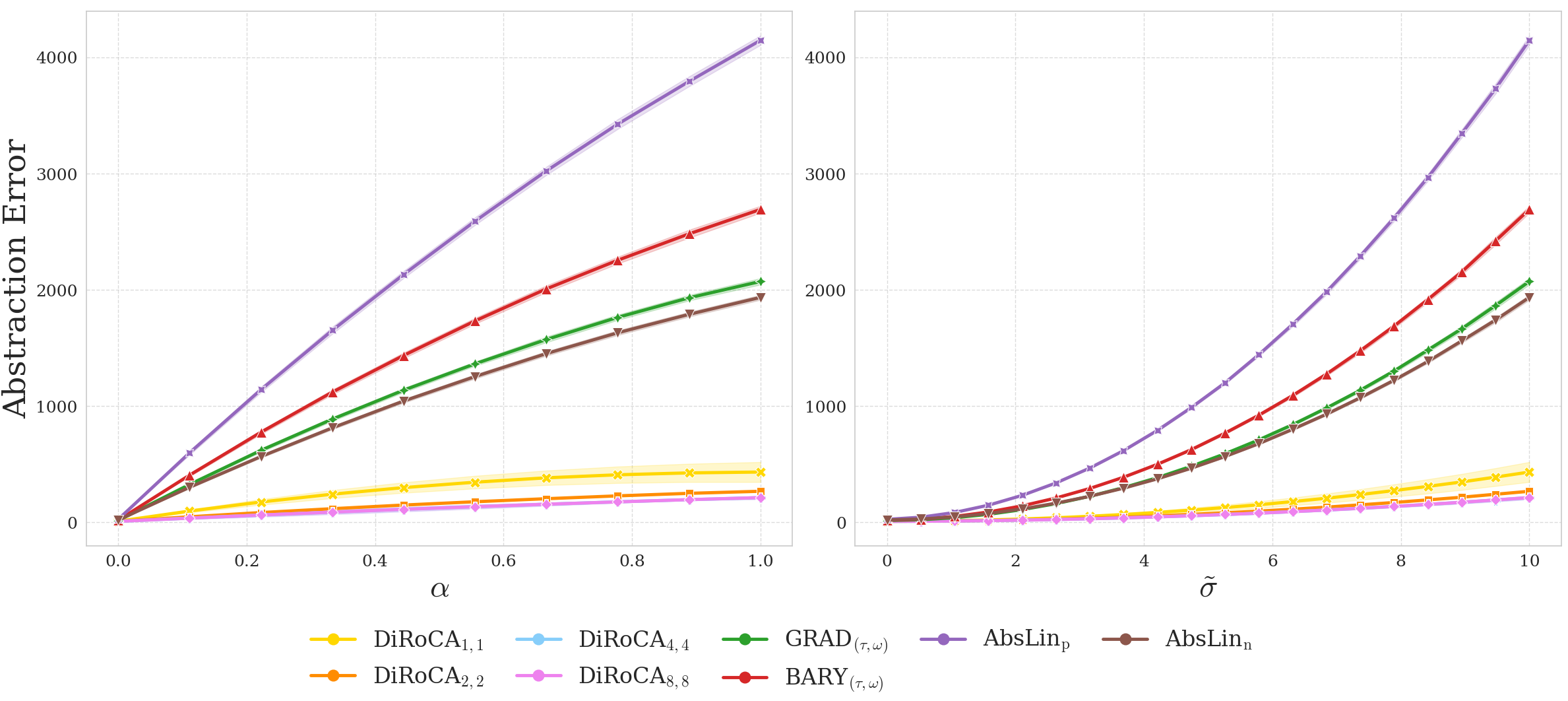}
    \caption{For the \texttt{nLUCAS} experiment. \textbf{Left:} Robustness to outlier fraction ($\alpha$) at a fixed \textit{Exponential} noise intensity ($\tilde{\sigma}=10.0$). \textbf{Right:} Robustness to \textit{Exponential} noise intensity ($\tilde{\sigma}$) at a fixed outlier fraction of $\alpha=1.0$.}
    \label{fig:placeholder}
\end{figure}

\begin{table}[htp]
\caption{Abstraction error and std under clean $\alpha=0$ and fully corrupted $\alpha=1, \tilde{\sigma}=10.0$ (Gaussian noise) settings for \texttt{EBM} with $(\hat{\radbase}, \hat{\radabst})=(0.55, 0.41)$.}
\label{tab:battery_gaussian_merged}
\centering
\begin{tabular}{lcc}
\toprule
Method & $\alpha=0$ & $\alpha=1$ \\
\midrule
$\absp$
& \textbf{58.85 $\pm$ 11.27} 
& 3034.75 $\pm$ 2129.28 \\

$\absn$
& 352.82 $\pm$ 47.07 
& 3330.89 $\pm$ 3238.00 \\

\textsc{Bary}$_{\tauomega}$ 
& 84.03 $\pm$ 14.06 
& 357.42 $\pm$ 144.91 \\

\textsc{Grad}$_{\tauomega}$ 
& 93.53 $\pm$ 26.43 
& 320.15 $\pm$ 148.20 \\

\textsc{DiRoCA}$_{\widehat{\epsilon}}$ 
& 83.96 $\pm$ 28.64 
& \textbf{297.15 $\pm$ 167.75} \\

\textsc{DiRoCA}$_{0.5,0.5}$ 
& 73.10 $\pm$ 13.12 
& \textbf{268.35 $\pm$ 158.99} \\

\textsc{DiRoCA}$_{1,1}$ 
& 77.80 $\pm$ 19.93 
& \textbf{268.38 $\pm$ 141.75} \\

\textsc{DiRoCA}$_{2,2}$ 
& 92.55 $\pm$ 26.11 
& 313.50 $\pm$ 125.72 \\

\textsc{DiRoCA}$_{4,4}$ 
& 91.76 $\pm$ 27.09 
& 321.39 $\pm$ 166.74 \\

\textsc{DiRoCA}$_{8,8}$ 
& 91.76 $\pm$ 27.09 
& \textbf{307.86 $\pm$ 172.27} \\
\bottomrule
\end{tabular}
\end{table}

\begin{figure}[t]
\centering

\includegraphics[width=.85\linewidth]{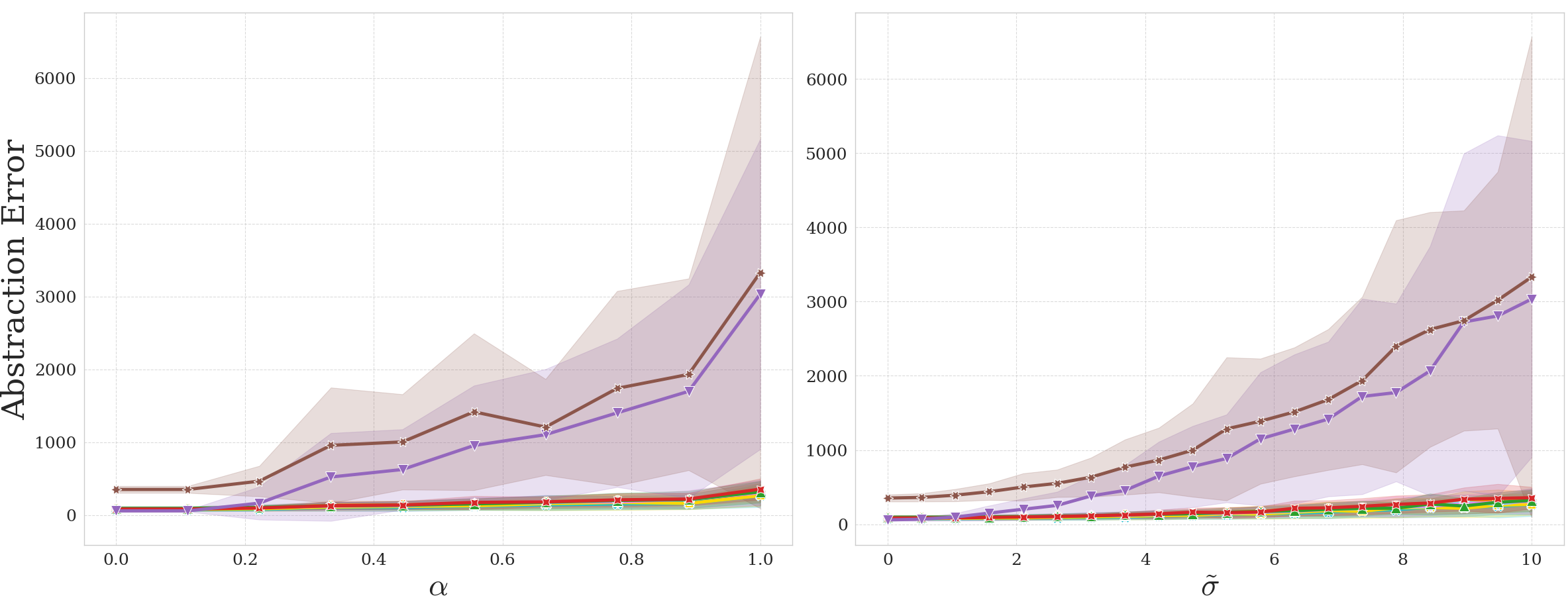}

\vspace{6pt} 

\includegraphics[width=.85\linewidth]{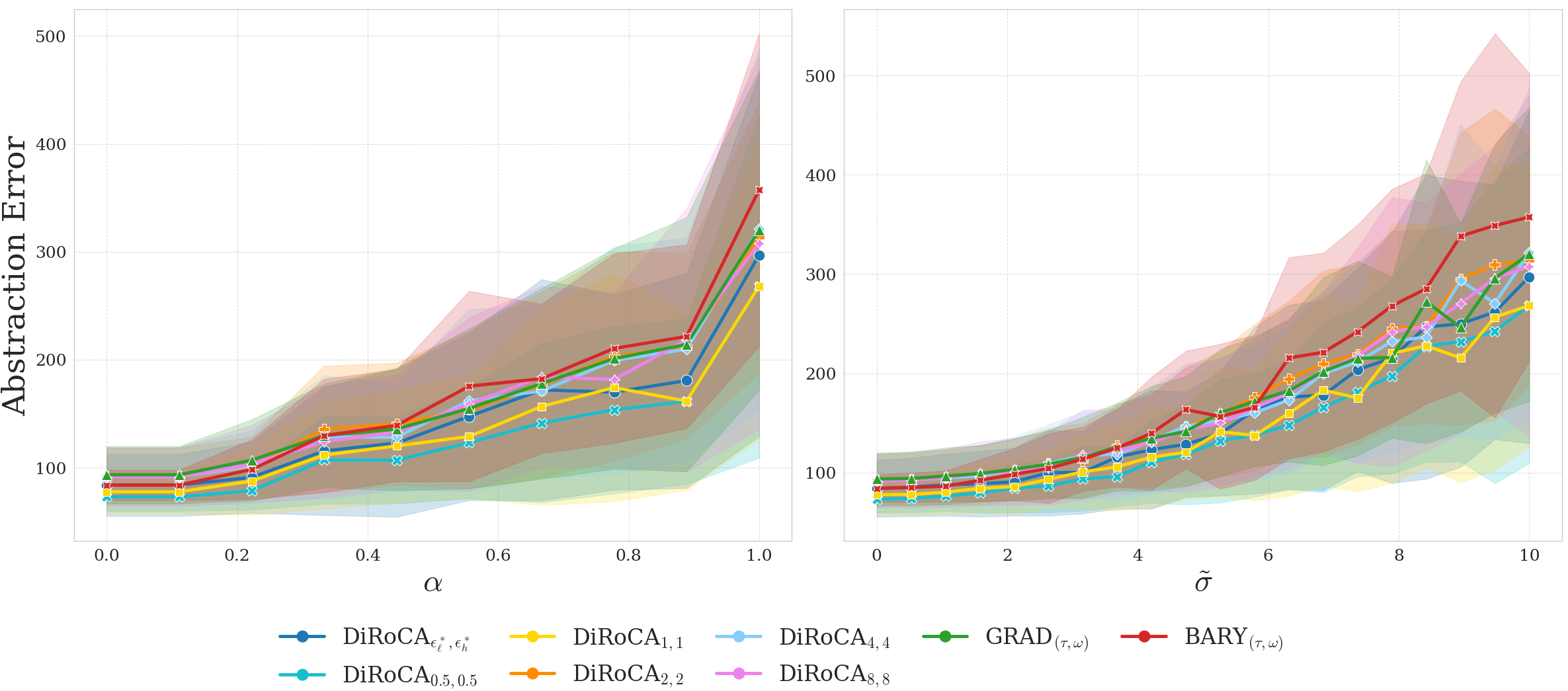}

\caption{For the \texttt{EBM} experiment. \textbf{Left:} Robustness to outlier fraction ($\alpha$) at a fixed \textit{Student-t} noise intensity ($\tilde{\sigma}=10.0$). \textbf{Right:} Robustness to \textit{Student-t} noise intensity ($\tilde{\sigma}$) at a fixed outlier fraction of $\alpha=1.0$. The bottom row presents results with the $\absp$ and $\absn$ baselines omitted for clearer visualization.}\label{fig:stacked_two}
\end{figure}

\begin{table}[htp]
\caption{Robustness to camera shifts on \texttt{cMNIST} (Relative L2 \%) with $(\hat{\radbase}, \hat{\radabst})=(61.69, 0)$. We evaluate two distinct shift types at increasing severity levels: \textbf{Lighting} (Low: dimming $\lambda=0.6$, High: overexposure $\lambda=3.0$) and \textbf{Rotation} (Low: $30^\circ$, High: $90^\circ$). \textsc{DiRoCA} remains stable under both shift types, while baselines degrade sharply.}
\label{tab:cmnist_lighting_rotation}
\centering
\begin{tabular}{lcccc}
\toprule
\multirow{2}{*}{Method} 
& \multicolumn{2}{c}{Lighting} 
& \multicolumn{2}{c}{Rotation} \\
\cmidrule(lr){2-3} \cmidrule(lr){4-5}
& Low & High & Low & High \\
\midrule
\textsc{Bary}$_{\tauomega}$             
& 97.48 $\pm$ 3.10  & 870.61 $\pm$ 34.88
& 2287.08 $\pm$ 198.31 & 4131.29 $\pm$ 363.07 \\
\textsc{Grad}$_{\tauomega}$             
& 32.90 $\pm$ 0.83  & 330.53 $\pm$ 8.36
& 1776.67 $\pm$ 235.11 & 3595.63 $\pm$ 653.21 \\
$\absp$                                 
& 56.09 $\pm$ 0.87  & 505.20 $\pm$ 11.27
& 541.15 $\pm$ 13.20 & 604.37 $\pm$ 24.12 \\
$\absn$                                 
& 95.76 $\pm$ 2.68  & 447.47 $\pm$ 10.69
& 276.81 $\pm$ 6.65 & 287.80 $\pm$ 9.26 \\
\midrule
\textsc{DiRoCA}$_{\hat{\radbase}, 0}$   
& 35.84 $\pm$ 3.64  & 137.99 $\pm$ 6.53
& 101.64 $\pm$ 6.20 & 106.40 $\pm$ 3.74 \\
\textsc{DiRoCA}$_{8, 0}$              
& \textbf{17.82 $\pm$ 0.66}  & 75.19 $\pm$ 4.03
& 52.63 $\pm$ 2.04 & 56.71 $\pm$ 2.59 \\
\textsc{DiRoCA}$_{20, 0}$               
& 18.96 $\pm$ 0.49  & \textbf{26.77 $\pm$ 0.94}
& \textbf{28.02 $\pm$ 0.39} & \textbf{30.64 $\pm$ 0.72} \\
\textsc{DiRoCA}$_{40, 0}$               
& 24.13 $\pm$ 2.63  & 72.93 $\pm$ 5.56
& 61.48 $\pm$ 4.36 & 66.60 $\pm$ 2.68 \\
\bottomrule
\end{tabular}
\end{table}

\begin{table}[htp]
\caption{Huber Noise Robustness on \texttt{cMNIST} (Relative L2 \%) for $\alpha=0$ (Clean) and $\alpha=1, \tilde{\sigma} = 0.5$ (fully corrupted) with $(\hat{\radbase}, \hat{\radabst})=(61.69, 0)$. 
Values are reported as mean $\pm$ standard deviation across trials. \textsc{DiRoCA} maintains low error even under full noise saturation.}
\label{tab:cmnist_huber_noise}
\centering
\begin{tabular}{lcc}
\toprule
Method & $\alpha=0.00$ & $\alpha=1.00$ \\
\midrule
\textsc{Bary}$_{\tauomega}$             
& 270.19 $\pm$ 8.69 
& 15338.79 $\pm$ 628.99 \\

\textsc{Grad}$_{\tauomega}$             
& 87.64 $\pm$ 2.17 
& 22207.40 $\pm$ 1501.67 \\

$\absp$                               
& 141.64 $\pm$ 2.44 
& 2951.69 $\pm$ 51.59 \\

$\absn$                                
& 227.89 $\pm$ 5.78 
& 470.84 $\pm$ 8.45 \\
\midrule
\textsc{DiRoCA}$_{\hat{\radbase},0}$   
& 69.80 $\pm$ 5.04 
& 214.91 $\pm$ 8.42 \\

\textsc{DiRoCA}$_{8.0,0}$              
& 35.97 $\pm$ 1.70 
& 60.33 $\pm$ 1.73 \\

\textsc{DiRoCA}$_{20,0}$               
& \textbf{20.05 $\pm$ 0.70} 
& \textbf{48.17 $\pm$ 0.64} \\

\textsc{DiRoCA}$_{40,0}$               
& 38.68 $\pm$ 3.20 
& 132.86 $\pm$ 7.48 \\
\bottomrule
\end{tabular}
\end{table}

\newpage
\subsection{Optimization}\label{sec:optapp}
\subsubsection{Distributionally Robust Causal Abstraction Learning (\textsc{DiRoCA})}
In this section, we provide a detailed analytical presentation of the optimization procedures of \textsc{DiRoCA} starting from the Linear models both for the Gaussian and empirical settings.

\paragraph{Gaussian case.} Recall that in the Gaussian case \textsc{DiRoCA} solves the following constrained min-max optimization problem:
\begin{align}
    \min_{\lintau \in \mathbb{R}^{h \times \ell}} \max_{\substack{\mu^\ell, \Sigma^\ell \\ \mu^h, \Sigma^h}} 
    \mathbb{E}_{\iota \sim q} \Big[ \mathcal{W}_2^2\big(\mathcal{N}(\lintau \mixbase_\iota \mu^\ell, \lintau \mixbase_\iota \Sigma^\ell \mixbase_\iota^\top \lintau^\top),~ \mathcal{N}(\mixabst_{\omega(\iota)} \mu^h, \mixabst_{\omega(\iota)} \Sigma^h \mixabst_{\omega(\iota)}^\top)\big) \Big]
\end{align}
supported by the environmental uncertainty constraints:
\begin{align}
\mathcal{W}^2_2(\mathcal{N}(\mu^\ell, \Sigma^\ell),~ \mathcal{N}(\widehat{\mu^\ell}, \widehat{\Sigma^\ell})) \leq \radbase^2, \quad
\mathcal{W}^2_2(\mathcal{N}(\mu^h, \Sigma^h),~ \mathcal{N}(\widehat{\mu^h}, \widehat{\Sigma^h})) \leq \radabst^2
\end{align}
where \(\widehat{\mu^\ell}, \widehat{\Sigma^\ell}, \widehat{\mu^h}, \widehat{\Sigma^h}\) are empirical estimates from observed data. Since we work with Markovian SCMs, the exogenous variables are jointly independent, which implies that their covariance matrices are diagonal. This structural property simplifies the form of the Wasserstein constraints, allowing them to be expressed as:
\begin{align}
\|\mu^\ell - \widehat{\mu^\ell}\|_2^2 + \|{\Sigma^\ell}^{1/2} - \widehat{\Sigma^{\ell}}^{1/2}\|_2^2 \leq \radbase^2
\quad
\text{and}
\quad
\|\mu^h - \widehat{\mu^h}\|_2^2 + \|{\Sigma^h}^{1/2} - \widehat{\Sigma^h}^{1/2}\|_2^2 \leq \radabst^{2}
\end{align}
Expanding the Gelbrich formula for the Gaussian 2-Wasserstein distance, the objective decomposes as:
\begin{align}\label{eq:exact_F}
    F(\lintau, \mu^\ell, \Sigma^\ell, \mu^h, \Sigma^h) = \underbrace{\mathbb{E}_{\iota \sim q} \Big[
    \|\lintau \mixbase_\iota \mu^\ell - \mixabst_{\omega(\iota)} \mu^h\|_2^2 + \operatorname{Tr}(\lintau \mixbase_\iota \Sigma^\ell \mixbase_\iota^\top \lintau^\top) + \operatorname{Tr}(\mixabst_{\omega(\iota)} \Sigma^h \mixabst_{\omega(\iota)}^\top)}_{F_s (\text{smooth term})}\Big] \\ 
    \underbrace{-\mathbb{E}_{\iota \sim q} \Big[ 2\operatorname{Tr}\big( (\mixabst_{\omega(\iota)} \Sigma^h \mixabst_{\omega(\iota)}^\top)^{1/2} (\lintau \mixbase_\iota \Sigma^\ell \mixbase_\iota^\top \lintau^\top) (\mixabst_{\omega(\iota)} \Sigma^h \mixabst_{\omega(\iota)}^\top)^{1/2} \big)^{1/2}}_{F_n (\text{non-smooth term})}
    \Big]
\end{align}
We rewrite the objective as the sum of a \emph{smooth} component \(F_{\text{s}}\) capturing the quadratic terms and a \emph{non-smooth} component \(F_{\text{n}}\) capturing the last trace term. Thus the objective becomes:
\begin{align}
\min_{\lintau} \max_{\mu^\ell, \Sigma^\ell, \mu^h, \Sigma^h} \quad F_{\text{s}}(\lintau, \mu^\ell, \Sigma^\ell, \mu^h, \Sigma^h) + F_{\text{n}}(\lintau, \Sigma^\ell, \Sigma^h)
\end{align}
Let \(S_\ell = \lintau \mixbase_\iota \Sigma^\ell \mixbase_\iota^\top \lintau^\top\) and \(S_h=\mixabst_{\omega(\iota)} \Sigma^h \mixabst_{\omega(\iota)}^\top\). Then the non-smooth term simplifies as:
\begin{align}
    \operatorname{Tr}\left(\left(S_h^{1/2} S_\ell S_h^{1/2}\right)^{1/2}\right) = \|S_\ell^{1/2} S_h^{1/2}\|_*
\end{align}
To handle the non-smoothness of the nuclear norm in optimization, we use the variational characterization from \citet{NIPS2004_e0688d13}. For any matrix \(X = UV^\top\),
\begin{align}
\|X\|_* = \min_{\substack{U,V:\\ X = UV^\top}} \|U\|_{\operatorname{F}} \|V\|_{\operatorname{F}}
\end{align}
In our case:
\begin{align}
    \|S_\ell^{1/2} S_h^{1/2}\|_* \leq \|S_\ell^{1/2}\|_{\operatorname{F}} \|S_h^{1/2}\|_{\operatorname{F}}
\end{align}
Applying this inequality, we upper bound the non-smooth term, allowing the optimization to proceed efficiently using standard gradient and proximal-gradient methods. Specifically, we maximize a surrogate lower bound \(\tilde{F} \leq F\):
\begin{align}\label{approxi_f}
 \tilde{F} \coloneq F_{\text{s}}(\lintau, \mu^\ell, \Sigma^\ell, \mu^h, \Sigma^h) + F^*_{\text{n}}(\lintau, \Sigma^\ell, \Sigma^h)
\end{align}
where \(F^*_{\text{n}}(\lintau, \Sigma^\ell, \Sigma^h) = - \|(\lintau \mixbase_\iota \Sigma^\ell \mixbase_\iota^\top \lintau^\top)^{1/2}\|_{\operatorname{F}} \|(\mixabst_{\omega(\iota)} \Sigma^h \mixabst_{\omega(\iota)}^\top)^{1/2}\|_{\operatorname{F}}\).

The optimization algorithm alternates between updating \(\lintau\) via gradient descent on \eqref{eq:exact_F}, and updating the moments \((\mu^\ell, \Sigma^\ell), (\mu^h, \Sigma^h)\) via proximal-gradient ascent on \eqref{approxi_f} (see Algorithm~\ref{alg:elliptical_algo}). Figures~\ref{fig:marginals_slc} and~\ref{fig:marginals_lilucas} provide a visual illustration, showing the initial empirical distribution ($\mathcal{N}(\widehat{\mu}, \widehat{\Sigma})$) and the final worst-case distribution ($\mathcal{N}(\mu^\star, \Sigma^\star)$) learned by the adversary for our two experimental datasets.

\textbf{Note 1.} The proximal operator of the Frobenius norm of a matrix $A\in R^{m\times n}$ is given as
\begin{align}
\mathrm{prox}_{\lambda \|\cdot\|_F}(A) =
\begin{cases}
\left(I - \frac{\lambda}{\|A\|_F} \right) A, & \text{if } \|A\|_F > \lambda \\
0, & \text{otherwise}
\end{cases}.
\end{align}
Here $A$ denotes the square root of the transformed covariance matrices (e.g.,$(\lintau L_\iota \Sigma^\ell L_\iota^\top \lintau^\top)^{1/2}$ and $(H_{\omega(\iota)} \Sigma^h H_{\omega(\iota)}^\top)^{1/2}$).

\textbf{Note 2.} The projection operation $\text{proj}_{\mathcal{W}_2 \leq \radabst}(\cdot,\cdot,\cdot,\cdot)$ enforces the Wasserstein constraints:
\begin{align}
\mathcal{W}^2_2\big(\mathcal{N}(\mu, \Sigma),~ \mathcal{N}(\hat{\mu}, \hat{\Sigma})\big) \leq \varepsilon^2
\quad \text{via projection onto a Gelbrich ball.}
\end{align}
This is done iteratively using the update:
\begin{align}
(\mu, \Sigma) \leftarrow \hat{\mu} + \alpha (\mu - \hat{\mu}),~ \hat{\Sigma}^{1/2} + \alpha(\Sigma^{1/2} - \hat{\Sigma}^{1/2})
\end{align}
with a scaling factor \(\alpha = \epsilon / \mathcal{W}_2((\mu, \Sigma), (\hat{\mu}, \hat{\Sigma}))\).

\begin{figure}[ht]
    \centering

    \begin{subfigure}[b]{\linewidth}
        \centering
        \includegraphics[width=\linewidth]{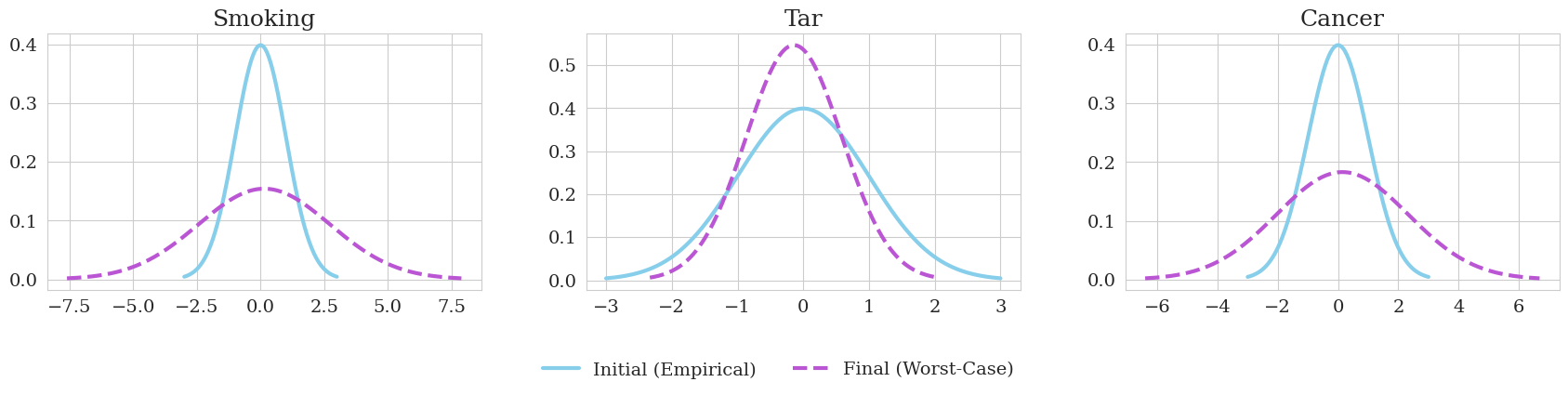}
        \caption{Low-level SCM ($\scmbase$)}
        \label{fig:marginals_slc}
    \end{subfigure}

    \vspace{0.5em}  

    \begin{subfigure}[b]{\linewidth}
        \centering
        \includegraphics[width=.67\linewidth]{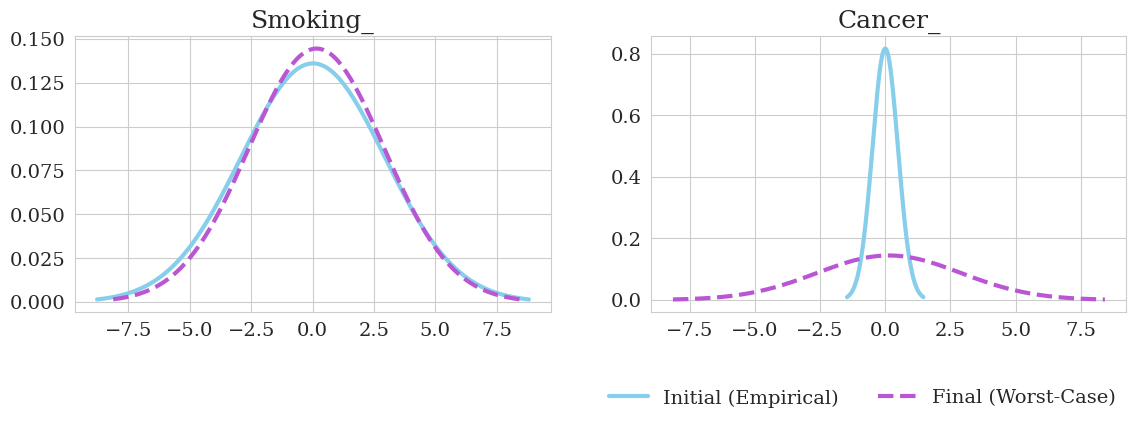}
        \caption{High-level SCM ($\scmabst$)}
        \label{fig:marginals_lilucas}
    \end{subfigure}
    \caption{Marginal noise distributions for each variable in the low-level (top) and high-level (bottom) SCMs of the \texttt{SLC} dataset. In each subplot, the \textcolor{SkyBlue}{solid blue} curve represents the initial empirical noise distribution, and the \textcolor{violet}{dashed purple} curve shows the corresponding worst-case distribution obtained via the distributionally robust optimization.}
    \label{fig:marginal-stacked-slc}
\end{figure}

\begin{figure}[ht]
    \centering

    \begin{subfigure}[b]{\linewidth}
        \centering
        \includegraphics[width=\linewidth]{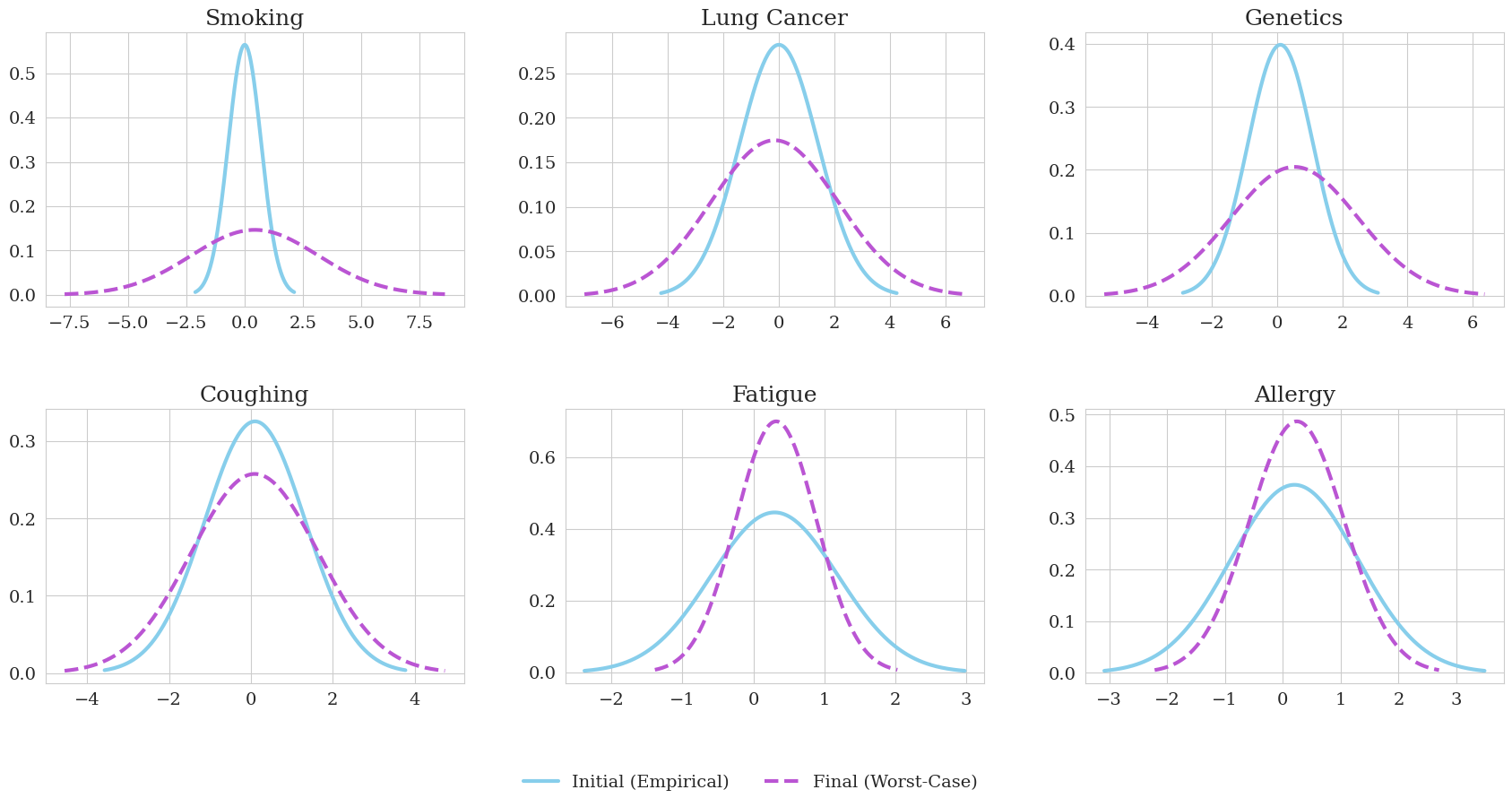}
        \caption{Low-level SCM ($\scmbase$)}
        \label{fig:marginal-low}
    \end{subfigure}

    \vspace{0.5em}  

    \begin{subfigure}[b]{\linewidth}
        \centering
        \includegraphics[width=\linewidth]{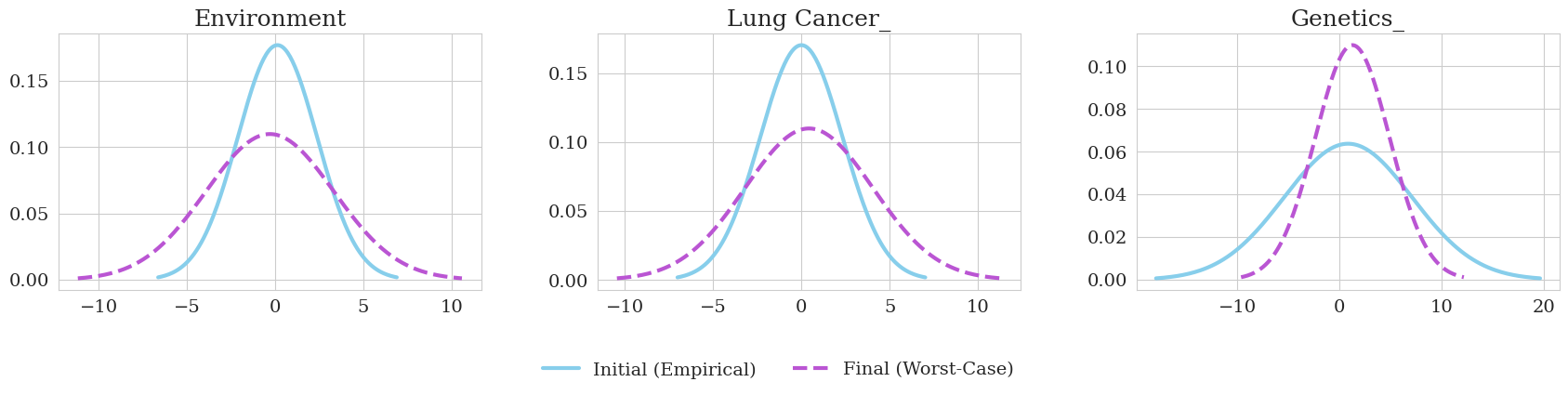}
        \caption{High-level SCM ($\scmabst$)}
        \label{fig:marginal-high}
    \end{subfigure}
    \caption{Marginal noise distributions for each variable in the low-level (top) and high-level (bottom) SCMs of the \texttt{LiLUCAS} dataset. In each subplot, the \textcolor{SkyBlue}{solid blue} curve represents the initial empirical noise distribution, and the \textcolor{violet}{dashed purple} curve shows the corresponding worst-case distribution obtained via the distributionally robust optimization.}
    \label{fig:marginal-stacked-luc}
\end{figure}

\paragraph{Empirical case.} When there are no structural assumptions about the distributional form of the environment, the empirical formulation provides a fully data-driven robust alternative, particularly useful when the low- and high-level environments are available only through finite samples. In this case, we introduce perturbation matrices \(\Theta_d \in \mathbb{R}^{N \times d}\) over the empirical exogenous samples and define $\rho^d(\Theta_d) = \frac{1}{N} \sum_{i=1}^{N} \delta_{\widehat{u}_i + \theta_i}$. The objective then measures the squared Frobenius discrepancy between the perturbed distributions of the two SCMs:
\begin{align}
F(\lintau, \Theta_\ell, \Theta_h) := \mathbb{E}_{\iota \sim q} \left[ 
\left\| 
\lintau \, \mixbase_{\iota}(\operatorname{U}^\ell + \Theta_\ell) - 
\mixabst_{\omega(\iota)}(\operatorname{U}^h + \Theta_h) 
\right\|_{\mathrm{F}}^2 
\right],
\end{align}
where $\operatorname{U}^\ell, \operatorname{U}^h$ are the empirical exogenous samples for $\scmbase$ and $\scmabst$, respectively. The perturbations capture adversarial shifts around these nominal empirical distributions. A proposed reparameterization reduces the Wasserstein ambiguity sets to constrained Frobenius norm balls \citep{kuhn2019wassersteindistributionallyrobustoptimization}: 
\begin{align}
\|\Theta_\ell\|_F \leq \radbase \sqrt{N_\ell}, \quad \|\Theta_h\|_F \leq \radabst \sqrt{N_h}
\end{align}
We solve this via alternating projected gradient descent. The optimization alternates between updating the abstraction map $\lintau$ with gradient descent and the adversarial perturbations of the samples through projected gradient ascent. Specifically,
\begin{itemize}
\item \textbf{Minimization step (updating $\lintau$):}  
Fix perturbations $(\Theta_\ell, \Theta_h)$ applied to the samples and perform gradient descent to update the abstraction map $\lintau$ by minimizing the empirical loss.
\item \textbf{Maximization step (updating perturbations $\Theta_\ell, \Theta_h$):}  
Fix $\lintau$ and perform adversarial updates to the perturbations:
\begin{itemize}
    \item Perturbations $\Theta_\ell$ are added to $U^\ell$ and $\Theta_h$ to $U^h$.
    \item Gradient ascent steps are performed to maximize the abstraction error.
    \item After each ascent step, perturbations are projected onto Frobenius norm balls to ensure they lie within a fixed robustness radius ($\radbase$ for $\Theta$, $\radabst$ for $\Theta_h$), enforcing constraints derived from the Wasserstein DRO ambiguity sets.
\end{itemize}
\end{itemize}

\paragraph{General Case.}
For general SCMs where mechanisms can be nonlinear, we rely on the $(D, U)$ decomposition detailed in Appendix~\ref{sec:anm_decomposition}. Unlike the linear case where mechanism matrices can be explicitly multiplied, here we operate on the abduced residuals directly.

We define the empirical marginal environments as $\widehat{\rho^d}= \frac{1}{N}\sum_{i=1}^{N} \delta_{\widehat{u}^d_i}$ for $d\in\{\ell,h\}$, and the joint environment as $\widehat{\exprob}= \widehat{\exprobbase} \otimes \widehat{\exprobabst}$. By the finite-dimensional reduction of \citet{kuhn2019wassersteindistributionallyrobustoptimization}, any distribution in a 2-Wasserstein ball around an empirical measure can be represented as a perturbed empirical distribution. We thus perturb the noise samples as $U^d \mapsto U^d + \Theta_d$, where $\Theta_d \in \mathbb{R}^{N \times d}$. The ambiguity radii $\varepsilon_\ell,\varepsilon_h$ correspond to Frobenius budgets $r_d = \varepsilon_d \sqrt{N}$, such that the condition $\|\Theta_d\|_F \le r_d$ is analytically equivalent to restricting the joint environment to $\bB_{\radprod,2}(\widehat{\exprob})$.

To formulate the robust objective, let $\mathbf{D}_\ell^{(\iota)}$, $\mathbf{D}_h^{(\omega(\iota))}$ be the pre-computed deterministic component matrices for intervention $\iota$ and the shared observational residuals are $\mathbf{U}^\ell$ and $\mathbf{U}^h$. We define the \emph{nominal misalignment} for an abstraction $\lintau$ as:
\begin{align}
    Z_{\lintau}^{\iota}(\mathbf{0}) \coloneq \lintau(\mathbf{D}_\ell^{(\iota)} + \mathbf{U}^\ell)^\top - (\mathbf{D}_h^{(\omega(\iota))} + \mathbf{U}^h).
\end{align}
Introducing perturbations $\mathbf{\Theta} \coloneq (\Theta_\ell, \Theta_h)$, the \emph{perturbed misalignment} shifts to:
\begin{align}
    Z_{\lintau}^{\iota}(\mathbf{\Theta}) \coloneq Z_{\lintau}^{\iota}(\mathbf{0}) + (\lintau\Theta_\ell^\top - \Theta_h).
\end{align}
The \textsc{DiRoCA} objective minimizes the expected perturbed misalignment under the worst-case observational shift within the ambiguity set:
\begin{align}\label{eq:minmax_general_app}
    \min_{\lintau} \sup_{\substack{\|\Theta_\ell\|_F \le r_\ell,~ \|\Theta_h\|_F \le r_h}}
    \mathbb{E}_{\iota\sim q}\!\left[
    \|Z_{\lintau}^{\iota}(\mathbf{\Theta})\|_F^2
    \right].
\end{align}
This formulation highlights that the adversary attempts to align the perturbation shift $(\lintau\Theta_\ell^\top - \Theta_h)$ with the nominal direction $Z_{\lintau}^{\iota}(\mathbf{0})$ to maximize the error. We optimize this problem via Alternating Projected Gradient Descent-Ascent, where the perturbations $\Theta$ are updated via projected gradient ascent (projecting onto the Frobenius balls at each step) and the abstraction $\lintau$ is updated via gradient descent, as detailed in Algorithm~\ref{alg:general_diroca}.
\subsubsection{Barycentric Abstraction Learning (\textsc{Bary}\texorpdfstring{$_{\tauomega}$}{[]}).}
We also consider an abstraction learning method based on \emph{barycentric aggregation} of interventional distributions, inspired by \citet{felekis2024causal}. 

\paragraph{Gaussian case.} In the case of linear models when the environments are Gaussian distributions, the key idea is to aggregate the different interventional distributions at both the low- and high-level SCMs into representative barycentric means and covariances, project the low-level distribution into the high-level space, and then minimize the Wasserstein distance between these two distributions.

Formally, given mixing matrices $\{\mixbase_\iota\}_{\iota \in \mathcal{I}^\ell}$ and $\{\mixabst_{\eta}\}_{\eta \in \mathcal{I}^h}$ corresponding to the low- and high-level SCMs under interventions, we define the \emph{barycentric mean} at each level according to Eq.~\ref{eq:barymean}:
\begin{align}
\mu^\ell_{\mathrm{bary}} = \frac{1}{|\mathcal{I}^\ell|} \sum_{\iota \in \mathcal{I}^\ell} \mixbase_\iota \mu_U^\ell,
\quad
\mu^h_{\mathrm{bary}} = \frac{1}{|\mathcal{I}^h|} \sum_{\eta \in \mathcal{I}^h} \mixabst_{\eta} \mu_U^h.
\end{align}
The \emph{barycentric covariance} at the low level, $\Sigma^\ell_{\mathrm{bary}}$, is obtained as the Wasserstein barycenter of the set $\{\mixbase_\iota \Sigma_U^\ell \mixbase_\iota^\top\}_{\iota \in \mathcal{I}^\ell}$, solving the fixed-point equation from Eq.~\ref{eq:barysigma}:
\begin{align}
\Sigma^\ell_{\mathrm{bary}} = \frac{1}{|\mathcal{I}^\ell|}\sum_{\iota \in \mathcal{I}^\ell} \lambda_\iota \left( \Sigma^\ell_{\mathrm{bary}}{}^{1/2} \left( \mixbase_\iota \Sigma_U^\ell \mixbase_\iota^\top \right) \Sigma^\ell_{\mathrm{bary}}{}^{1/2} \right)^{1/2},
\end{align}
The barycentric covariance $\Sigma^h_{\mathrm{bary}}$ at the high level is computed analogously over the set $\{\mixabst_{\eta} \Sigma_U^h \mixabst_{\eta}^\top\}_{\eta \in \mathcal{I}^h}$. To map from the low-level to the high-level space, we first project the low-level barycentric distribution onto the high-level dimension. Specifically, letting $V \in \mathbb{R}^{\ell \times h}$ denote the projection matrix, obtained via PCA, we define the projected mean and covariance:
\begin{align}
\mu^\ell_{\mathrm{proj}} = V^\top \mu^\ell_{\mathrm{bary}},
\quad
\Sigma^\ell_{\mathrm{proj}} = V^\top \Sigma^\ell_{\mathrm{bary}} V.
\end{align}
Consequently, by defining the abstraction map as the optimal transport Monge map between the distributions \(\mathcal{N}(\mu^\ell_{\mathrm{proj}}, \Sigma^\ell_{\mathrm{proj}})\) and \(\mathcal{N}(\mu^h_{\mathrm{bary}}, \Sigma^h_{\mathrm{bary}})\), from Eq.~\ref{eq:mongegauss}, since both distributions are Gaussian, this admits the following closed-form expression:
\begin{align}
T(x) = \mu^h_{\mathrm{bary}} + A(x - \mu^\ell_{\mathrm{proj}}),
\quad \text{where} \quad A = \Sigma^h_{\mathrm{bary}}{}^{1/2} \left( \Sigma^\ell_{\mathrm{proj}} \right)^{-1/2}.
\end{align}
Finally, since we initially projected the low-level distribution, the full transformation matrix is:
\begin{align}
\lintau = A V^\top,
\end{align} 
In both cases, the resulting abstraction approximately aligns the barycentric summaries of the low- and high-level interventional distributions. 

\paragraph{Empirical case.} In the linear setting where only empirical samples of the environments are available and no structural assumptions are made about their underlying distributions, we simply take the arithmetic mean across all the interventions of the mixing matrices $\mixbase_{bary}$ and $\mixabst_{bary}$ and propagate the exogenous datasets $\operatorname{U}^\ell$ and $\operatorname{U}^h$ through them and solve with gradient descent the following problem:
\begin{align}
\min_{\lintau \in \bR^{h\times \ell}}
\left\| 
\lintau \, \mixbase_{bary}(\operatorname{U}^\ell) - 
\mixabst_{bary}(\operatorname{U}^h) 
\right\|_{\mathrm{F}}^2 
\end{align}
While \textsc{Bary}$_{\tauomega}$ is computationally efficient, it does not explicitly optimize robustness to distributional shifts, unlike \textsc{DiRoCA}.

\paragraph{General Case.}
In the general setting (including nonlinear ANMs), we implement \textsc{Bary}$_{\tauomega}$ using the $(D, U)$ decomposition. Instead of averaging linear mixing matrices, we compute the \emph{barycentric deterministic component} by taking the arithmetic mean of the pre-computed deterministic matrices $\mathbf{D}^{(\iota)}$ across all training interventions. This yields an "average mechanism" $\overline{\mathbf{D}}^\ell = \frac{1}{|\mathcal{I}^\ell|} \sum_{\iota} \mathbf{D}_\ell^{(\iota)}$ and $\overline{\mathbf{D}}^h = \frac{1}{|\mathcal{I}^h|} \sum_{\eta} \mathbf{D}_h^{(\eta)}$. We then form a representative barycentric dataset by combining these centroids with the empirical residuals: $\mathbf{X}^\ell_{\text{bary}} = \overline{\mathbf{D}}^\ell + \mathbf{U}^\ell$ and $\mathbf{X}^h_{\text{bary}} = \overline{\mathbf{D}}^h + \mathbf{U}^h$. The abstraction $\lintau$ is learned by minimizing the Frobenius error between these two representative datasets, effectively aligning the "average" causal behavior of the low-level system to that of the high-level system, while ignoring specific interventional variations.

\begin{algorithm}[H]
\caption{General \textsc{Bary}$_{\tau\omega}$ $~~~~~~~~~~~~~~~~~~~~~~~~~~~~~~~~~~~~~~~~~~~~~~~~~~~~~~~~~~~~~~~~~~~~~~~~~~~~~~~~~~~~~~~~~~~~~~~~~~~~~~~~~~~~~~~~~~~~~~~~~~~~~~~[d \in \{\ell, h\}$]}\label{alg:general_baryca}
\begin{algorithmic}[1]
\REQUIRE $\omega, \mathcal{I}^d, \scmdag_{\scmsimple^d}, \{\mathbf{X}_\iota^d\}_{\iota \in \mathcal{I}^d}$
\REQUIRE Hyperparameters: $\eta_\tau$ (learning rate)

\STATE \textit{// Step 1: Abduction and deterministic computation}
\STATE $\mathbf{U}^d \leftarrow \text{Abduct}(\mathbf{X}^d, \mathcal{M}^d)$ \AlgComment{\#Abduct observational noise}

\STATE \textit{// Step 2: Compute barycentric deterministic components}
\STATE $\overline{\mathbf{D}}^\ell \leftarrow \frac{1}{|\mathcal{I}^\ell|}\sum_{\iota \in \mathcal{I}^\ell}\mathbf{D}_\ell^{(\iota)}$
\STATE $\overline{\mathbf{D}}^h \leftarrow \frac{1}{|\mathcal{I}^h|}\sum_{\eta \in \mathcal{I}^h}\mathbf{D}_h^{(\eta)}$
\AlgComment{\#Average mechanism}

\STATE \textit{// Step 3: Optimization}
\STATE \textbf{Initialize:} $\lintau^{(0)}$
\REPEAT
    \STATE $\mathcal{L}(\lintau) \coloneq
    \left\| \lintau (\overline{\mathbf{D}}^\ell + \mathbf{U}^\ell)^\top
      - (\overline{\mathbf{D}}^h + \mathbf{U}^h) \right\|_F^2$
    \STATE $\lintau^{(t+1)} \leftarrow \lintau^{(t)} - \eta_\tau \nabla_{\lintau}\mathcal{L}(\lintau^{(t)})$
\UNTIL{convergence}

\textbf{Return:} $\lintau^\star$
\end{algorithmic}
\end{algorithm}

\begin{algorithm}[H]
\caption{Gaussian \textsc{DiRoCA} for linear models}\label{alg:elliptical_algo}
\begin{algorithmic}[1]
\REQUIRE $\cS^\ell, \cS^h, \intervsetbase, \intervsetabst, \omega:\intervsetbase \to \intervsetabst$, samples from $\prob_{\scmbase_\iota}(\enbase), \prob_{\scmabst_{\omega(\iota)}}(\enabst)$
\STATE $\widehat{\exprobbase}\sim \cN(\widehat{\mul}, \widehat{\sigl}) \leftarrow (\mixfbase_{\#}(\prob_{\scmbase}))^{-1}$
\AlgComment{\#Abduction for $\scmbase$}
\STATE $\widehat{\exprobabst}\sim \cN(\widehat{\muh}, \widehat{\sigh}) \leftarrow (\mixfabst_{\#}(\prob_{\scmabst}))^{-1}$
\AlgComment{\#Abduction for $\scmabst$}
\STATE $\radprod \leftarrow$ from Theorem 1
\AlgComment{\#Compute robustness radius}
\STATE $\mathbb{B}_{\radprod,2}(\widehat{\exprob}) \leftarrow \mathbb{B}_{\radbase,2}(\widehat{\exprobbase}) \times \mathbb{B}_{\radabst,2}(\widehat{\exprobabst})$
\STATE \textbf{Initialize:} $\lintau^{(0)} \in \mathbb{R}^{h\times \ell}$,
$\exprob^{(0)}=\cN(\mu^{\ell(0)},\Sigma^{\ell(0)})\otimes \cN(\mu^{h(0)},\Sigma^{h(0)})$,
$\eta_{\lintau},\eta_{\exprob}$, $k_{\min},k_{\max}$
\REPEAT
    \FOR{$k_{\min}$ steps}
        \STATE $\lintau^{(t+1)} \leftarrow \lintau^{(t)} - \eta_{\lintau}\,\nabla_{\lintau}F(\lintau^{(t)},\rho^{(t)})$
        \AlgComment{\#Update abstraction $\lintau$}
    \ENDFOR
    \FOR{$k_{\max}$ steps}
        \STATE $\mu^{\ell(t+1)} \leftarrow \mu^{\ell(t)} + \eta_{\exprob}\,\nabla_{\mu^{\ell}} F(\mu^{\ell},\lintau^{(t+1)})$
        \STATE $\mu^{h(t+1)} \leftarrow \mu^{h(t)} + \eta_{\exprob}\,\nabla_{\mu^h}F(\mu^h,\lintau^{(t+1)})$
        \STATE $\Sigma^{\ell(t+\frac{1}{2})} \leftarrow \Sigma^{\ell(t)} + \eta_{\rho}\,\nabla_{\Sigma^\ell}\tilde{F}(\lintau^{(t+1)},\rho^{(t)})$
        \STATE $\Sigma^{\ell(t+1)} \leftarrow \mathrm{prox}_{\lambda_\ell}\!\left(\Sigma^{\ell(t+\frac{1}{2})}\right)$
        \AlgComment{\#Proximal step}
        \STATE $\Sigma^{h(t+\frac{1}{2})} \leftarrow \Sigma^{h(t)} + \eta_{\rho}\,\nabla_{\Sigma^h}\tilde{F}(\lintau^{(t+1)},\rho^{(t)})$
        \STATE $\Sigma^{h(t+1)} \leftarrow \mathrm{prox}_{\lambda_h}\!\left(\Sigma^{h(t+\frac{1}{2})}\right)$
        \AlgComment{\#Proximal step}
        \STATE $(\mu^{\ell(t+1)},\Sigma^{\ell(t+1)}) \leftarrow \text{proj}_{\mathcal{W}_2\le \radbase}(\mu^{\ell(t+1)},\Sigma^{\ell(t+1)},\widehat{\mu^\ell},\widehat{\Sigma^\ell})$
        \AlgComment{\#Project to Gelbrich ball}
        \STATE $(\mu^{h(t+1)},\Sigma^{h(t+1)}) \leftarrow \text{proj}_{\mathcal{W}_2\le \radabst}(\mu^{h(t+1)},\Sigma^{h(t+1)},\widehat{\mu^h},\widehat{\Sigma^h})$
        \AlgComment{\#Project to Gelbrich ball}
        \STATE $\exprob^{(t+1)} \leftarrow \cN(\mu^{\ell(t+1)},\Sigma^{\ell(t+1)})\otimes \cN(\mu^{h(t+1)},\Sigma^{h(t+1)})$ \AlgComment{\#Update joint environment $\exprob$}
        
    \ENDFOR
\UNTIL{convergence}

\textbf{Return:} $\lintau^\star \in \mathbb{R}^{h\times \ell}$, $\exprob^\star=\rho^{\star\ell}\otimes \rho^{\star h}$
\end{algorithmic}
\end{algorithm}

\newpage
\begin{algorithm}[H]
\caption{Empirical \textsc{DiRoCA} for linear models}\label{alg:empirical_algo}
\begin{algorithmic}[1]
\REQUIRE $\cS^\ell,\cS^h,\intervsetbase,\intervsetabst,\omega:\intervsetbase\to \intervsetabst$, samples from $\prob_{\scmbase_\iota}(\enbase),\prob_{\scmabst_{\omega(\iota)}}(\enabst)$

\STATE $\widehat{\exprobbase}=\frac{1}{N_\ell}\sum_{i=1}^{N_\ell}\delta_{\widehat{u^\ell}_i}\leftarrow \mixbase^{-1}\enbase$
\AlgComment{\#Abduction for $\scmbase$}

\STATE $\widehat{\exprobabst}=\frac{1}{N_h}\sum_{i=1}^{N_h}\delta_{\widehat{u^h}_i}\leftarrow \mixabst^{-1}\enabst$
\AlgComment{\#Abduction for $\scmabst$}

\STATE $\widehat{\exprob}\leftarrow \widehat{\exprobbase}\otimes \widehat{\exprobabst}$
\AlgComment{\#Joint environment}

\STATE $\radprod \leftarrow$ from Theorem~\ref{thm:ca_product_concentration_emp} \AlgComment{\#Robustness radius}

\STATE \textbf{Initialize:} $\lintau^{(0)},\Theta_\ell^{(0)},\Theta_h^{(0)}$
\STATE $\rho^{\ell(0)}=\frac{1}{N_\ell}\sum_{i=1}^{N_\ell}\delta_{\widehat{u}^\ell_i+\theta^{\ell(0)}_i}$
\STATE $\rho^{h(0)}=\frac{1}{N_h}\sum_{i=1}^{N_h}\delta_{\widehat{u}^h_i+\theta^{h(0)}_i}$
\STATE $\exprob^{(0)}\leftarrow \rho^{\ell(0)}\otimes \rho^{h(0)}$

\REPEAT
    \FOR{$k_{\min}$ steps}
        \STATE $\lintau^{(t+1)} \leftarrow \lintau^{(t)} - \eta\,\nabla_{\lintau}F(\lintau^{(t)},\Theta_\ell^{(t)},\Theta_h^{(t)})$
        \AlgComment{\#Update abstraction $\lintau$}
    \ENDFOR

    \FOR{$k_{\max}$ steps}
        \STATE $\Theta_\ell^{\text{temp}} \leftarrow \Theta_\ell^{(t)} + \eta\,\nabla_{\Theta_\ell}F(\lintau^{(t+1)},\Theta_\ell^{(t)},\Theta_h^{(t)})$
        \STATE $\Theta_\ell^{(t+1)} \leftarrow \text{proj}_{\|\cdot\|_F \le \radbase\sqrt{N_\ell}}(\Theta_\ell^{\text{temp}})$
        \AlgComment{\#Project onto Frobenius norm ball}

        \STATE $\Theta_h^{\text{temp}} \leftarrow \Theta_h^{(t)} + \eta\,\nabla_{\Theta_h}F(\lintau^{(t+1)},\Theta_\ell^{(t)},\Theta_h^{(t)})$
        \STATE $\Theta_h^{(t+1)} \leftarrow \text{proj}_{\|\cdot\|_F \le \radabst\sqrt{N_h}}(\Theta_h^{\text{temp}})$
        \AlgComment{\#Project onto Frobenius norm ball}

        \STATE $\rho^{\ell(t+1)} \leftarrow \frac{1}{N_\ell}\sum_{i=1}^{N_\ell}\delta_{\widehat{u}^\ell_i+\theta^{\ell(t+1)}_i}$
        \STATE $\rho^{h(t+1)} \leftarrow \frac{1}{N_h}\sum_{i=1}^{N_h}\delta_{\widehat{u}^h_i+\theta^{h(t+1)}_i}$
        \STATE $\exprob^{(t+1)} \leftarrow \rho^{\ell(t+1)}\otimes \rho^{h(t+1)}$ \AlgComment{\#Update joint environment $\exprob$}
    \ENDFOR
\UNTIL{convergence}

\textbf{Return:} $\lintau^\star\in \bR^{h\times \ell}$, $\exprob^\star$
\end{algorithmic}
\end{algorithm}

\begin{algorithm}[H]
\caption{General \textsc{DiRoCA} $~~~~~~~~~~~~~~~~~~~~~~~~~~~~~~~~~~~~~~~~~~~~~~~~~~~~~~~~~~~~~~~~~~~~~~~~~~~~~~~~~~~~~~~~~~~~~~~~~~~~~~~~~~~~~~~~~~~~~~~~~~~~~~~[d \in \{\ell, h\}$]}\label{alg:general_diroca}
\begin{algorithmic}[1]

\REQUIRE $\omega, \mathcal{I}^d, \scmdag_{\scmsimple^d},
\{\mathbf{X}_\iota^d \in \mathbb{R}^{N \times d}\}_{\iota \in \mathcal{I}^d}, \varepsilon_d$ from Thm.~\ref{thm:ca_product_concentration_emp}.

\REQUIRE Hyperparameters: $\eta_\tau,\eta_\theta$ (learning rates), $N$ (batch size), $k_{\min},k_{\max}$ (inner steps)

\STATE \textit{// Step 1: Abduction and initialization}
\STATE $\mathbf{U}^d \leftarrow \text{Abduct}(\mathbf{X}^d, \mathcal{M}^d)$
\AlgComment{\#Abduct observational noise}

\STATE $\widehat{\rho^d} \leftarrow \frac{1}{N}\sum_{i=1}^{N}\delta_{u^d_i}$
\AlgComment{\#Define empirical marginals}

\STATE $\widehat{\exprob} \leftarrow \widehat{\rho^\ell}\otimes \widehat{\rho^h}$
\AlgComment{\#Joint environment}

\STATE \textit{// Step 2: Define constraints}
\STATE $r_d \leftarrow \varepsilon_d\sqrt{N}$
\AlgComment{\#Frobenius budgets}

\STATE \textit{// Step 3: Adversarial optimization}
\STATE \textbf{Initialize:} $\lintau^{(0)}$, $\Theta_\ell^{(0)} \leftarrow \mathbf{0}$, $\Theta_h^{(0)} \leftarrow \mathbf{0}$
\STATE $\rho^{d(0)} \leftarrow \frac{1}{N}\sum_{i=1}^{N}\delta_{u^d_i+\theta^{d(0)}_i}$
\STATE $\exprob^{(0)} \leftarrow \rho^{\ell(0)}\otimes \rho^{h(0)}$

\REPEAT

  \STATE \textit{// Minimization step (update abstraction)}
  \FOR{$k_{\min}$ steps}
    \STATE $\mathcal{J}(\lintau,\mathbf{\Theta}) \coloneq
    \mathbb{E}_{\iota}\!\left[\left\|
      \lintau(\mathbf{D}_\ell^{(\iota)}+\mathbf{U}^\ell+\Theta_\ell)^\top
      -(\mathbf{D}_h^{(\omega(\iota))}+\mathbf{U}^h+\Theta_h)
    \right\|_F^2\right]$
    \STATE $\lintau^{(t+1)} \leftarrow \lintau^{(t)} - \eta_\tau \nabla_{\lintau}\mathcal{J}(\lintau^{(t)},\mathbf{\Theta}^{(t)})$
  \ENDFOR

  \STATE \textit{// Maximization step (update joint environment)}
  \FOR{$k_{\max}$ steps}
    \STATE $\Theta_\ell^{\text{temp}} \leftarrow \Theta_\ell^{(t)} + \eta_\theta \nabla_{\Theta_\ell}\mathcal{J}(\lintau^{(t+1)},\mathbf{\Theta}^{(t)})$
    \STATE $\Theta_h^{\text{temp}} \leftarrow \Theta_h^{(t)} + \eta_\theta \nabla_{\Theta_h}\mathcal{J}(\lintau^{(t+1)},\mathbf{\Theta}^{(t)})$

    \STATE $\Theta_\ell^{(t+1)} \leftarrow \text{Proj}_{\|\cdot\|_F\le r_\ell}(\Theta_\ell^{\text{temp}})$
    \STATE $\Theta_h^{(t+1)} \leftarrow \text{Proj}_{\|\cdot\|_F\le r_h}(\Theta_h^{\text{temp}})$
    \AlgComment{\#Project onto ambiguity set}

    \STATE $\rho^{\ell(t+1)} \leftarrow \frac{1}{N}\sum_{i=1}^{N}\delta_{u^\ell_i+\theta^{\ell(t+1)}_i}$
    \STATE $\rho^{h(t+1)} \leftarrow \frac{1}{N}\sum_{i=1}^{N}\delta_{u^h_i+\theta^{h(t+1)}_i}$
    \AlgComment{\#Update marginals}

    \STATE $\exprob^{(t+1)} \leftarrow \rho^{\ell(t+1)}\otimes \rho^{h(t+1)}$
    \AlgComment{\#Update joint environment}
  \ENDFOR

\UNTIL{convergence} \\

\textbf{Return:} $\lintau^\star \in \bR^{h\times \ell}$, $\exprob^\star$

\end{algorithmic}
\end{algorithm}
\end{document}